\documentclass[twoside,11pt]{article}

\usepackage{jmlr2e}
\usepackage{algorithm}
\usepackage{algorithmic}

\usepackage{multirow}
\usepackage{multicol}
\usepackage{color}
\usepackage{amsmath}
\usepackage{cleveref}
\usepackage{booktabs}

\usepackage{lastpage}
\jmlrheading{22}{2021}{1-\pageref{LastPage}}{9/20; Revised
8/21}{9/21}{20-1006}{Jinshan Zeng, Shao-Bo Lin, Yuan Yao, Ding-Xuan Zhou}
\ShortHeadings{On ADMM in Deep Learning: Convergence and Saturation-Avoidance}{Zeng, Lin, Yao, Zhou}

\newtheorem{assumption}{Assumption}

\newcommand{\cL}{{\cal L}}
\newcommand{\hcL}{\hat{\cal L}}

\newcommand{\cQ}{{\cal Q}}
\newcommand{\hcQ}{{\hat{\cal Q}}}

\begin{document}

\title{On ADMM in Deep Learning: Convergence and Saturation-Avoidance}

\author{\name Jinshan Zeng \email jinshanzeng@jxnu.edu.cn \\
       \addr School of Computer and Information Engineering, Jiangxi Normal University, Nanchang, China\\
       \addr Liu Bie Ju Centre for Mathematical Sciences, City University of Hong Kong, Hong Kong\\
       \addr Department of Mathematics, Hong Kong University of Science and Technology, Hong Kong
       \AND
       \name Shao-Bo Lin\thanks{Corresponding author.} \email sblin1983@gmail.com \\
       \addr Center of Intelligent Decision-Making and Machine Learning, School of Management, Xi'an Jiaotong University, Xi'an, China
       \AND
       \name Yuan Yao$^*$  \email yuany@ust.hk \\
       \addr Department of Mathematics, Hong Kong University of Science and Technology, Hong Kong
       \AND
       \name Ding-Xuan Zhou \email  mazhou@cityu.edu.hk \\
       \addr  School of Data Science and Department of Mathematics, City University of Hong Kong, Hong Kong
       }

\editor{Marc Schoenauer}

\maketitle

\begin{abstract}
In this paper, we develop an alternating direction method of multipliers (ADMM) for  deep neural networks training with sigmoid-type activation functions (called \textit{sigmoid-ADMM pair}), mainly motivated by the gradient-free nature of ADMM in avoiding  the saturation of sigmoid-type activations and the advantages of deep neural networks with sigmoid-type activations (called deep sigmoid nets) over their rectified linear unit (ReLU) counterparts (called deep ReLU nets) in terms of approximation. In particular, we  prove that the approximation capability of deep sigmoid nets is not worse  than that of deep ReLU nets by showing that ReLU activation function can be well approximated by deep  sigmoid nets with two hidden layers and finitely many free parameters but not vice-verse. We also establish the global convergence of the proposed ADMM for the nonlinearly constrained formulation of the deep sigmoid nets training from arbitrary initial points to a Karush-Kuhn-Tucker (KKT) point at a rate of order ${\cal O}(1/k)$. Besides sigmoid activation, such a convergence theorem holds for a general class of smooth activations. Compared with the widely used stochastic gradient descent (SGD) algorithm for the deep ReLU nets training (called ReLU-SGD pair), the proposed sigmoid-ADMM pair is practically stable with respect to the algorithmic hyperparameters including the learning rate, initial schemes and the pro-processing of the input data. Moreover, we find that to approximate and learn simple but important functions   the proposed sigmoid-ADMM pair numerically outperforms the ReLU-SGD pair.
\end{abstract}


\section{Introduction}

In the era of big data, data of massive size are collected in a wide range of applications including image processing, recommender systems, search engineering, social activity mining and natural language processing \citep{Zhou2014}. These massive data provide a springboard to design machine learning systems matching or outperforming human capability but pose several challenges on how to develop learning systems to sufficiently exploit the data. As shown in Figure \ref{Fig:mechnism}, the traditional approach comes down to a three-step learning process.  It at first adopts delicate data transformations to yield a tractable representation of the original massive data; then develops some interpretable and computable optimization models based on the transformed data to embody the utility of data;  finally designs  efficient algorithms to solve the proposed optimization problems. These three steps are called feature extraction, model selection and algorithm designation respectively.  Since feature extraction usually involves  human ingenuity and prior knowledge, it is labor intensive, especially when the data size is huge. Therefore, it is highly desired to reduce the human factors in the learning process.
\begin{figure}[h]
    \centering
    \includegraphics[scale=0.5]{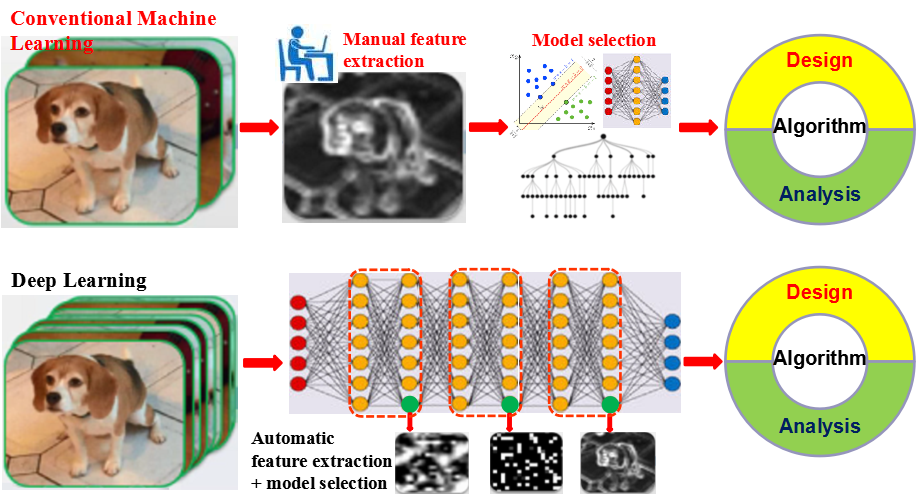}
    \caption{Philosophy behind deep learning}
    \label{Fig:mechnism}
\end{figure}

Deep learning \citep{Hinton-Salakhutdinov-DL06,LeCun-Bengio-Hinton-DL15}, which utilizes deep neural networks (deep nets for short) for feature extraction and model selection simultaneously, provides a promising way to reduce  human factors in machine learning. Just as Figure \ref{Fig:mechnism} purports to show, deep learning transforms the classical three-step strategy into a two-step approach: neural networks selection and algorithm designation. It is thus important to pursue   why such a transformation is feasible and  efficient. In particular, we are interested in  making clear of when deep nets are better than  classical methods such as shallow neural networks (shallow nets)  and kernel methods,  and which optimization algorithm is good enough to realize the benefits brought from deep nets.

In the past decade, deep nets with ReLU  activations (deep ReLU nets) equipped with the well known stochastic gradient descent (SGD) algorithm have been successfully used in image classification  \citep{Hinton-imagenet-2012}, speech recognition \citep{Hinton-speech-2012,Sainath-speech-2013},  natural language processing \citep{Devlin-NLP-2014},
demonstrating the power of ReLU-SGD pair  in deep learning. The problem is, however, that there is a crucial inconsistency between approximation and optimization for the ReLU-SGD pair. To be detailed,  from the approximation  theory viewpoint, it is necessary to deepen the network to approximate smooth function \citep{Yarotsky17}, extract  manifold structures \citep{Shaham-Cloninger-Coifman18}, realize rotation-invariance features \citep{han2020depth} and provide localized approximation \citep{Safran-Shamir17}. However, from the optimization viewpoint, it is difficult to solve optimization problems associated with  too deep networks with theoretical guarantees \citep{Goodfellow-et-al-2016}.  Besides the lack of convergence (to a global minima) guarantees, deep ReLU nets equipped with  SGD  may suffer  from the issue of gradient explosion/vanishing  \citep{Goodfellow-et-al-2016} and is usually sensitive to its algorithmic hyper-parameters such as the initialization \citep{Glorot-Bengio10,Sutskever-optdnn-2013,Hanin-Rolnick18} and learning rate \citep{Senior-lr13,Daniel-lr16,Ruder-SGD-review16} in the sense that these parameters have dramatic impacts on the performance of SGD and thus should be carefully tuned in practice. In a nutshell, deep  ReLU nets should be deep enough to exhibit excellent approximation  capability while too deep networks frequently impose additional difficulty  in optimization.

There are numerous remedies to tackle the aforementioned inconsistency for the ReLU-SGD pair with intuition that SGD as well as its variants is capable of efficiently solving the optimization problem associated with deep ReLU nets. In particular, some tricks on either the network architectures such as ResNets \citep{He-ResNet-2016} or the training procedure such as the batch normalization \citep{Ioffe-Szegedy-BN15} and weight normalization \citep{Salimans-Kingma-WN16} have been developed to address the issue of gradient vanishing/explosion; several efficient initialization  schemes including the \textit{MSRA} initialization \citep{He-msra-init15} have been proposed for deep ReLU nets; some  guarantees have been established \citep{Allen-Zhu-SGD19,Du-GD19,Zou-Gu-SGD19} in the over-parametrized setting to verify the convergence of SGD; and numerous strategies of
learning rates \citep{Chollet2015keras,Gotmare-lr-warmup19,Smith-lr-cyclical19} have been provided to enhance the feasibility of SGD.

Different from the aforementioned approach focusing on modifying SGD for deep ReLU nets, we pursue an alternative direction to ease the training via reducing the depth. Our studies stem from an interesting observation in neural networks approximation. As far as the approximation capability is concerned, deep nets with sigmoid-type activation functions (deep sigmoid nets) theoretically perform better than deep ReLU nets for some function classes in the sense that to attain the same approximation accuracy, the depth and number of parameters of the former is much smaller than those of the latter.  This phenomenon was observed in approximating smooth functions \citep{Mhaskar1996,Yarotsky17}, reflecting the rotation invariance feature \citep{Chui-Lin-Zhou19,han2020depth} and capturing sparse signals \citep{LinH2017,Schwab2019}.

\begin{figure}[!t]
\begin{minipage}[b]{0.325\linewidth}
\centering
\includegraphics*[scale=.35]{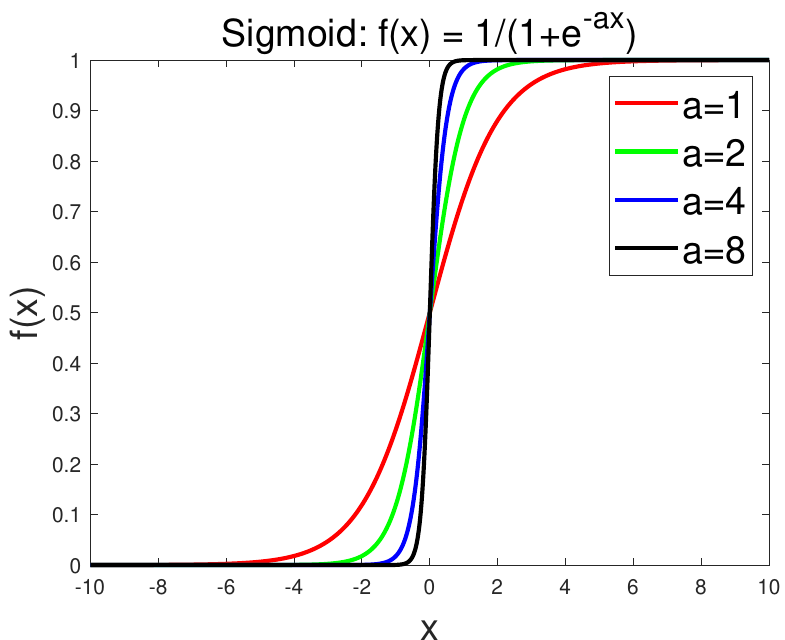}
\centerline{{\small (a) Sigmoid functions}}
\end{minipage}
\hfill
\begin{minipage}[b]{0.325\linewidth}
\centering
\includegraphics*[scale=.35]{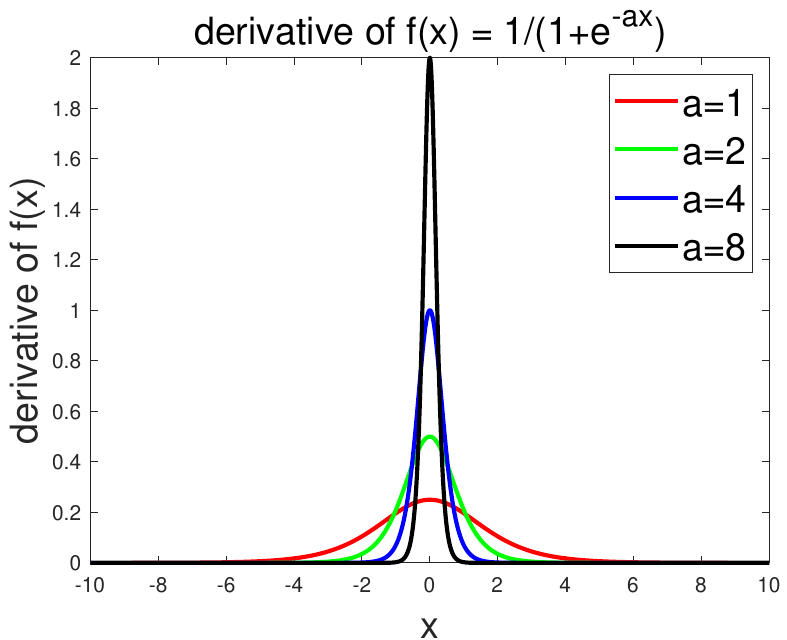}
\centerline{{\small (b) Saturation of sigmoid}}
\end{minipage}
\hfill
\begin{minipage}[b]{0.325\linewidth}
\centering
\includegraphics*[scale=.35]{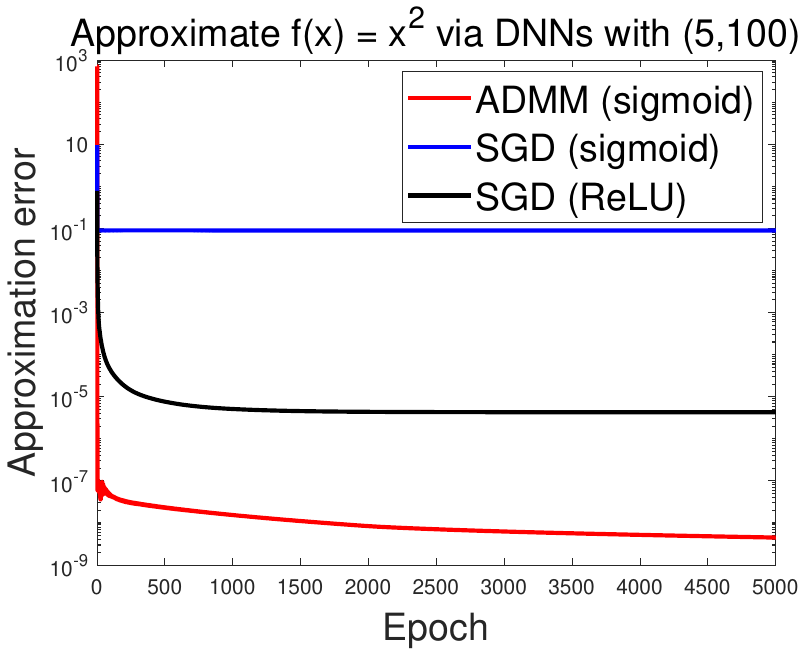}
\centerline{{\small (c) SGD vs. ADMM}}
\end{minipage}
\hfill
\caption{The cons of SGD and pros of ADMM in solving deep sigmoid nets. The setting of numerical simulation in (c) can be found in the Table \ref{Tab:exp-setting} below.}
\label{Fig:sigmoid}
\end{figure}

In spite of their advantages in approximation, deep sigmoid nets have not been widely used in the deep learning community. The major reason is due to the saturation problem of the sigmoid function \footnote{A function $f: \mathbb{R} \rightarrow \mathbb{R}$ is said to be \textit{saturating} if it is differentiable and its derivative $f'(x)$ satisfies $\lim_{|x| \rightarrow +\infty} f'(x)= 0$.} \citep[Section 6.3]{Goodfellow-et-al-2016}, which is easy to cause gradient vanishing for  gradient-descent based algorithms in the deep sigmoid nets training \citep{Bengio-gradientvanish1994,LeCun-init98b}. Specifically, as shown in Figure \ref{Fig:sigmoid} (b), derivatives of sigmoid functions  vanish numerically in a large range. In this paper, we aim at developing a gradient-free algorithm for the deep sigmoid nets training to avoid saturation of deep sigmoid nets and sufficiently embody their theoretical advantages. As a typical gradient-free optimization algorithm, alternating direction method of multipliers (ADMM)  can be regarded as a  primal-dual method based on an augmented Lagrangian by introducing nonlinear constraints and enables a convergent sequence satisfying the nonlinear constraints. Therefore, ADMM attracted rising attention in deep learning with various implementations \citep{Carreira2014-MAC,Goldstein-ADMM-DNN2016,Kiaee2016-ADMM-sparseCNN,Yang-ADMM-net2016,Gotmare2018,Murdock-ADNN2018}. Under this circumstance, we propose an efficient ADMM algorithm based on a novel update order and an efficient sub-problem solver. Surprisingly, as shown in Figure \ref{Fig:sigmoid} (c), the proposed  sigmoid-ADMM pair  performs better than ReLU-SGD pair in approximating the simple but extremely important square function \citep{Yarotsky17,Petersen-Voigtlaender18,han2020depth}. This implies that ADMM may be an efficient algorithm to sufficiently realize theoretical advantages of deep sigmoid nets. Our contributions of this paper can be summarized as the following three folds.

$\bullet$ {\bf Methodology Novelty:} We develop a novel sigmoid-ADMM pair for deep learning. Compared with the widely used ReLU-SGD pair, the proposed  sigmoid-ADMM pair is stable with respect to algorithmic hyperparameters including  learning rates, initial schemes and the pro-processing of input data.  Furthermore, we find that to approximate and learn simple but important functions including the square function, radial functions and product gate, deep sigmoid nets theoretically beat deep ReLU nets and the proposed sigmoid-ADMM pair outperforms the ReLU-SGD pair. In terms of algorithm designs, different from existing ADMM methods in deep learning, our proposed ADMM adopts a backward-forward update order that is similar as BackProp \citep{Hinton-BP1986} and a local linear approximation for sub-problems, and more importantly keeps all the nonlinear constraints such that the solution found by the proposed algorithm can converge to a solution satisfying these nonlinear constraints.

$\bullet$ {\bf Theoretical Novelty:}  To demonstrate the theoretical advantages of deep sigmoid nets, we rigorously prove that the approximation capability of deep sigmoid nets is not worse  than deep ReLU nets by showing that ReLU   can be well approximated by deep  sigmoid nets with two hidden layers and finitely many free parameters but not vice-verse.
We also establish the global convergence of the proposed ADMM for the nonlinearly constrained formulation of the deep sigmoid nets training from arbitrary initial points to a Karush-Kuhn-Tucker (KKT) point at a rate of order ${\cal O}(1/k)$. Different from the existing literature on convergence of nonconvex ADMM \citep{Hong-ADMM-2016,Wang-ADMM2018,Gao-Goldfarb-ADMM20} for linear or multiaffine constrained optimization problems, our analysis provides a new methodology to deal with the nonlinear constraints in deep learning. In a word, our approach  actually leads to a general convergence framework for ADMM with ``smooth'' enough activations.

$\bullet$ {\bf Numerical Novelty:} In terms of numerical performance, the effectiveness (particularly the stability to initial schemes and the easy-to-tune property of algorithmic parameters) of the proposed ADMM has been demonstrated by numerous experiments including a series of toy simulations and three real-data experiments. Numerical results illustrate the outperformance of the sigmoid-ADMM pair over the ReLU-SGD pair in approximating the extremely important square function, product gate, piecewise $L_1$ radial and smooth $L_2$ radial functions with stable algorithmic hyperparameters. Together with some other important functions such as the localized approximation \citep{Chui-Li-Mhaskar94}, these natural functions realized in this paper can represent some important data features such as piecewise smoothness in image processing \citep{Hinton-imagenet-2012}, sparseness in computer vision \citep{LeCun-Bengio-Hinton-DL15}, and rotation-invariance in earthquake prediction \citep{Vikraman-earthquake16}. The effectiveness of the proposed ADMM is further demonstrated by real-data experiments, i.e., earthquake intensity, extended Yale B databases and PTB Diagnostic ECG databases, which reflect the partially radial and low-dimensional manifold features  in some extent.

The rest of this paper is organized as follows. In the next section, we demonstrate the advantage of deep sigmoid nets in approximation (see Theorem \ref{theorem:app}).
Section \ref{sc:ADMM} describes the proposed ADMM method for the considered DNN training model followed by the main convergence theorem (see Theorem \ref{Thm:Conv-ADMM-sigmoid}).
Section \ref{sc:related-work} provides some discussions on related work and key ideas of our proofs.
Section \ref{sc:simulations} provides some toy simulations to show the effectiveness of the proposed ADMM method in realizing some important natural functions.
Section \ref{sc:real-experiments} provides two real data experiments to further demonstrate the effectiveness of the proposed method.
All proofs are presented in Appendix.

{\bf Notations:} For any matrix $A \in \mathbb{R}^{m\times n}$, $[A]_{ij}$ denotes its $(i,j)$-th entry. Given a matrix $A$, $\|A\|_F$, $\|A\|_2$ and $\|A\|_{\max}$ denote the Frobenius norm, operator norm, and max-norm of $A$, respectively, where $\|A\|_{\max} = \max_{i,j} |[A]_{ij}|$. Then obviously, $\|A\|_{\max} \leq \|A\|_2 \leq \|A\|_F$. We let $W_{<i}:=[W_1, W_2, \ldots, W_{i-1}]$, $W_{>i} := [W_{i+1},\ldots,W_N]$ for $i=1,\ldots, N$,  $W_{<1} = \emptyset$ and $W_{>N} = \emptyset$. ${\bf I}$ denotes the identity matrix whose size can be determined according to the text. Denote by $\mathbb{R}$ and $\mathbb{N}$ the real and natural number sets, respectively.

\section{Deep Sigmoid Nets in Approximation}
\label{Sec.versus}
For the depth $N\in\mathbb N$ of a neural network, let $d_i \in \mathbb{N}$ be the number of hidden neurons at the $i$-th hidden layer for $i=1,\ldots,N-1$. Denote an affine mapping ${\cal J}_i:\mathbb R^{d_{i-1}}\rightarrow\mathbb R^{d_i}$ by $\mathcal J_i(x):=W_i  x+b_i$ for $d_i\times d_{i-1}$ weight matrix and  thresholds  $b_i\in\mathbb R^{d_i}$. For a univariate activation function $\sigma_i, i=1\dots,N$, denote further $\sigma_i(x)$ when  $\sigma_i$ is applied component-wise to the vector $x$. Define an $N$-layer feedforward neural network by
\begin{equation}\label{deep-net}
     \mathcal N_{N,d_1,\dots,d_N,\sigma}
     = a\cdot \sigma_N\circ \mathcal J_N \circ \sigma_{N-1}\circ \mathcal J_{N-1} \circ \dots \circ \sigma_{1}\circ\mathcal J_1(x),
\end{equation}
where $a \in\mathbb R^{d_N}$.
The deep net  defined in (\ref{deep-net}) is called the deep ReLU net and deep sigmoid net, provided $\sigma_i(t)\equiv \sigma_{relu}(t)=\max\{t,0\}$ and $\sigma_i(t)\equiv \sigma(t)=\frac{1}{1+e^{-t}}$ respectively.

Since the (sub-)gradient computation of ReLU is very simple, deep ReLU nets have attracted enormous research activities  in the deep learning community \citep{Hinton2010-ReLU}.
The power   of deep ReLU nets, compared with shallow nets with ReLU (shallow ReLU nets) has been sufficiently explored in the literature \citep{Yarotsky17,Petersen-Voigtlaender18,Shaham-Cloninger-Coifman18,Schwab2019,Chui2020,han2020depth}. In particular, it was proved in \citep[Proposition 2]{Yarotsky17} that the following ``square-gate'' property  holds  for deep ReLU nets,  which is beyond the capability of shallow ReLU nets due to the non-smoothness of ReLU.

\begin{lemma}\label{Lemma:ReLU-square gate for infinite}
The function $f(t)=t^2$ on the segment $[-M,M]$ for $M>0$ can be approximated within any accuracy $\varepsilon>0$ by a deep ReLU net with the depth and free parameters of order $\mathcal O(\log(1/\varepsilon))$.
\end{lemma}

The above lemma exhibits the  necessity of the depth for deep ReLU nets to act as a ``square-gate''. Since the depth depends on the accuracy, it requires many hidden layers for deep ReLU nets for such an easy task and too many hidden layers enhance the difficulty for analyzing SGD \citep[Sec.8.2]{Goodfellow-et-al-2016}. This presents the reason why the numerical accuracy of deep ReLU nets in approximating $t^2$ is not so good, just as Figure \ref{Fig:sigmoid} exhibits. Differently, due to the infinitely differentiable property of the sigmoid function, it is easy for shallow sigmoid nets \citep[Proposition 1]{Chui-Lin-Zhou19} to play as a ``square-gate'', as shown in  the following lemma.
\begin{lemma}\label{lemma:sig-square-gate}
Let $M>0$. For $\varepsilon>0$, there is  a shallow sigmoid net $\mathcal N_3$ with 3 free parameters bounded by $\mathcal O(\varepsilon^{-6})$ such that
$$
     |t^2- \mathcal N_{3}(t) |\leq \varepsilon,\qquad t\in[-M,M].
$$
\end{lemma}

Besides the ``square-gate'', deep sigmoid nets are capable of acting as a ``product-gate'' \citep{Chui-Lin-Zhou19}, providing localized approximation \citep{Chui-Li-Mhaskar94}, extracting the rotation-invariance property \citep{Chui-Lin-Zhou19} and reflecting the sparseness in spatial domain \citep{Lin19} and frequency domain \citep{LinH2017}  with much fewer hidden layers than deep ReLU nets. The following Table \ref{ReLU_fea} presents a comparison between deep sigmoid nets and deep ReLU nets in feature selection and approximation.
\begin{table}[!h]
\caption{Depth required for deep nets  in feature extraction and approximation  within accuracy
$\varepsilon$}\label{ReLU_fea}
\begin{center}
\vspace{-.5cm}
\begin{tabular}{|l|l|l|l|l|}
 \hline  Features & sigmoid & ReLU\\
 \hline   Square-gate  & $1$ \citep{Chui-Lin-Zhou19}  & $\log(\varepsilon^{-1})$ \citep{Yarotsky17}\\
 \hline   Product-gate & $1$ \citep{Chui-Lin-Zhou19}  & $\log(\varepsilon^{-1})$ \citep{Yarotsky17}\\
\hline   Localized approximation  & $2$ \citep{Chui-Li-Mhaskar94}  & $2$\citep{Chui2020}\\
\hline   $k$-spatially sparse+smooth & $2$ \citep{Lin19} & $>8$ \citep{Chui2020}  \\
\hline  Smooth+Manifold & $3$ \citep{Chui-Lin-Zhou18} & $4$ \citep{Shaham-Cloninger-Coifman18} \\
\hline   Smooth & 1 \citep{Mhaskar1996}  & $\log(\varepsilon^{-1})$ \citep{Yarotsky17} \\
\hline  $k$-sparse (frequency) & $k$ \citep{LinH2017}& $k\log(\varepsilon^{-1})$ \citep{Schwab2019}  \\
\hline Radial+smooth & 4 \citep{Chui-Lin-Zhou19} & $>8$ \citep{han2020depth} \\
\hline
\end{tabular}%
\end{center}
\end{table}


Table \ref{ReLU_fea} presents   theoretical  advantages of deep  sigmoid nets over deep ReLU nets. In fact, as the following theorem shows, it is easy to construct a sigmoid net with two hidden layers and small number of free parameters to approximate ReLU,  implying that
 the approximation capability of deep sigmoid nets is at least not worse than that of deep ReLU nets with comparable hidden layers and free parameters.

\begin{theorem}\label{theorem:app}
Let $1\leq p<\infty$ and $M\geq 1$. Then for any $\varepsilon\in(0,1/2)$ and $M>0$, there is a sigmoid net $h^*$ with $2$ hidden layers and at most 27 free parameters bounded by $\mathcal O(\varepsilon^{-7})$ such that
\begin{equation}\label{approx-comp}
      \|h^*-\sigma_{relu}\|_{L^p([-M,M])}\leq \varepsilon,
\end{equation}
where $L^p([-M,M])$ denotes the $L^p$ space of functions defined on $[-M, M]$.
\end{theorem}

The proof of Theorem \ref{theorem:app} is postponed in Appendix \ref{Sec.Proof.th1}.
For an arbitrary deep ReLU net
$$
     \mathcal N^{relu}_{L,d_1,\dots,d_L}
     = a\cdot \sigma_{relu}\circ \mathcal J_L \circ \sigma_{relu}\circ \mathcal J_{L-1} \circ \dots \circ \sigma_{relu}\circ\mathcal J_1(x)
$$
with bounded free parameters, Theorem \ref{theorem:app} shows that we can construct a deep sigmoid net
$$
      \mathcal N^{sigmoid}_{L,d_1,\dots,d_L}
     = a\cdot h^{*}\circ \mathcal J_L \circ h^{*}\circ \mathcal J_{L-1} \circ \dots \circ h^{*}\circ\mathcal J_1(x)
$$
that possesses at least similar approximation capability. However, due to the infinitely differentiable property of the sigmoid function, it is difficult to construct a deep ReLU net with accuracy-independent depth and width to  approximate it. Indeed, it can be found in \cite[Theorem 4.5]{Petersen-Voigtlaender18} that for any open interval $\Omega$ and deep ReLU net $\mathcal N^{relu}_{L,n}$ with $L$ hidden layers and $n$ free parameters, there holds
\begin{equation}\label{converse}
     \|\sigma- \mathcal N^{relu}_{L,n}\|_{L^p(\Omega)}\geq C''n^{-2L},
\end{equation}
where $\sigma$ is the sigmoid activation and $C''$ is a constant independent of $n$ or $L$. Comparing   (\ref{approx-comp}) with (\ref{converse}), we find that any functions being well approximated by deep ReLU nets can also  be well approximated by deep sigmoid nets, but not vice-verse.

It should be mentioned that though the depth and width in Theorem \ref{theorem:app} are independent of $\varepsilon$ and relatively small, the magnitude of  free parameters depends heavily on the accuracy and may be large. Since such large free parameters are difficult to realize  for an optimization algorithm, a preferable way to shrink them is to deepen the network further. In particular, it can be found in \citep{Chui-Lin-Zhou19} that  there is a shallow  sigmoid net $\mathcal N^{sigmoid}_{2}$ with 2 free parameters of order $\mathcal O(1/\varepsilon)$ such that
$$
    |\sigma(\sqrt{w}\mathcal N^{sigmoid}_{2}(\sqrt{w}t))-\sigma(wt)|\leq \varepsilon,
$$
where $t,w\in\mathbb R$.
Because we only focus on the power of deep sigmoid nets in approximation, we do not shrink free parameters in Theorem \ref{theorem:app}.

\section{ADMM for Deep Sigmoid Nets}
\label{sc:ADMM}
Let  ${\cal Z}:= \{(x_j, y_j)\}_{j=1}^n \subset \mathbb{R}^{d_0} \times \mathbb{R}^{d_N}$ be $n$ samples. Denote $X:= (x_1,x_2,\ldots, x_n) \in \mathbb{R}^{d_0 \times n}$ and $Y:= (y_1, y_2, \ldots, y_n) \in \mathbb{R}^{d_N \times n}$. It is natural to consider the following regularized DNN training problem
\begin{equation}
\label{Eq:dnn-L2-org}
   \min_{\cal W} \left\{ \frac1n\sum_{i=1}^n \|\Phi(x_i,\mathcal W)-y_i\|_2^2+ \lambda' \|W_i\|_F^2\right\},
\end{equation}
where $\Phi(x_i,\mathcal W)$ denotes a deep sigmoid net with  $N$ layers, $\mathcal W=\{W_i\}_{i=1}^N$ and $\lambda'>0$ is the regularization parameter. Here, we consider the square loss as analyzed in the literature \citep{Allen-Zhu-SGD19,Du-GD19,Zou-Gu-SGD19}. We also absorb thresholds into the weight matrices for the sake of simplicity.   Based on the advantage of  deep sigmoid nets in approximation, \citep{Chui-Lin-Zhou19,Lin19} proved that the model defined by (\ref{Eq:dnn-L2-org}) with $N=2$ are optimal in embodying data features such as  the spatial sparseness, smoothness and rotation-invariance in the sense that it can achieve almost optimal generalization error bounds in the framework of learning theory. The aim of this section is to introduce an  efficient algorithm to solve the optimization problem (\ref{Eq:dnn-L2-org}).

Due to the saturation problem of the sigmoid function (see Figure \ref{Fig:sigmoid} (b)), the issue of gradient vanishing or explosion frequently happens for running SGD on deep sigmoid nets (see Figure \ref{Fig:squarefun-motivation} (a) for example), implying that the classical SGD is not a good candidate to solve (\ref{Eq:dnn-L2-org}). We then turn to designing a gradient-free optimization algorithm, like ADMM, to efficiently solve (\ref{Eq:dnn-L2-org}). For DNN training, there are generally two important ingredients in designing ADMM: update order and solution to each sub-problem. The novelty of our proposed algorithm is  the use of backward-forward update order similar to BackProp in \citep{Hinton-BP1986}  and  local linear approximation to sub-problems.

\subsection{Update order in ADMM for deep learning training}
\label{sc:ADMM-description}

The optimization problem \eqref{Eq:dnn-L2-org} can be equivalently reformulated as the following constrained optimization problem
\begin{align}
\label{Eq:dnn-L2-admm-reg}
&\mathop{\mathrm{minimize}}_{{\cal W}, {\cal V}} \  \frac{1}{2} \|V_N-Y\|_F^2 + \frac{\lambda}{2} \sum_{i=1}^N \|W_i\|_F^2\\
&\text{subject to} \quad V_i = \sigma(W_iV_{i-1}), \  i=1,\ldots, N-1, \quad V_N = W_{N}V_{N-1}, \nonumber
\end{align}
where ${\cal V}:=\{V_i\}_{i=1}^N$ represents the set of responses of all layers and $\lambda = \frac{\lambda'n}{2}$. We define the augmented Lagrangian of \eqref{Eq:dnn-L2-admm-reg} as follows:
\begin{align}
\label{Eq:ALMFun-L2}
{\cal L}({\cal W}, {\cal V}, \{\Lambda_i\}_{i=1}^N)
& := \frac{1}{2} \|V_N-Y\|_F^2 + \frac{\lambda}{2} \sum_{i=1}^N \|W_i\|_F^2 \\
& + \sum_{i=1}^{N-1} \left(\frac{\beta_i}{2} \|\sigma(W_iV_{i-1}) - V_i\|_F^2 + \langle \Lambda_i, \sigma(W_iV_{i-1}) - V_i\rangle \right) \nonumber\\
& + \frac{\beta_N}{2} \|W_NV_{N-1} - V_N\|_F^2 + \langle \Lambda_N, W_NV_{N-1} - V_N\rangle, \nonumber
\end{align}
where $\Lambda_i \in \mathbb{R}^{d_i \times n}$ is the multiplier matrix associated with the $i$-th constraint, and $\beta_i$ is the associated penalty parameter for $i=1,\ldots,N$.

ADMM is an augmented-Lagrangian based primal-dual method, which updates the primal variables ($\{W_i\}_{i=1}^N$ and $\{V_i\}_{i=1}^N$ in \eqref{Eq:ALMFun-L2}) via a Gauss-Seidel scheme and then multipliers ($\{\Lambda_i\}_{i=1}^N$ in \eqref{Eq:ALMFun-L2}) via a gradient ascent scheme in a parallel way \citep{Boyd-DADMM2011}. As suggested in \citep{Wang-ADMM2018}, the update order of the primal variables is tricky for ADMM in terms of the convergence analysis in the nonconvex setting. In light of \citep{Wang-ADMM2018}, the key idea to yield a \textit{desired update order} with convergence guarantee is to arrange the updates of some \textit{special primal variables} followed by the updates of multipliers such that the updates of multipliers can be \textit{explicitly expressed} by the updates of these \text{special primal variables}, and thus the dual ascent quantities arisen by the updates of multipliers shall be controlled by the descent quantities brought by the updates of these \text{special primal variables}.
Hence, the arrangement of these \text{special primal variables} is crucial.

It can be noted that there are $2N$ blocks of primal variables, i.e., $\{W_i\}_{i=1}^N$ and $\{V_i\}_{i=1}^N$  and $N$ blocks of multipliers $\{\Lambda_i\}_{i=1}^N$ involved in (\ref{Eq:ALMFun-L2}). For better elaboration of our idea, we take $N=3$ for an example. Notice that the multipliers $\{\Lambda_i\}_{i=1}^3$ are only involved in these inner product terms $\langle \Lambda_1, \sigma(W_1X)-V_1\rangle$, $\langle \Lambda_2, \sigma(W_2V_1) - V_2\rangle$ and $\langle \Lambda_3, W_3V_2 - V_3\rangle$. By these terms, the gradient of the $i$-th inner product with respect to $V_i$ is $-\Lambda_i$, while the associated gradient with respect to $W_i$ is a more complex term (namely, $(\Lambda_1\odot \sigma'(W_1X))X^T$ for $W_1$, $(\Lambda_2\odot \sigma'(W_2V_1))V_1^T$ for $W_2$, and $\Lambda_3V_2^T$ for $W_3$, where $\odot$ represents Hadamard product). If the update of $W_i$ is used to express $\Lambda_i$, then according to the $W_i$ subproblem, an inverse operation of a nonlinear or linear mapping is required, while such an inverse does not necessarily exist. Specifically, following the analysis of Lemma \ref{Lemm:dual-expressed-primal} shown later and taking the expression of $W_3$ for example, the term $\Lambda_3V_2^T$ will be involved in the expression of $W_3$. In this case, if we wish to express $\Lambda_3$ by $W_3$, then the inverse of $V_2$ is generally required, while it does not necessarily exist.
Due to this, it should be more convenient to express $\Lambda_i$ $(i=1,2,3)$ via exploiting the $V_i$ subproblem instead of the $W_i$ subproblem. Therefore, we suggest firstly update the blocks of $W_i$'s and then $V_i$'s such that $\Lambda_i$'s can be explicitly expressed via the latest updates of $V_i$'s. To be detailed, for each loop, we update $\{W_i\}_{i=1}^N$ in the backward order, i.e., $W_N \rightarrow W_{N-1} \rightarrow \cdots \rightarrow W_1$, then update $\{V_j\}_{j=1}^N$ in the forward order, i.e., $V_1 \rightarrow V_{2} \rightarrow \cdots \rightarrow V_N$, motivated by BackProp in \citep{Hinton-BP1986}, and finally update the multipliers $\{\Lambda_i\}_{i=1}^N$ in a parallel way, as shown by the following Figure \ref{Fig:update-order}.

\begin{figure}[h]
    \centering
    \includegraphics[scale=0.5]{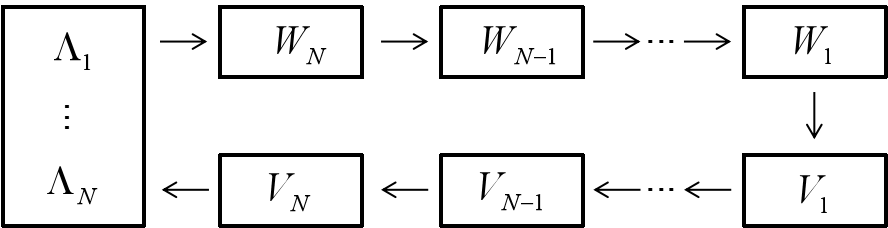}
    \caption{Update order of ADMM}
    \label{Fig:update-order}
\end{figure}

Specifically, given an initialization $\{W_i^0\}_{i=1}^N$, we set
\begin{align}
\label{Eq:initialization-admm}
V_j^0 = \sigma(W_j^0V_{j-1}^0), \ j=1,\ldots, N-1, \quad V_N^0 = W_N^0 V_{N-1}^0, \  \text{and}\ \Lambda_i^0 =0, \ i=1,\ldots, N,
\end{align}
where $V_0^0 = X$. Given the ($k$-$1$)-th iterate $\Big(\{W_i^{k-1}\}_{i=1}^N, \{V_i^{k-1}\}_{i=1}^N,$ $\{\Lambda_i^{k-1}\}_{i=1}^N\Big)$, we define the $W_i$- and $V_i$-subproblems at the $k$-th iteration via minimizing the augmented Lagrangian \eqref{Eq:ALMFun-L2} with respect to only one block but fixing the other blocks at the latest updates, according to the update order specified in Figure \ref{Fig:update-order}, shown as follows:
\begin{align}
& W_N^k = \arg \min_{W_N} \left\{\frac{\lambda}{2}\|W_N\|_F^2 + \frac{\beta_N}{2}\|W_NV_{N-1}^{k-1} - V_N^{k-1}\|_F^2 + \langle \Lambda_N^{k-1}, W_NV_{N-1}^{k-1} - V_N^{k-1}\rangle\right\}, \label{Eq:Wnk-prox}
\end{align}
and for $i=N-1,\ldots, 1,$
\begin{align}
\label{Eq:Wik-orig}
& W_i^k = \arg \min_{W_i} \left\{\frac{\lambda}{2}\|W_i\|_F^2 + \frac{\beta_i}{2}\|\sigma(W_iV_{i-1}^{k-1}) - V_i^{k-1}\|_F^2 + \langle \Lambda_i^{k-1}, \sigma(W_iV_{i-1}^{k-1}) - V_i^{k-1}\rangle\right\},
\end{align}
and for $j=1,\ldots, N-2$,
\begin{align}
\label{Eq:Vjk-orig}
& V_j^k = \arg\min_{V_j} \left\{\frac{\beta_j}{2}\|\sigma(W_j^kV_{j-1}^k)-V_j\|_F^2 + \langle \Lambda_j^{k-1},\sigma(W_j^kV_{j-1}^k)-V_j \right. \rangle\\
& \ \ \ \ \ \ \ \ \ \ \ \ \ \ \ \ \ \ \ \ \left. +\frac{\beta_{j+1}}{2}\|\sigma(W_{j+1}^kV_{j})-V_{j+1}^{k-1}\|_F^2 + \langle \Lambda_{j+1}^{k-1},\sigma(W_{j+1}^kV_{j})-V_{j+1}^{k-1} \rangle \right\}, \nonumber\\
& V_{N-1}^k = \arg\min_{V_{N-1}} \left\{\frac{\beta_{N-1}}{2}\|\sigma(W_{N-1}^kV_{N-2}^k)-V_{N-1}\|_F^2 + \langle \Lambda_{N-1}^{k-1},\sigma(W_{N-1}^kV_{N-2}^k)-V_{N-1} \right. \rangle \nonumber\\
& \ \ \ \ \ \ \ \ \ \ \ \ \ \ \ \ \ \ \ \ \left. +\frac{\beta_{N}}{2}\|W_{N}^kV_{N-1}-V_{N}^{k-1}\|_F^2 + \langle \Lambda_{N}^{k-1},W_{N}^kV_{N-1}-V_{N}^{k-1} \rangle \right\}, \label{Eq:Vn-1k-prox}\\
& V_N^k = \arg\min_{V_N} \left\{\frac{1}{2}\|V_N-Y\|_F^2+\frac{\beta_N}{2}\|W_N^kV_{N-1}^k-V_N\|_F^2 + \langle \Lambda_N^{k-1},W_N^kV_{N-1}^k-V_N \rangle \right\}. \label{Eq:Vnk-prox}
\end{align}
Once $\left(\{W_i^k\}_{i=1}^N, \{V_i^k\}_{i=1}^N \right)$ have been updated, we then update the multipliers $\{\Lambda_i^k\}_{i=1}^N$ parallelly according to the following:
for $ i=1,\ldots, N-1$,
\begin{align}
\label{Eq:Lambdak}
\Lambda_i^k = \Lambda_i^{k-1} + \beta_i (\sigma(W_i^kV_{i-1}^k)-V_i^k), \quad
\Lambda_N^k = \Lambda_N^{k-1} + \beta_N (W_N^kV_{N-1}^k-V_N^{k}).
\end{align}
Based on these, each iterate of ADMM only involves several relatively simpler sub-problems.

It should be mentioned that the suggested update order is actually a technical requirement in the convergence proof (see, Lemma \ref{Lemm:dual-expressed-primal} below), which also appears in the previous work \citep{Wang-ADMM2018}. Moreover, the following local linear approximation is also required to establish Lemma \ref{Lemm:dual-expressed-primal}.

\subsection{Local linear approximation for sub-problems}
\label{sc:ADMM-linearized-trick}

Note that $W_i$-subproblems ($i=1, \ldots, N-1$) involve functions of the following form
\begin{align}
\label{Eq:H-sigmoid}
H_{\sigma}(W;A,B) = \frac{1}{2}\|\sigma(WA)-B\|_F^2,
\end{align}
while $V_j$-subproblems ($j=1,\ldots, N-2$) involve functions of the following form
\begin{align}
\label{Eq:M-sigmoid}
M_{\sigma}(V;\tilde{A},\tilde{B}) = \frac{1}{2}\|\sigma(\tilde{A}V)-\tilde{B}\|_F^2,
\end{align}
where $A, B, \tilde{A}, \tilde{B}$ are four given matrices related to the previous updates. Due to the nonlinearity of the sigmoid activation function, the subproblems are generally difficult to be solved, or at least some additional numerical solvers are required to solve these subproblems. To break such computational hurdle, we adopt the first-order approximations of the original functions presented in \eqref{Eq:H-sigmoid} and \eqref{Eq:M-sigmoid} at the latest updates, instead of themselves, to update the variables, that is,
\begin{align}
H^k_{\sigma}(W;A,B) &:= H_{\sigma}(W^{k-1};A,B)+ \langle (\sigma(W^{k-1}A)-B)\odot \sigma'(W^{k-1}A), (W-W^{k-1})A \rangle \nonumber\\
&+ \frac{h^k}{4}\|(W-W^{k-1})A\|_F^2, \label{Eq:Hk-sigmoid}\\
M^k_{\sigma}(V;\tilde{A},\tilde{B}) &:= M_{\sigma}(V^{k-1};\tilde{A},\tilde{B}) + \langle (\sigma(\tilde{A}V^{k-1})-\tilde{B})\odot \sigma'(\tilde{A}V^{k-1}), \tilde{A}(V-V^{k-1}) \rangle \nonumber\\
& + \frac{\mu^k}{4}\|\tilde{A}(V-V^{k-1})\|_F^2, \label{Eq:Mk-sigmoid}
\end{align}
where $W^{k-1}$ and $V^{k-1}$ are the $(k$-$1)$-th iterate, and $\sigma'(W^{k-1}A)$ and $\sigma'(\tilde{A}V^{k-1})$ represent the componentwise derivatives,
$h^k $ and $\mu^k$ can be specified as the upper bounds of twice of the locally Lipschitz constants of functions $H_{\sigma}$ and $M_{\sigma}$, respectively, shown as
\[
h^k = \mathbb{L}(\|B\|_{\max}), \quad \mu^k = \mathbb{L}(\|\tilde{B}\|_{\max}).
\]
Here, for any given $c\in \mathbb{R}$,
\begin{align}
\label{Eq:Lipschitz-constant}
\mathbb{L}(|c|):= 2L_2(L_0 + |c|) + 2L_1^2
\end{align}
is an upper bound of the Lipschitz constant of the gradient of function $(\sigma(u)-c)^2$ with constants $L_0 = 1, L_1 = \frac{1}{4}$ and $L_2 = \frac{1}{4}$ related to the sigmoid activation $\sigma$.

Henceforth, we call this treatment as the \textbf{local linear approximation (LLA)}, which can be viewed as adopting certain \textit{prox-linear scheme} \citep{Xu-Yin-BCD2013} to update the subproblems of ADMM. Based on \eqref{Eq:Hk-sigmoid} and \eqref{Eq:Mk-sigmoid}, the original updates \eqref{Eq:Wik-orig} of $\{W_i^k\}_{i=1}^{N-1}$  are replaced by
\begin{align}
\label{Eq:Wik-prox}
W_i^k = \arg\min_{W_i} \left\{ \frac{\lambda}{2}\|W_i\|_F^2 + \beta_i H_{\sigma}^k(W_i;V_{i-1}^{k-1}, V_{i}^{k-1} - \beta_i^{-1}\Lambda_i^{k-1})\right\},
\end{align}
and by completing perfect squares and some simplifications, the original updates \eqref{Eq:Vjk-orig} of $\{V_j^k\}_{j=1}^{N-2}$  are replaced by
\begin{align}
\label{Eq:Vjk-prox}
V_j^k = \arg\min_{V_j} \left\{ \frac{\beta_j}{2}\|\sigma(W_j^kV_{j-1}^k) + \beta_j^{-1}\Lambda_j^{k-1} - V_j\|_F^2 + \beta_{j+1} M_{\sigma}^k(V_j;W_{j+1}^{k}, V_{j+1}^{k-1} - \beta_{j+1}^{-1}\Lambda_{j+1}^{k-1})\right\},
\end{align}
with $h_i^k$ and $\mu_j^k$ being specified as follows
\begin{align}
\label{Eq:hik}
h_i^k = {\mathbb L}(\|V_{i}^{k-1} - \beta_i^{-1}\Lambda_i^{k-1}\|_{\max}), \ i= 1,\ldots, N-1,
\end{align}
\begin{align}
\label{Eq:mujk}
\mu_j^k = {\mathbb L}(\|V_{j+1}^{k-1} - \beta_{j+1}^{-1}\Lambda_{j+1}^{k-1}\|_{\max}),\ j=1,\ldots, N-2,
\end{align}
where ${\mathbb L}(\cdot)$ is defined in \eqref{Eq:Lipschitz-constant}. Note that with these alternatives, all the subproblems can be solved with analytic expressions (see, Lemma \ref{Lemm:dual-expressed-primal} in Appendix \ref{app:dual-expressed-primal}).

\subsection{ADMM for deep sigmoid nets}
The ADMM algorithm for DNN training problem \eqref{Eq:dnn-L2-admm-reg} is summarized in Algorithm \ref{alg:ADMM}. As shown in Figure \ref{Fig:squarefun-motivation} (b) and Figure \ref{Fig:sigmoid} (c), ADMM does not suffer from either the issue of gradient explosion or the issue of gradient vanishing caused by the saturation of sigmoid activation and thus can approximate the square function within high precision. The intuition behind ADMM to avoid the issues of gradient explosion and vanishing is that the suggested ADMM does not exactly follow the chain rule as exploited in BackProp and SGD, but introduces the multipliers as certain compensation to eventually fit the chain rule at the stationary point.
From Algorithm \ref{alg:ADMM}, besides the regularization parameter $\lambda$ related to the DNN training model, only the penalty parameters $\beta_i$'s should be tuned. In the algorithmic perspective,  penalty parameters can be regarded as the dual step sizes for the updates of multipliers, which play similar roles as learning rates in SGD. As shown by our experiment results below, the performance of ADMM is not sensitive to  penalty parameters, making the parameters be easy-to-tune. Moreover, by exploiting the LLA, the updates for all variables can be very cheap with analytic expressions (see Lemma \ref{Lemm:dual-expressed-primal} in \Cref{app:dual-expressed-primal}).

As compared to the existing ADMM methods for deep learning \citep{Carreira2014-MAC,Goldstein-ADMM-DNN2016,Kiaee2016-ADMM-sparseCNN,Murdock-ADNN2018}, there are two major differences shown as follows. The first one is that the existing ADMM type methods in deep learning only keep partial nonlinear constraints for the sake of reducing the difficulty of optimization, while the ADMM method suggested in this paper keeps all the nonlinear constraints, and thus our proposed ADMM can come back to the original DNN training model in the sense that its convergent limit fits all the nonlinear constraints as shown in Theorem \ref{Thm:Conv-ADMM-sigmoid} below. To overcome the difficulty from optimization, we introduce an elegant update order and the LLA technique for subproblems. The second one is that most of existing ADMM methods focus on deep ReLU nets, while our proposed ADMM is designed for deep sigmoid nets.

It should be pointed out that the subproblems of the proposed algorithm require inverting matrices at each iteration, which could be expensive. Although there are some practical tricks like warm-start and solving inexactly via doing gradient descent by a fixed number of times to improve the computational efficiency of the proposed ADMM (e.g., in \citep{Liu-AccADMM2021} ), the major focus of this paper is mainly on the development of an effective ADMM method with theoretical guarantees for the training of deep sigmoid nets, and we will consider its practical acceleration in the future.

\begin{algorithm}[t]
{\small
\begin{algorithmic}\caption{ADMM for Deep Sigmoid Nets Training}
\label{alg:ADMM}
\STATE {\bf Samples}: $X:= [x_1, \ldots, x_n] \in \mathbb{R}^{d_0 \times n}$, $Y:= [y_1, \ldots, y_n] \in \mathbb{R}^{d_N \times n}$.
\STATE {\bf Initialization}: $(\{W_i^0\}_{i=1}^N, \{V_i^0\}_{i=1}^N, \{\Lambda_i^0\}_{i=1}^N)$ is set according to \eqref{Eq:initialization-admm}.
$V_0^k \equiv X, \forall k \in \mathbb{N}$.
\STATE {\bf Parameters:} $\lambda>0$, $\beta_i >0, i=1,\ldots, N$.
\smallskip
\FOR{$k=1,\ldots$}
\STATE{$\blacktriangleright$ (Backward Estimation)}
\FOR{$i= N :-1:1$}
\STATE Update $W_N^k$ via \eqref{Eq:Wnk-prox}
and the other $W_i^k$ via \eqref{Eq:Wik-prox}.
\ENDFOR
\STATE{$\blacktriangleright$(Forward Prediction)}
\FOR{$j= 1: N$}
\STATE Update $V_j^k (j=1,\ldots, N-2)$ via \eqref{Eq:Vjk-prox},
$V_{N-1}^k$ via \eqref{Eq:Vn-1k-prox},
and $V_N^k$ via \eqref{Eq:Vnk-prox}.
\ENDFOR
\STATE{$\blacktriangleright$(Updating Multipliers)}
\STATE $\Lambda_i^{k} = \Lambda_i^{k-1} + \beta_i (\sigma(W_i^{k}V_{i-1}^{k})-V_i^{k}), \ i=1,\ldots,N-1$,
\STATE $\Lambda_N^{k} = \Lambda_N^{k-1} + \beta_N (W_N^{k}V_{N-1}^{k}-V_N^{k})$.
\STATE $k \leftarrow k+1$
\ENDFOR
\end{algorithmic}}
\end{algorithm}

\begin{figure}[!t]
\begin{center}
\begin{minipage}[b]{0.49\linewidth}
\centering
\includegraphics*[scale=.45]{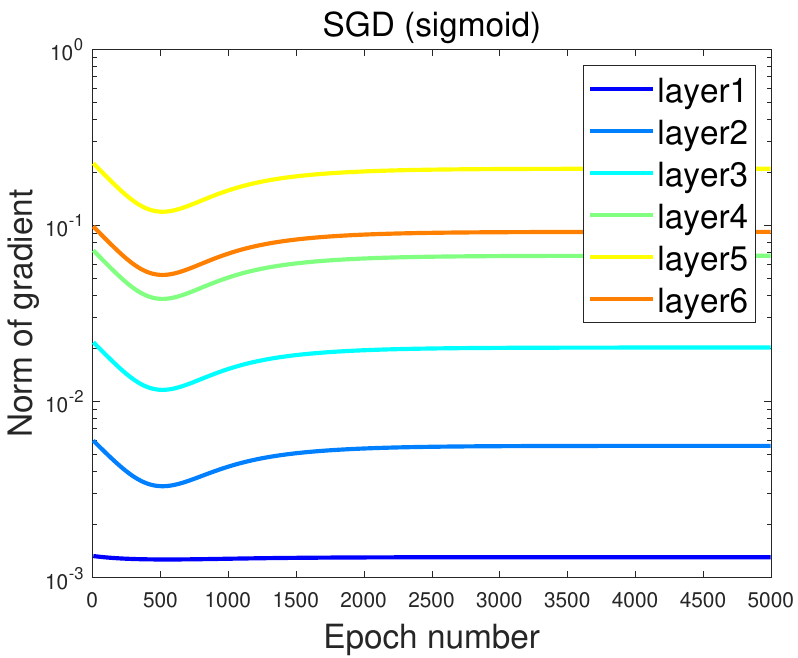}
\centerline{{\small (a) Gradient vanishing of SGD}}
\end{minipage}
\hfill
\begin{minipage}[b]{0.49\linewidth}
\centering
\includegraphics*[scale=.45]{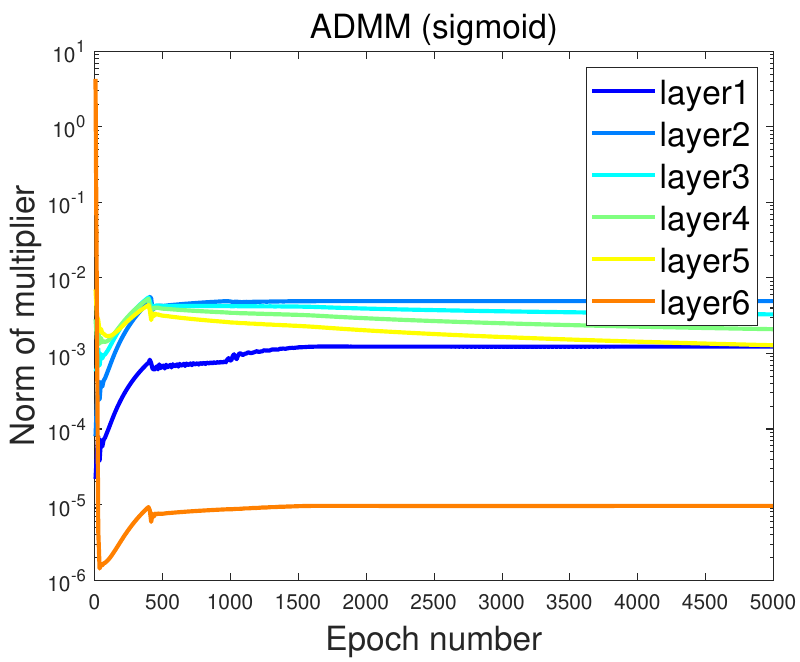}
\centerline{{\small (b) Saturation-avoidance of ADMM}}
\end{minipage}
\hfill
\end{center}
\caption{Gradient vanishing of SGD and saturation-avoidance of ADMM in the training of deep sigmoid nets. The numerical setting is the same as that of Figure \ref{Fig:sigmoid}.}
\label{Fig:squarefun-motivation}
\end{figure}

\subsection{Convergence of ADMM for deep sigmoid nets}
\label{sc:convergence-result}

Without loss of generality, we assume that $X, Y$ and $\{W_i^0\}_{i=1}^N$ are normalized with
$\|X\|_F = 1$, $\|Y\|_F=1$ and $\|W_i^0\|_F = 1$, $i=1,\ldots, N$, and all numbers of hidden layers are the same, i.e., $d_i=d, \ \forall i=1,\ldots,N-1$.
Under these settings, we present the main convergence theorem of ADMM in the following, while
that of ADMM under more general settings is presented in Theorem \ref{Thm:global-generic} in Appendix \ref{app:ADMM-GenericDNN}.

\begin{theorem}
\label{Thm:Conv-ADMM-sigmoid}
Let $\{ \c Q^k:=(\{W_i^k\}_{i=1}^N, \{V_i^k\}_{i=1}^N, \{\Lambda_i^k\}_{i=1}^N)\}$ be a sequence generated by Algorithm \ref{alg:ADMM}. If $2\leq N \leq \sqrt{n}$, $\lambda \geq \tilde{c}N^{\frac{N-3}{2}}(nd)^{\frac{N}{2}-\frac{1}{4}}$ and $\{\beta_i\}_{i=1}^N$ satisfy
\begin{align}
\label{Eq:cond-beta-sigmoid}
\beta_N \geq 3.5, \ \beta_{N-1} \geq 16 \beta_N, \ \beta_i\geq \tilde{c}_1 \beta_{N-1}(Nnd)^{\frac{N-1-i}{2}}, \ i=1,\ldots, N-2
\end{align}
for   some constants $\tilde{c},\tilde{c}_1>0$ independent of $n,N$,
 then we have:
\begin{enumerate}
\item[(a)] the augmented Lagrangian sequence $\{\cL(\cQ^k)\}$ is convergent.

\item[(b)] $\{\cQ^k\}$ converges to a stationary point $\cQ^*:= (\{W_i^*\}_{i=1}^N, \{V_i^*\}_{i=1}^N, \{\Lambda_i^*\}_{i=1}^N)$ of the augmented Lagrangian $\cL$, which is also a KKT point (defined in\eqref{Eq:kkt-cond} below) of problem \eqref{Eq:dnn-L2-admm-reg}, implying that $\{W_i^*\}_{i=1}^N$ is a stationary point of problem \eqref{Eq:dnn-L2-org} with $\lambda' = 2\lambda/n$.

\item[(c)] $\frac{1}{K}\sum_{k=1}^K \|\nabla \cL(\cQ^k)\|_F^2 \rightarrow 0$ at a rate of order ${\cal O}(\frac{1}{K})$.
\end{enumerate}
\end{theorem}

Theorem \ref{Thm:Conv-ADMM-sigmoid} establishes the global convergence of ADMM   to a KKT point at a rate of ${\cal O}(1/K)$. By \eqref{Eq:cond-beta-sigmoid}, the parameters $\{\beta_i\}_{i=1}^N$ increase exponentially fast from the output layer to the input layer. Moreover, by Theorem \ref{Thm:Conv-ADMM-sigmoid}, the regularization parameter $\lambda$ is also required to grow exponentially fast as the depth increases. Back to the original DNN training model \eqref{Eq:dnn-L2-org}, the requirement on the regularization parameter $\lambda'$ is $\lambda'\geq \tilde{c}N^{\frac{N-3}{2}}d^{\frac{N}{2}-\frac{1}{4}}n^{\frac{N}{2}-\frac{5}{4}}$. Particularly, when $N=2$, namely, the neural networks with single hidden layer, then $\lambda' = \tilde{c}\sqrt[4]{d^3/n}$ is a good choice, which implies that the regularization parameter can be small when the sample size $n$ is sufficiently large. Despite these convergence conditions seem a little stringent, by the existing literature \citep{Chui-Lin-Zhou18,Chui-Lin-Zhou19}, the depth of deep   sigmoid nets is usually small, say, 2 or 3 for realizing some important data features in deep learning. Moreover, as shown in the numerical results to be presented in Sections \ref{sc:simulations} and \ref{sc:real-experiments}, a moderately large augmented Lagrangian parameter (say, each $\beta_i=1$) and a small regularization parameter (say, $\lambda=10^{-6}$) are empirically enough for the proposed ADMM. In this case, the KKT point found by ADMM should be close to the optimal solutions to the empirical risk minimization of DNN training.

{\bf Remark 1: KKT conditions.}
Based on \eqref{Eq:ALMFun-L2}, the Karush-Kuhn-Tucker (KKT) conditions of the problem \eqref{Eq:dnn-L2-admm-reg} can be derived as follows. Specifically, let $\{W_i, V_i\}_{i=1}^N$ be an optimal solution of problem \eqref{Eq:dnn-L2-admm-reg}, then there exit multipliers $\{\Lambda_i\}_{i=1}^N$ such that the following hold:
\begin{align}
\label{Eq:kkt-cond}
& 0 = \lambda W_1 + (\Lambda_1 \odot \sigma'(W_1 V_0))V_0^T, \nonumber\\
& 0 = \lambda W_i + (\Lambda_i \odot \sigma'(W_i V_{i-1}))V_{i-1}^T, \ i=2,\ldots, N-1,\nonumber\\
& 0 = \lambda W_N + \Lambda_N V_{N-1}^T, \nonumber\\
& 0 = -\Lambda_i + W_{i+1}^T(\Lambda_{i+1}\odot \sigma'(W_{i+1}V_i)), \ i = 1,\ldots, N-2,\\
& 0 = -\Lambda_{N-1} + W_{N}^T\Lambda_{N}, \nonumber\\
& 0 = -\Lambda_N + (V_N - Y), \nonumber\\
& 0 = \sigma(W_iV_{i-1}) - V_i, \ i= 1, \ldots, N-1, \nonumber\\
& 0 = W_NV_{N-1} - V_N \nonumber
\end{align}
where $V_0 = X$. From \eqref{Eq:kkt-cond}, the KKT point of problem \eqref{Eq:dnn-L2-admm-reg} exactly fits these nonlinear constraints. Moreover, given a KKT point $(\{W_i^*\}_{i=1}^N, \{V_i^*\}_{i=1}^N, \{\Lambda_i^*\}_{i=1}^N)$ of \eqref{Eq:dnn-L2-admm-reg}, substituting the last five equations into the first three equations of \eqref{Eq:kkt-cond} shows that $\{W_i^*\}_{i=1}^N$ is also a stationary point of the original DNN training model \eqref{Eq:dnn-L2-org}.

{\bf Remark 2: More general activations:}
As presented in Theorem \ref{Thm:global-generic} in Appendix \ref{app:ADMM-GenericDNN}, the convergence results in Theorem \ref{Thm:Conv-ADMM-sigmoid} still hold for a general class of smooth activations such as the hyperbolic tangent activation as studied in \citep{Lin-CFN2019}. Actually, the approximation result yielded in Theorem \ref{theorem:app} can be also easily extended to a class of twice differentiable sigmoid-type activations.

\section{Related Work and Discussions}
\label{sc:related-work}
In this section, we present some related works and show the novelty of our studies.

\subsection{Deep sigmoid nets versus deep ReLU nets in approximation}
Deep ReLU nets are the most popular neural networks in deep learning. Compared with deep sigmoid nets, there are commonly three  advantages of deep ReLU nets \citep{Hinton2010-ReLU}. At first, the piecewise linear property makes it easy to compute the derivative to ease the training via gradient-type algorithms. Then, the derivative of ReLU is either $1$ or $0$, which in a large extent alleviates the saturation phenomenon for deep sigmoid nets and particularly the gradient vanishing/explotion issue of the gradient-descent based algorithms for the training of deep neural networks. Finally, $\sigma_{relu}(t)=0$ for $t<0$ enables the sparseness of the neural networks, which coincides with the biological mechanism for neural systems.

Theoretical verification  for the power of depth in deep ReLU nets is a hot topic in deep learning theory.  It stems from the study in \citep{Eldan2016}, where some functions were constructed to be well approximated by deep ReLU nets but cannot be expressed by shallow ReLU nets with similar number of parameters. Then, numerous interesting results on the expressivity and generalization of deep ReLU nets have been provided in \citep{Yarotsky17,Safran-Shamir17,Shaham-Cloninger-Coifman18,Petersen-Voigtlaender18,Schwab2019,Guo-Shi-Lin19,Zhou18,Zhou20a,Chui2020,han2020depth}.
Typically, it was proved in \citep{Yarotsky17} that deep ReLU nets perform at least not worse than the classical linear approaches in approximating smooth functions, and are beyond the capability of shallow ReLU nets. Furthermore, it was also exhibited  in \citep{Shaham-Cloninger-Coifman18} that deep ReLU nets can  extract the manifold structure of the input space and the smoothness of the target functions simultaneously.

The problem is, however, that there are frequently  too many hidden layers for deep ReLU nets to extract data features. Even for approximating the extremely simple square function, Lemma \ref{Lemma:ReLU-square gate for infinite} requires $\log( \varepsilon^{-1})$ depth, which is totally different from deep sigmoid nets. Due to its  infinitely differentiable property, sigmoid function is the most popular activation for shallow nets \citep{Pinkus1999}. The universal approximation property of shallow sigmoid nets has been verified in \citep{Cybenko1989} for thirty years. Furthermore,  \citep{Mhaskar93,Mhaskar1996} showed that the approximation capability of shallow sigmoid nets is at least not worse than that of polynomials. However, there are also several bottlenecks for shallow sigmoid nets in embodying the locality \citep{Chui-Li-Mhaskar94}, extracting the rotation-invariance \citep{Chui-Lin-Zhou19} and producing sparse estimators \citep{LinH2017}, which show the necessity to deepen the neural networks. Different from deep ReLU nets, adding only a few hidden layers can significantly improve the approximation capability of shallow sigmoid nets. In particular, deep sigmoid nets with two hidden layers are capable of providing localized approximation \citep{Chui-Li-Mhaskar94}, reflecting the spatially sparseness \citep{Lin19} and embodying the rotation-invariance \citep{Chui-Lin-Zhou19}.

In a nutshell, as shown in Table \ref{ReLU_fea}, it was proved in the existing literature that any function  expressible  for deep ReLU nets can also be  well approximated by deep sigmoid nets with fewer hidden layers and free parameters. Our Theorem \ref{theorem:app} partly reveals the reason for such a phenomenon in the sense that ReLU can be well approximated by sigmoid nets but not vice-verse.

\subsection{Algorithms for DNN training}

In order to address the choice of learning rate in SGD, there are many variants of SGD incorporated with adaptive learning rates called \textit{adaptive gradient} methods.
Some important adaptive gradient methods are Adagrad \citep{Duchi-adagrad-2011}, Adadelta \citep{Zeiler-adadelta12}, RMSprop \citep{Hinton-RMSProp-2012}, Adam \citep{Kingma2015}, and AMSGrad \citep{Reddi-Amsgrad-iclr2018}. Although these adaptive gradient methods have been widely used in deep learning, there are few theoretical guarantees when applied to the deep neural network training, a highly nonconvex and possibly nonsmooth optimization problem \citep{Wu-adagrad19}. Regardless the lack of theoretical guarantees of the existing variants of SGD, another major flaw is that they may suffer from the issue of gradient explosion/vanishing  \citep{Goodfellow-et-al-2016}, essentially due to the use of BackProp \citep{Hinton-BP1986} for updating the gradient during the iteration procedure.

To address the issue of gradient vanishing, there are some tricks that focus on either the design of the network architectures such as ResNets \citep{He-ResNet-2016} or the training procedure such as the batch normalization \citep{Ioffe-Szegedy-BN15} and weight normalization \citep{Salimans-Kingma-WN16}. Besides these tricks, there are many works in the perspective of algorithm design, aiming to propose some alternatives of SGD to alleviate the issue of gradient vanishing. Among these alternatives, the so called  \textit{gradient-free} type methods have recently attracted rising attention in deep learning since they may in principle alleviate this issue by their gradient-free natures, where the alternating direction method of multipliers (ADMM) and block coordinate descent (BCD) methods are two   most popular ones (see, \citep{Carreira2014-MAC,Goldstein-ADMM-DNN2016,Kiaee2016-ADMM-sparseCNN,Yang-ADMM-net2016,Murdock-ADNN2018,Gotmare2018,Zhang-BCD-NIPS2017,Gu-BCD2018,Lau-BCD-2018,Zeng-BCD19}).
Besides the gradient-free nature, another advantage of both ADMM and BCD is that they can be easily implemented in a distributed and parallel manner, and thus are capable of solving distributed/decentralized large-scale problems \citep{Boyd-DADMM2011}.

In the perspective of constrained optimization, all the BackProp (BP), BCD and ADMM can be regarded as certain Lagrangian methods or variants for the nonlinearly constrained formulation of DNN training problem. In \citep{LeCun-BP-1988}, BP was firstly reformulated as a Lagrangian multiplier method. The fitting of nonlinear equations motivated by the forward pass of the neural networks plays a central role in the development of BP. Following the Lagrangian framework, the BCD methods for  DNN training proposed by \citep{Zhang-BCD-NIPS2017,Lau-BCD-2018,Zeng-BCD19,Gu-BCD2018} can be regarded as certain Lagrangian relaxation methods without requiring the exact fitting of nonlinear constraints.
Unlike in BP, such nonlinear constraints are directly lifted as quadratic penalties to the objective function in BCD, rather than involving these nonlinear constraints with Lagrangian multipliers. However, such a lifted treatment of nonlinear constraints in BCD as penalties suffers from an inconsistent issue in the sense that the solution found by BCD cannot converge to a solution satisfying these nonlinear constraints. To tackle this issue, ADMM, a primal-dual method based on the augmented Lagrangian by introducing the nonlinear constraints via Lagrangian multipliers, enables a convergent sequence satisfying the nonlinear constraints. Therefore, ADMM attracted rising attention in deep learning with various implementations \citep{Carreira2014-MAC,Goldstein-ADMM-DNN2016,Kiaee2016-ADMM-sparseCNN,Yang-ADMM-net2016,Gotmare2018,Murdock-ADNN2018}.
However, most of the existing ADMM type methods in deep learning only keep partial nonlinear constraints for the sake of reducing the difficulty of optimization, and there are few convergence guarantees \citep{Gao-Goldfarb-ADMM20}.

\subsection{Convergence of ADMM and challenges}

Most results in the literature on the convergence of nonconvex ADMM focused on linear constrained optimization problems (e.g. \citep{Hong-ADMM-2016,Wang-ADMM2018}). Following the similar analysis of \citep{Wang-ADMM2018}, \citep{Gao-Goldfarb-ADMM20} extended the convergence results of ADMM to multiaffine constrained optimization problems. In the analysis of \citep{Hong-ADMM-2016,Wang-ADMM2018,Gao-Goldfarb-ADMM20}, the separation of a special block of variables is crucial for the convergence of ADMM in both linear and multiaffine scenarios. Following the notations of \citep{Wang-ADMM2018}, the linear constraint considered in \citep{Wang-ADMM2018} is of the form
\begin{align}
\label{Eq:ADMM-linear}
{\bf A}{\bf x} + B{y} =0
\end{align}
where ${\bf x}:= (x_0,\ldots, x_p)$ includes $p+1$ blocks of variables, $y$ is a special block of variables, ${\bf A}:= [A_0,\ldots, A_p]$ and $B$ are two matrices satisfying
${\mathrm{Im}({\bf A}) \subseteq \mathrm{Im}(B)}$, where $\mathrm{Im}(\cdot)$ returns the image of a matrix. Similarly, \citep{Gao-Goldfarb-ADMM20} extended \eqref{Eq:ADMM-linear} to multiaffine constraint of the form, ${\cal A}(x_1,x_2) + {\cal B}(y)=0,$ where ${\cal A}$ and ${\cal B}$ are respectively some multiaffine and linear maps satisfying ${\mathrm{Im}({\cal A}) \subseteq \mathrm{Im}({\cal B})}$. Leveraging this special block variable $y$, the dual variables (namely, multipliers) is expressed solely by $y$
\citep[Lemma 3]{Wang-ADMM2018}, and the amount of dual ascent part is controlled by the amount of descent part brought by the primal $y$-block update \citep[Lemma 5]{Wang-ADMM2018}.
Together with the descent quantity arisen by the ${\bf x}$-block update, the total progress of one step ADMM update is descent along the augmented Lagrangian. Such a technique is in the core of   analysis in \citep{Wang-ADMM2018} and \citep{Gao-Goldfarb-ADMM20} to deal with some multiaffine constraints in deep learning.

However, in a general formulation of ADMM for DNN training (e.g.  \eqref{Eq:dnn-L2-admm-reg}), it is impossible to separate such a special variable block ${y}$ satisfying these requirements. Let us take a three-layer neural network for example. Let ${\cal W}:= \{W_i\}_{i=1}^3$ be the weight matrices of the neural network, and ${\cal V}:= \{V_i\}_{i=1}^3$ be the response matrices of the neural network and $X$ be the input matrix, then the nonlinear constraints are of the following form,
\begin{subequations}
\label{Eq:DNN-constraints}
    \begin{align}
        \label{Eq:DNN-constraints-a}
        \sigma(W_1X)-V_1 &= 0,\\
        \label{Eq:DNN-constraints-b}
        \sigma(W_2V_1) - V_2 &= 0,\\
        \label{Eq:DNN-constraints-c}
        W_3V_2 - V_3 &= 0,
    \end{align}
\end{subequations}
where $\sigma$ is the sigmoid activation. Note that in \eqref{Eq:DNN-constraints-b} and \eqref{Eq:DNN-constraints-c}, $W_2$ is coupled with $V_1$ and $W_3$ is coupled with $V_2$, respectively, so none of these four variable blocks can be separated from the others. Although $W_1$ in \eqref{Eq:DNN-constraints-a} and $V_3$ in \eqref{Eq:DNN-constraints-c} can be separated, the image inclusion constraint above is not satisfied. Therefore, one cannot exploit the structure in \citep{Wang-ADMM2018,Gao-Goldfarb-ADMM20} to study such constraints in deep learning.

\subsection{Key stones to the challenges and main idea of proof}
\label{sc:keystone-proofidea}

In order to address the challenge of such nonlinear constraints $\sigma(W_iV_{i-1})-V_i=0$,   we introduce a \textit{local linear approximation} (LLA) technique. Let us take \eqref{Eq:DNN-constraints} for example to illustrate this idea. The most difficult block of variable is $V_1$ which involves two constraints, namely, a linear constraint in \eqref{Eq:DNN-constraints-a}, and a nonlinear constraint in \eqref{Eq:DNN-constraints-b}. Now we fix $W_1, W_2$ and $V_2$ as the previous updates, say $W_1^0, W_2^0$ and $V_2^0$, respectively. For the update of $V_1$-block, we keep the linear constraint, but relax the nonlinear constraint with its linear approximation at the previous update $V_1^0$,
\begin{align}
&\sigma(W_2^0V_1^0) - V_2^0 + \sigma'(W_2^0V_1^0)\odot \left[W_2^0(V_1 - V_1^0)\right] \approx 0, \label{Eq:linearized-V1}
\end{align}
assuming the differentiability of activation function $\sigma$. The other blocks can be handled in a similar way. Taking $W_1$ block for example, we relax the related nonlinear constraint via its linear approximation at the previous update $W_1^0$, namely, $\sigma(W_1^0X)-V_1^0 + \sigma'(W_1^0X)\odot ((W_1 - W_1^0)X) \approx 0.$ The operations of LLA on the nonlinear constraints can be regarded as applying certain \textit{prox-linear updates} \citep{Xu-Yin-BCD2013} to replace the subproblems of ADMM involving nonlinear constraints as shown in Section \ref{sc:ADMM-linearized-trick}.

To make such a local linear approximation valid, intuitively one needs: (a) the activation function $\sigma$ is smooth enough; and (b) the linear approximation occurs in a small enough neighbourhood around the previous updates. Condition (a) is mild and naturally satisfied by the sigmoid type activations. But condition (b) requires us to introduce a new Lyapunov function defined in \eqref{Eq:def-hatL} by adding to the original augmented Lagrangian a proximal control between $V_i$ and its previous updates. Equipped with such a Lyapunov function, we are able to show that an auxiliary sequence converges to a stationary point of the new Lyapunov function (see Theorem \ref{Thm:globconv-hatQk} below), which leads to the convergence of the original sequence generated by ADMM (see Theorem \ref{Thm:Conv-ADMM-sigmoid} in Section \ref{sc:convergence-result}). Specifically, denote $\{\hcQ^k\}$ as
\begin{align}
\label{Eq:def-hatQ}
\hcQ^k:= (\cQ^k, \{\hat{V}_i^k\}_{i=1}^N),
\end{align}
with $\hat{V}_i^k:= V_i^{k-1}$ for $i=1,\ldots, N$ and $k \geq 1$,
and $\hcL(\hcQ^k)$ as
\begin{align}
\label{Eq:def-hatL}
\hcL(\hcQ^k):= \cL(\cQ^k) + \sum_{i=1}^N \xi_i \|V_i^k - \hat{V}^k_i\|_F^2
\end{align}
for some positive constant $\xi_i>0$ ($i=1,\ldots, N$) specified later in Appendix \ref{app:proof-sufficient-descent}. Then we state the convergence of $\{\hcQ^k\}$ as follows.

\begin{theorem}[Convergence of $\{\hcQ^k\}$]
\label{Thm:globconv-hatQk}
Under conditions of Theorem \ref{Thm:Conv-ADMM-sigmoid}, we have:
\begin{enumerate}
\item[(a)] $\hcL(\hcQ^k)$ is convergent.
\item[(b)] $\hcQ^k$ converges to some stationary point $\hcQ^*$ of $\hcL$.
\item[(c)] $\frac{1}{K}\sum_{k=1}^K \|\nabla \hcL(\hcQ^k)\|_F^2 \rightarrow 0$ at a ${\cal O}(\frac{1}{K})$ rate.
\end{enumerate}
\end{theorem}

Theorem \ref{Thm:globconv-hatQk} presents the function value convergence and sequence convergence to a stationary point at a ${\cal O}(1/K)$ rate of the auxiliary sequence $\{\hcQ^k\}$. By the definitions \eqref{Eq:def-hatQ} and \eqref{Eq:def-hatL} of $\hcQ^k$ and $\hcL$ , Theorem \ref{Thm:globconv-hatQk} directly implies Theorem \ref{Thm:Conv-ADMM-sigmoid}. As shown by the proofs in Appendix \ref{app:proof-main-theorem}, the claims in Theorem \ref{Thm:globconv-hatQk} also hold under the more general assumptions for Theorem \ref{Thm:global-generic} in Appendix \ref{app:ADMM-GenericDNN}. In Theorem \ref{Thm:globconv-hatQk}, we only give the convergence guarantee for the proposed ADMM. It would be interesting to derive the convergence rate to highlight the role of algorithmic parameters. We will keep in study and report the result in future work.

Our main idea of proof for Theorem \ref{Thm:globconv-hatQk} can be summarized as follows: we firstly establish a \textit{sufficient descent} lemma along the new Lyapunov function (see Lemma \ref{Lemm:suff-descent} in \Cref{app:sufficient-descent-lemma}), then show a \textit{relative error} lemma (see Lemma \ref{Lemm:bound-grad} in \Cref{app:proof-relative-error}), and later verify the \textit{Kurdyka-{\L}ojasiewicz property} \citep{Lojasiewicz-KL1993,Kurdyka-KL1998}  (see Lemma \ref{Lemm:KL-property} in \Cref{app:Kurdyka-Lojasiewicz-inequality}) and the \textit{limiting continuity} property of the new Lyapunov function by Assumption \ref{Assump:activ-fun},
and finally establish Theorem \ref{Thm:globconv-hatQk} via following the analysis framework formulated in \cite[Theorem 2.9]{Attouch2013}. In order to prove Lemma \ref{Lemm:suff-descent}, we prove the following three lemmas, namely, a one-step progress lemma (see Lemma \ref{Lemm:descent-two-iterates} in \Cref{app:proof-change-quantity}), a dual-bounded-by-primal lemma (see Lemma \ref{Lemm:dual-controlled-primal} in \Cref{app:proof-dual-by-primal}), and a boundedness lemma (see Lemma \ref{Lemm:boundedness-seq} in \Cref{app:proof-boundedness-sequence}), while to prove Lemma \ref{Lemm:bound-grad}, besides Lemmas \ref{Lemm:dual-controlled-primal} and \ref{Lemm:boundedness-seq}, we also use the Lipschitz continuity of the activation and its derivative by Assumption \ref{Assump:activ-fun} in Appendix \ref{app:ADMM-GenericDNN}. The proof sketch can be illustrated by Figure \ref{Fig:proof-sketch}.

\begin{figure}[h]
    \centering
    \includegraphics[scale=0.5]{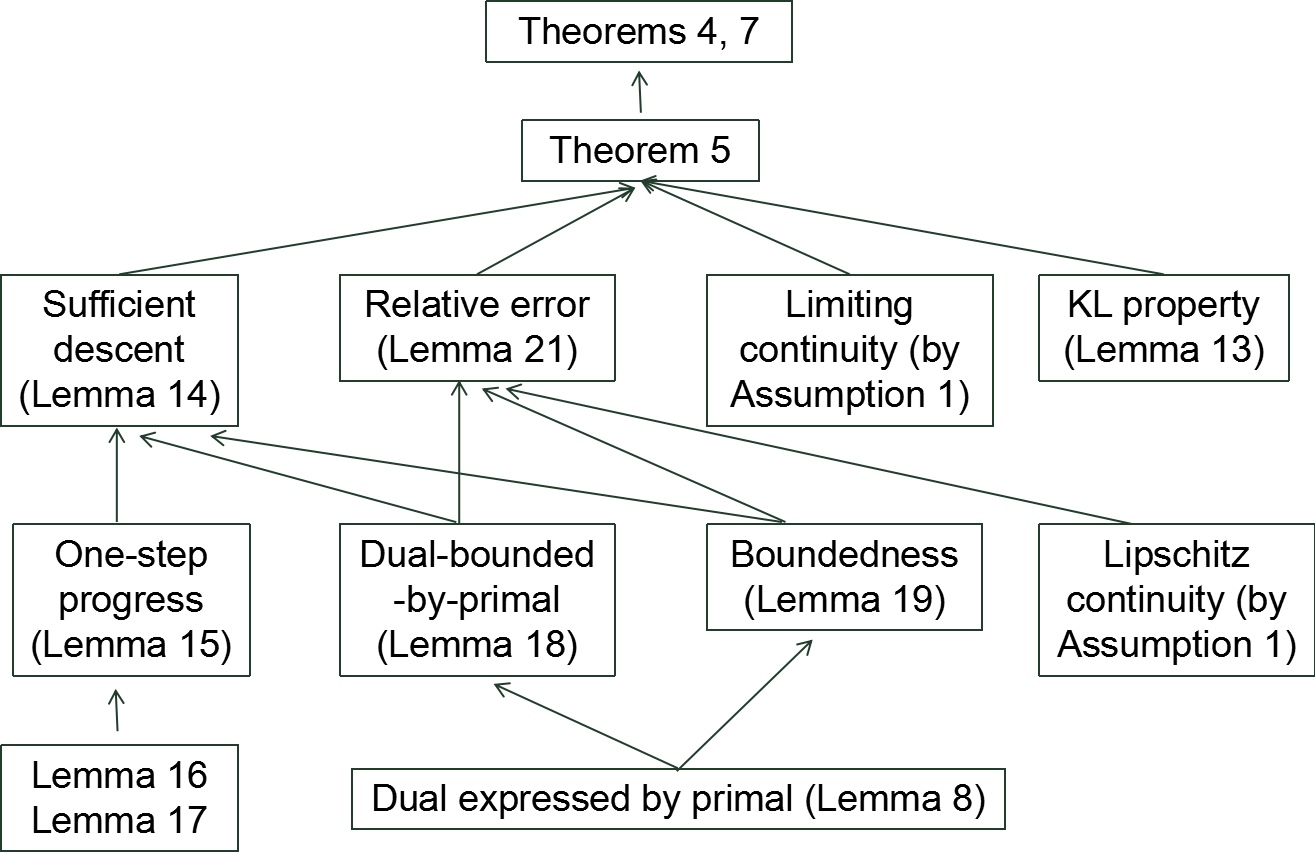}
    \caption{Proof sketch of convergence of ADMM.}
    \label{Fig:proof-sketch}
\end{figure}

According to Figure \ref{Fig:proof-sketch}, we show the boundedness of the sequence before the establishment of the sufficient descent lemma (i.e., Lemma \ref{Lemm:suff-descent}).
Such a proof procedure is different from the existing ones in the literature (say, \citep{Wang-ADMM2018}), where a sequence boundedness is usually implied by firstly showing the (sufficient) descent lemma \citep[Lemma 6]{Wang-ADMM2018}.

\section{Toy Simulations}
\label{sc:simulations}

In this section, we conduct a series of simulations to show the effectiveness of the proposed ADMM in approximating and learning  some natural functions including the square function, product gate, a piecewise $L_1$ radial function, and a smooth $L_2$ radial function, which    play important roles in reflecting some commonly used data features \citep{Safran-Shamir17,Shaham-Cloninger-Coifman18,Chui-Lin-Zhou19,Guo-Shi-Lin19}. In particular, we provide empirical studies to show that these important natural functions can be numerically well approximated or learned by the proposed ADMM-sigmoid pair. Furthermore we also show that the proposed ADMM-sigmoid pair is stable to its algorithmic hyper-parameters, via comparing to the popular deep learning optimizers including the vanilla SGD, SGD with momentum (called SGDM for short henceforth) and Adam \citep{Kingma2015}. There are four experiments concerning  approximation and learning tasks: \textbf{(a)}  approximation of square function, \textbf{(b)}  approximation of product gate, \textbf{(c)} learning of a piecewise $L_1$ radial function, and \textbf{(d)}  learning of a smooth $L_2$ radial function. All  numerical experiments were carried out in Matlab R2015b environment running Windows 10, Intel(R) Xeon(R) CPU E5-2667 v3 @ 3.2GHz 3.2GH.  The codes are available at \url{https://github.com/JinshanZeng/ADMM-DeepLearning}.

\subsection{Experimental settings}
\label{sc:simul-exp-setting}

In all our experiments, we use deep fully connected neural networks with different depths and widths. Throughout the paper,  the depth  and width  of deep neural networks are respectively the number of hidden layers and number of  neurons in each hidden layer. For simplicity, we only consider deep neural networks with the same width for all the hidden layers. We consider both   deep sigmoid nets and deep ReLU nets in the simulation. Motivated by the existing literature \citep{Guo-Shi-Lin19}, the depth required for deep ReLU nets is in general much more than that for deep sigmoid nets in the aforementioned approximation or learning tasks. For the fairness of comparison, we consider deep ReLU nets with more hidden layers, i.e., the maximal depth of deep ReLU nets is 20 while that of deep sigmoid nets is only 5 or 6, as presented in Table \ref{Tab:exp-setting}. Besides the vanilla SGD-ReLU pair ({\it SGD (ReLU)} for short), we also consider SGD-sigmoid pair ({\it SGD (sigmoid)} for short), SGDM-ReLU pair  ({\it SGDM (ReLU)} for short), and Adam-ReLU pair ({\it Adam (ReLU)} for short) as the competitors. Similarly, we denote by {\it ADMM (sigmoid)} the proposed ADMM-sigmoid pair.

For each experiment, our purposes are mainly two folds: excellent approximation or learning performance, and stability with respect to  initialization schemes and  penalty parameters with appropriate neural network structures for the proposed ADMM-sigmoid pair. For the first purpose, we consider  deep neural networks with different depths and widths as presented in Table \ref{Tab:exp-setting}. Moreover, for ADMM,  we empirically set the regularization parameter $\lambda = 10^{-6}$ and the augmented Lagrangian parameters $\beta_i$'s as the same $1$, while for  SGD methods, we empirically utilize the step exponential decay (or, called \textit{geometric decay}) learning rate schedule with the decay exponent $0.95$ for every 10 epochs. For SGDM and Adam, we use the default settings as presented in Table \ref{Tab:exp-setting}. The number of epochs in all experiments is empirically set as 2000. The specific settings of these experiments are presented in Table \ref{Tab:exp-setting}.

\begin{table*}
\caption{Experimental settings for toy simulations. Sample sizes for  approximation tasks are set as 1000, and training and test samples sizes for  learning tasks are set as 1000 respectively. The notation $[a:b]$ is denoted by the set $\{a,a+1,\ldots,b\}$ for two integers $a, b$. $\dag$ In the case of  learning  the $L_2$ radial function,  ranges of depth of   deep sigmoid and ReLU nets are $[2,6]$ and $[2,20]$, respectively.}
\vspace{-0.5cm}
\tiny
\begin{center}
\begin{tabular}{|c|c|c|c|c|c|c|c|c|}\hline
\multirow{2}*{task} &\multirow{2}*{functions}  & \multicolumn{2}{|c|}{deep fully-connected NNs} &\multicolumn{2}{|c|}{SGDs (sigmoid/ReLU), SGDM} & SGDM & \multirow{2}*{Adam} & ADMM \\
\cline{3-6}
~&~  & width & depth & learning rate (lr) & batch size & (momentum) & ~ & $(\lambda,\beta)$ \\\hline
\multirow{2}*{Approx.} &$x^2$  & $20\times [1:5]$ & $[1,5]$ (sigmoid),  & ~  & ~  & ~ & lr: 0.001 & ~\\
\cline{2-3}
~&$uv$  &$60\times [1:5]$ &$[1,20]$ (ReLU) & $0.1\times 0.95^k$ & 50 & 0.5 & $\beta_1$: 0.9 & $(10^{-6},1)$ \\
\cline{1-3}
\multirow{2}*{Learn.} & $(\|x\|_1 - 1)_+$  &$10\times[1:6]$ &~ &per 10 epochs, & ~ & ~ & $\beta_2$: 0.999 &~  \\
\cline{2-4}
~ & $g(|x|^2)$  &$100\times[1:5]$ & $[2,6]$, $[2,20]$ $^\dag$ &~ & ~ & ~ & $\epsilon$: 1e-8 & ~ \\\hline
\end{tabular}
\end{center}
\label{Tab:exp-setting}
\vspace{-.8cm}
\end{table*}

For the second purpose, we consider different regularization and penalty parameters $(\lambda \in \{10^{-6}, 10^{-5},10^{-4}, 10^{-3},10^{-2},10^{-1}\}, \beta \in \{0.01, 0.1, 0.5, 1, 5, 10, 100\})$, as well as several existing initialization schemes for ADMM under the optimal neural network structure determined by the first part. Particularly, we consider the following four types of  schemes yielding  six typical initializations:
\begin{enumerate}
\item[(1)]
\textbf{LeCun random initialization} \citep{LeCun-init98b}: the components of the weight matrix $W_l$ at the $l$-th layer are generated i.i.d. according to some random distribution with zero mean and variance $\frac{1}{d_{l-1}}$, $l =1,\ldots, N$. Particularly, we consider two special \textit{LeCun random} initialization schemes generated respectively according to the uniform and Gaussian distributions, i.e., $W_{l} \sim U([-\sqrt{\frac{3}{d_{l-1}}}, \sqrt{\frac{3}{d_{l-1}}}])$ (\textit{LeCun-Unif} for short) and $W_{l} \sim {\cal N}(0, \frac{1}{d_{l-1}})$ (\textit{LeCun-Gauss} for short).

\item[(2)]
\textbf{Random orthogonal initialization} \citep{Saxe-orth-init14}: the weight matrix $W$ is set as some random orthogonal  matrix such that $W^TW = {\bf I}$ or $WW^T = {\bf I}$. We call it \textit{Orth-Unif} (or \textit{Orth-Gauss}) for short if the random matrix is generated i.i.d. according to the uniform random (or, Gaussian) distribution.

\item[(3)]
\textbf{Xavier initialization} \citep{Glorot-Bengio10}: each $W_{l} \sim U([-\sqrt{\frac{6}{d_{l-1}+d_{l}}}, \sqrt{\frac{6}{d_{l-1}+d_{l}}}])$, $l=1,\ldots, N$.

\item[(4)]
\textbf{MSRA initialization} \citep{He-msra-init15}: $W_l \sim {\cal N}(0,\frac{2}{d_l})$, $l=1,\ldots, N-1$, and $W_N \sim {\cal N}(0,\frac{1}{d_N})$ since there is no ReLU activation for the last layer.
\end{enumerate}
The default settings for  initial threshold vectors in the above initialization schemes are set to be 0. For each group of parameters, we run 20 trails for average. Specifically, for the approximation tasks, the performance of an algorithm is measured by the \textit{approximation error}, defined as the average of these 20 trails's mean square errors, while for the learning tasks, the performance of an algorithm is measured by the \textit{test error}, defined as the average of the mean square errors of these trails over the test data.

\subsection{Approximation of square function}
\label{sc:simul-sqfun}

In these experiments, we consider the performance of   the ADMM-sigmoid pair in approximating the univariate square function, that is, $f(x)=x^2$ on $[-1,1]$. The specific experimental settings can be found in Table \ref{Tab:exp-setting}.

{\bf A. Approximation performance of ADMM.} Experiment results over the best neural network structures are presented in Table \ref{Tab:squarefun-result}, and trends of  approximation errors with respect to the depth  are shown in Figure \ref{Fig:squarefun-depth}. From Table \ref{Tab:squarefun-result}, the  ADMM-SGD pair can approximate the square function within very high precision, i.e., in the order of $10^{-9}$, which is slightly better than that of competitors for deeper ReLU nets, and is much better than  the SGD-sigmoid pair with the same depth. Specifically,   optimal depths for SGD (ReLU), SGDM (ReLU) and Adam (ReLU) are 18, 15, 15, respectively, while the optimal depth for ADMM (sigmoid) is only 2, which  matches the theoretical results in approximation theory, as shown in \citep[Proposition 2]{Chui-Lin-Zhou19}. In terms of running time, ADMM (sigmoid) with   optimal network structures is generally faster than the  SGDM (ReLU) and Adam (ReLU)  with optimal network structures as presented in the third row of Table \ref{Tab:squarefun-result}, mainly due to the  depth  required for ADMM (sigmoid) is much less  than  those for deep ReLU nets   SGD (ReLU), SGDM (ReLU) and Adam (ReLU).
Moreover, according to Figure \ref{Fig:squarefun-depth},   ADMM (sigmoid) can yield  high approximation precision with less layers than the competitors. These experimental results demonstrate that the proposed ADMM  can embody the advantage of deep sigmoid nets on approximating the square function, as pointed out in the existing theoretical literature \citep{Chui-Lin-Zhou19}.

\begin{table*}
\caption{Experimental results of different algorithms in approximating $f(x)=x^2$. The standard derivation of the approximation error is presented in the parentheses. The running time is recorded in seconds. The depth and width of the optimal network structure in terms of approximation error is presented in the last row. }
\vspace{-0.5cm}
\footnotesize
\begin{center}
\begin{tabular}{|c|c|c|c|c|c|}\hline
Algorithm     & SGD (ReLU)       & SGDM (ReLU)      & Adam (ReLU)      & SGD (sigmoid)     & ADMM (sigmoid) \\\hline
Approx. Error    & 5.34e-8(2.34e-8) & 3.95e-8(1.25e-8) & 3.33e-8(1.46e-8) & 2.46e-4(1.69e-4)  &{\bf 2.53e-9(1.18e-9)} \\\hline
Run Time (s)      & 26.99            & 41.35            & 38.26            & 3.45              &9.47 \\\hline
(depth, width) & (18,100)        & (15,100)         & (15,80)          & (2,60)            & (2,100)\\\hline
\end{tabular}
\end{center}
\label{Tab:squarefun-result}
\end{table*}

\begin{figure}[!t]
\begin{minipage}[b]{0.49\linewidth}
\centering
\includegraphics*[scale=.48]{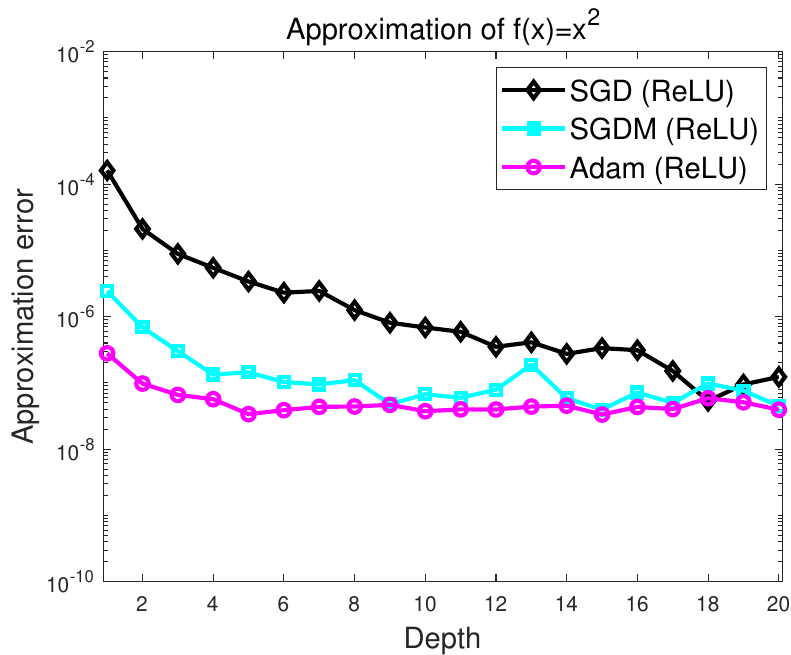}
\centerline{{\small (a) Deep ReLU nets}}
\end{minipage}
\hfill
\begin{minipage}[b]{0.49\linewidth}
\centering
\includegraphics*[scale=.48]{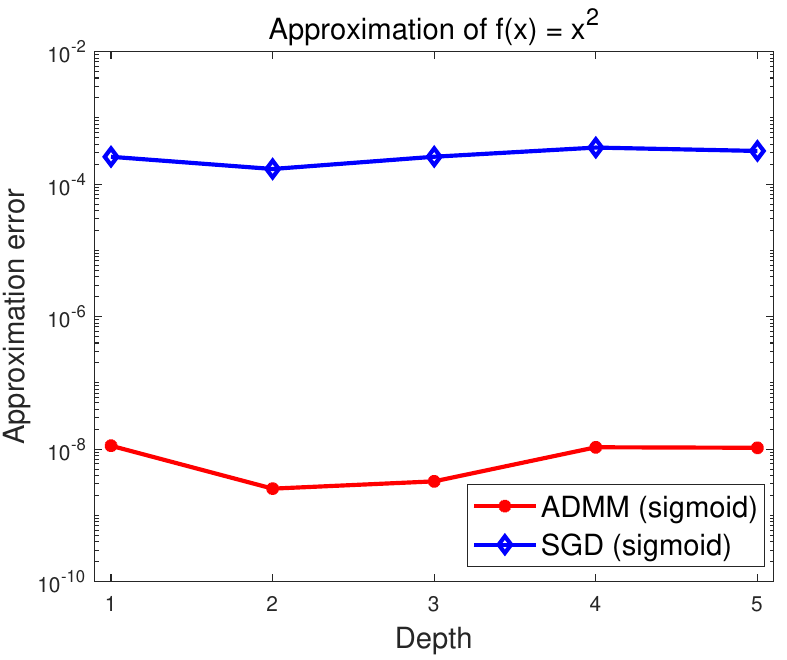}
\centerline{{\small (b) Deep sigmoid nets}}
\end{minipage}
\hfill
\caption{Effect of the depth of neural networks in approximating the square function.
}
\label{Fig:squarefun-depth}
\end{figure}

{\bf B. Effect of parameters for ADMM.}  There are mainly two parameters for the proposed ADMM, i.e., the  model parameter  $\lambda$ (also called as the regularization parameter) and the algorithm parameter $\beta$ involved in the augmented Lagrangian (also called as the penalty parameter). In this experiment, we consider the performance of ADMM in approximating the univariate square function with different model and algorithmic  parameters, under the optimal neural networks, i.e., deep fully-connected neural networks with depth 2 and width 100. Specifically, the regularization and penalty parameters vary from $\{10^{-6}, 10^{-5},10^{-4}, 10^{-3},10^{-2},10^{-1}\}$ and $\{0.01,0.1,0.5,1,5,10,100\}$, respectively. The approximation errors of ADMM with these parameters are shown in Figure \ref{Fig:squarefun-para-admm}(a). From Figure \ref{Fig:squarefun-para-admm}(a),  considering the penalty parameter $\beta$, ADMM with $\beta=1$ achieves the best performance in most cases, as also observed in the experiments later. Thus, in practice, we can empirically set the penalty parameter $\beta$ as $1$. Since $\lambda$ is a model parameter, it usually has significant effect on the performance of the proposed ADMM. From Figure \ref{Fig:squarefun-para-admm}(a), a small regularization parameter (say, $\lambda = 10^{-6}$) is sufficient to yield an ADMM solver with  high approximation precision.

\begin{figure}[!t]
\begin{minipage}[b]{0.49\linewidth}
\centering
\includegraphics*[scale=.5]{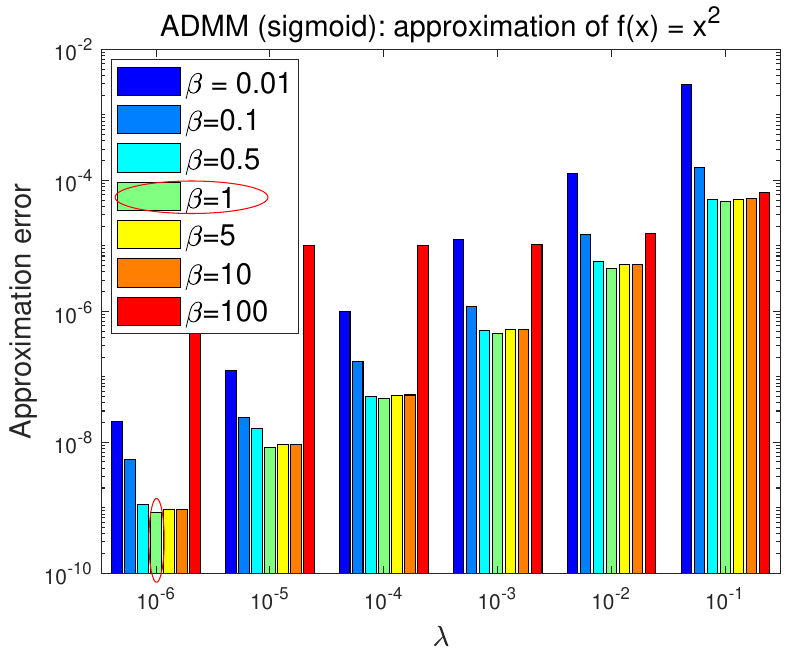}
\centerline{{\small (a) Effect of parameters of ADMM}}
\end{minipage}
\hfill
\begin{minipage}[b]{0.49\linewidth}
\centering
\includegraphics*[scale=.5]{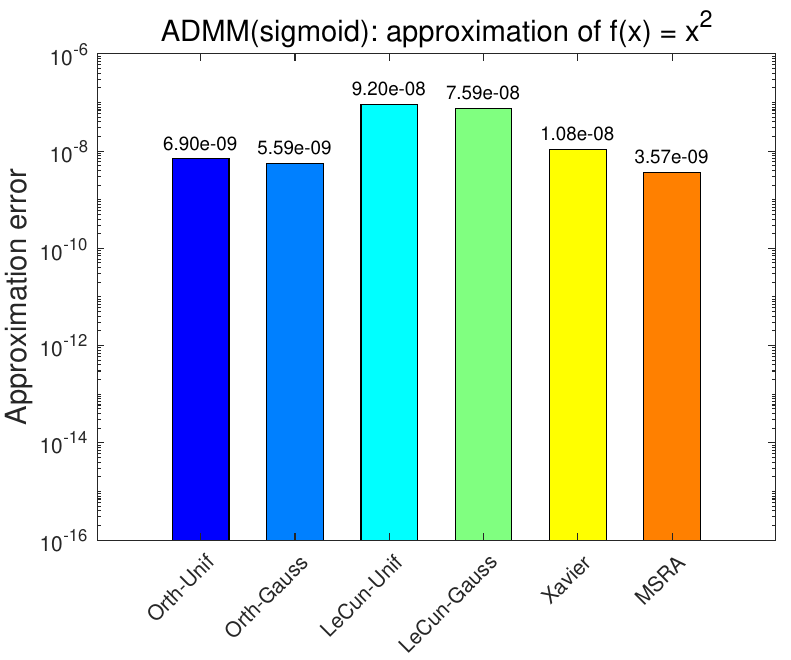}
\centerline{{\small (b) Stability of initialization}}
\end{minipage}
\hfill
\caption{Effect of parameters for ADMM in approximating the univariate square function.
}
\label{Fig:squarefun-para-admm}
\end{figure}

{\bf C. Effect of initial schemes.} Besides the MSRA initialization \citep{He-msra-init15}, there are some other commonly used initial schemes such as the random orthogonal initializations \citep{Saxe-orth-init14}, LeCun random initializations \citep{LeCun-init98b}, and Xavier initialization \citep{Glorot-Bengio10}. Under the optimal parameter settings presented in Table \ref{Tab:squarefun-result}, the performance of the ADMM-sigmoid pair with different initialization schemes is presented in Figure \ref{Fig:squarefun-para-admm}(b). From Figure \ref{Fig:squarefun-para-admm}(b), the proposed ADMM performs  well for all the initialization schemes. This demonstrates that the proposed ADMM is stable to the initial scheme.

\subsection{Approximation of product gate}
\label{sc:simul-prodgate}

In this subsection, we present  experimental results in approximating the product gate function, i.e., $f(u,v)=uv$ for $u,v \in [-1,1]$. The specific experimental settings in approximating the product gate function can be found in Table \ref{Tab:exp-setting}.

{\bf A. Approximation performance of ADMM}. The performance of ADMM and competitors is presented in Table \ref{Tab:prodgate-result}, and their performance with respect to the depth   is depicted in Figure \ref{Fig:prodgate-depth}. From Table \ref{Tab:prodgate-result}, the product gate function can be well approximated by the ADMM-sigmoid pair with precision in the order of $10^{-9}$, which is  better than those of    competitors including SGD (ReLU), SGDM (ReLU) and Adam (ReLU), even when more hidden layers are involved in the training.  It follows from Figure \ref{Fig:prodgate-depth}(b) and Table \ref{Tab:prodgate-result} that the optimal depth for ADMM in approximating the product gate function is $2$, which matches the theoretical depth for the approximation of product gate as shown in \citep[Proposition 3]{Chui-Lin-Zhou19}. Similar to the case of approximating square function, the running time of the proposed ADMM-sigmoid pair is less than the SGD type competitors for deep ReLU nets with more hidden layers.

\begin{table*}
\caption{Experimental results of different algorithms in approximating $f(u,v)=uv$. }
\footnotesize
\begin{center}
\begin{tabular}{|c|c|c|c|c|c|}\hline
Algorithm        & SGD (ReLU)       & SGDM (ReLU)      & Adam (ReLU)      & SGD (sigmoid)     & ADMM (sigmoid) \\\hline
Approx. Error    & 1.22e-6(3.68e-7) & 3.37e-7(1.29e-7) & 1.13e-6(4.37e-7) & 1.13e-3(2.33e-4)  &{\bf 2.62e-9(1.05e-9)} \\\hline
Run Time (s)     & 66.37            & 54.17            & 46.58            & 9.72              &17.29 \\\hline
(depth, width)   & (20,300)         & (18,180)         & (13,120)         & (2,240)           &(2,300)\\\hline
\end{tabular}
\end{center}
\label{Tab:prodgate-result}
\end{table*}

\begin{figure}[!t]
\begin{minipage}[b]{0.49\linewidth}
\centering
\includegraphics*[scale=.48]{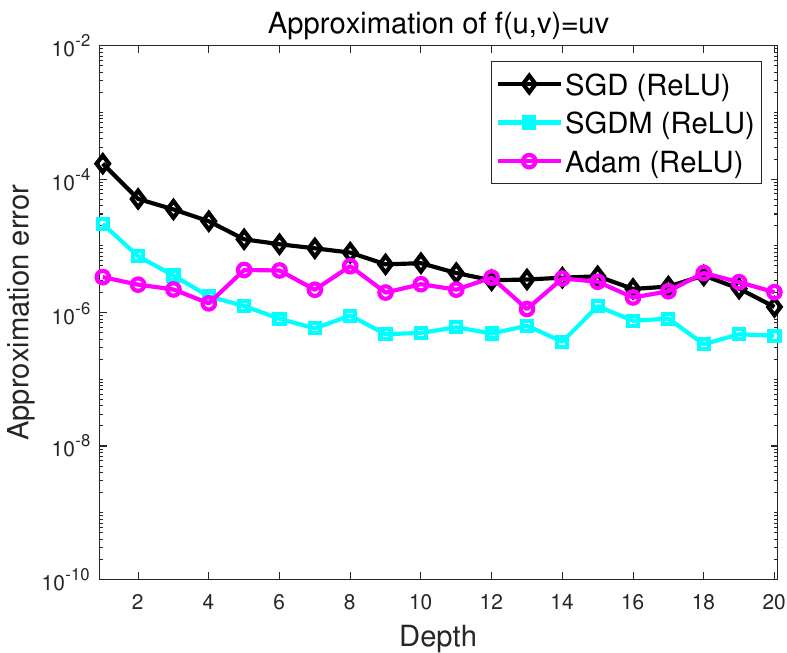}
\centerline{{\small (a) Deep ReLU nets}}
\end{minipage}
\hfill
\begin{minipage}[b]{0.49\linewidth}
\centering
\includegraphics*[scale=.48]{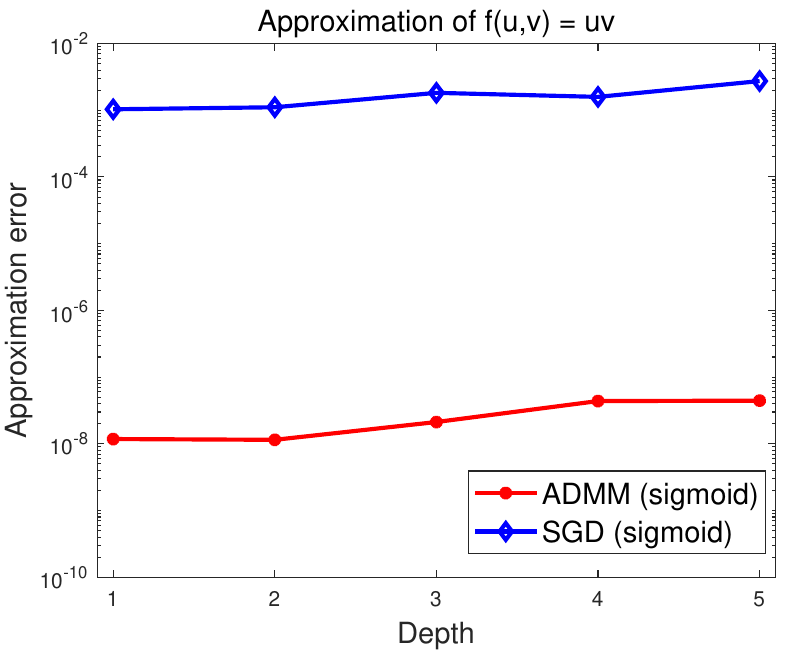}
\centerline{{\small (b) Deep sigmoid nets}}
\end{minipage}
\hfill
\caption{Effect of the depth of neural networks in approximating the product-gate function.
}
\label{Fig:prodgate-depth}
\end{figure}

{\bf B. Effect of parameters of ADMM.} Similar to Section \ref{sc:simul-sqfun} B, we also consider the effect of  parameters   $\lambda$ and $\beta$  for ADMM under the optimal  network structures, which is presented in Figure \ref{Fig:prodgate-para-admm}(a). From Figure \ref{Fig:prodgate-para-admm}(a), the effect of the concerned parameters on the performance of ADMM in approximating the product gate function is very similar to that in the approximation of univariate square function. It can be observed that the settings of parameters with $(\lambda = 10^{-6}, \beta=1)$ are empirically  good for ADMM in these two approximation tasks.

\begin{figure}[!t]
\begin{minipage}[b]{0.49\linewidth}
\centering
\includegraphics*[scale=.5]{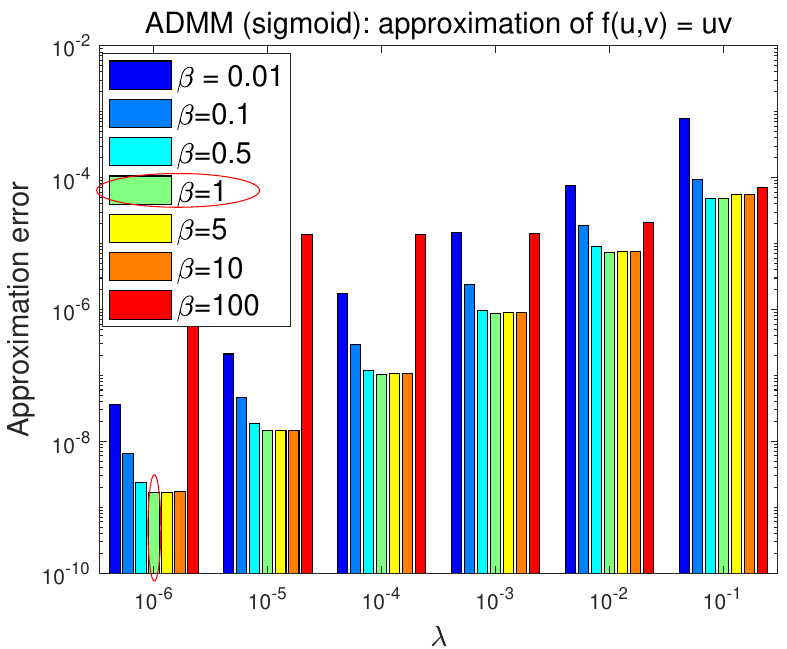}
\centerline{{\small (a) Effect of parameters of ADMM}}
\end{minipage}
\hfill
\begin{minipage}[b]{0.49\linewidth}
\centering
\includegraphics*[scale=.5]{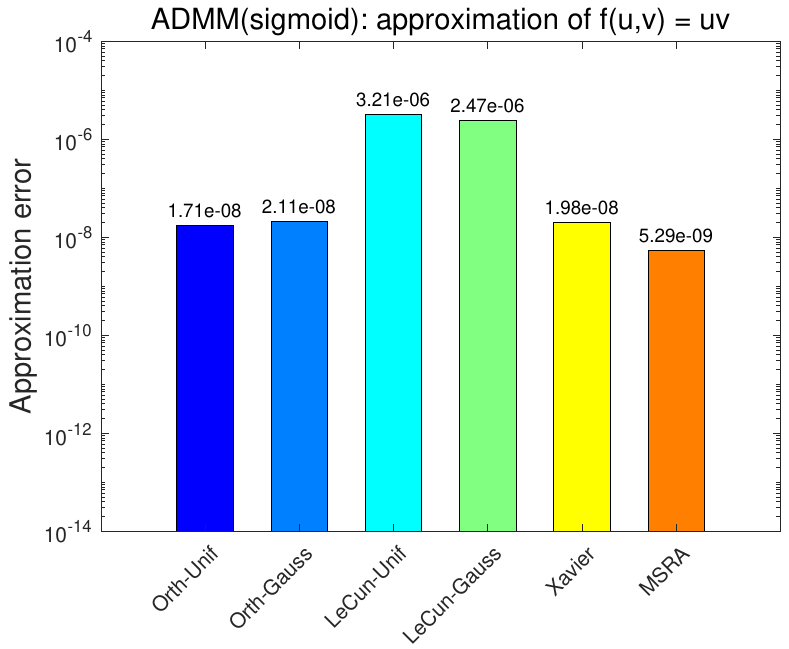}
\centerline{{\small (b) Stability of initialization}}
\end{minipage}
\hfill
\caption{Effect of parameters for ADMM in approximating the product-gate function.
}
\label{Fig:prodgate-para-admm}
\end{figure}

{\bf C. Effect of initial schemes.} Moreover, in this experiment, we consider the performance of ADMM (sigmoid) for the aforementioned six different initialization schemes. The experimental results are shown in Figure \ref{Fig:prodgate-para-admm}(b). It can be observed in  Figure \ref{Fig:prodgate-para-admm}(b) that all the initial schemes are generally effective  in yielding an ADMM solver with high precision. Among these effective initialization schemes, the LeCun type of initializations perform slightly worse than the others. This, in some extent, also implies that the proposed ADMM is usually stable to initial schemes.

\subsection{Learning  $L_1$ radial function}
\label{sc:simul-L1radial}

In this subsection, we consider the performance of the   ADMM-sigmoid pair for learning a two-dimensional $L_1$ radial function, i.e., $f(x)=(\|x\|_1-1)_+:=\max\{0,\|x\|_1-1\}$ for $x\in [r, (1+\epsilon)r]\times [r, (1+\epsilon)r]$ for some $r>0$ and $\epsilon>0$. Such an $L_1$ radial function was particularly considered in \citep{Safran-Shamir17}. In our experiments, we let $\epsilon = 1/2$ and $r = 1-\frac{\epsilon}{2}$ in light of the theoretical studies in \citep{Safran-Shamir17}. Different from the approximation tasks in Sections \ref{sc:simul-sqfun} and \ref{sc:simul-prodgate},   samples generated for the learning task include both  training and test samples, where  training samples are commonly generated with certain noise and the test samples are clean data. In these experiments, we consider  Gaussian noises with different variances.

{\bf A. Learning performance of ADMM.} Optimal test errors of different algorithms for learning the $L_1$ radial function are presented in Table \ref{Tab:L1radial-result}, where the variance of Gaussian noise added into the training samples is $0.1$. The associated test errors of these algorithms with respect to the depth of neural networks are presented in Figure \ref{Fig:L1radial-depth}. From Table \ref{Tab:L1radial-result}, the considered $L_1$ radial function can be well learned by both ADMM and SGD type methods. Specifically, the performance of the proposed ADMM is slightly better than   SGD type methods. In particular, the optimal depth of deep sigmoid nets trained by ADMM is only 4, which is much less than those of deep ReLU nets trained by SGD type methods. Under optimal network structures, the running time of the suggested ADMM-pair  is slightly less than that of SGD type methods for  deep ReLU nets, due to the less depth of  deep sigmoid nets. According to Figure \ref{Fig:L1radial-depth}(b), the proposed ADMM performs better than SGD for  training deep sigmoid nets, and as the depth increasing, the performance of SGD gets worse possibly due to the vanishing gradient issue, while our suggested ADMM can alleviate the issue of vanishing gradient and thus achieve better and better performance in general as the depth increases in our considered range of depth, i.e., $\{1,2,3,4,5\}$.

\begin{table*}
\caption{Performance of different algorithms for learning $L_1$ radial function with $0.1$ Gaussian noise.}
\footnotesize
\begin{center}
\begin{tabular}{|c|c|c|c|c|c|}\hline
Algorithm        & SGD (ReLU)       & SGDM (ReLU)      & Adam (ReLU)      & SGD (sigmoid)     & ADMM (sigmoid) \\\hline
Test Error    & 2.48e-5(9.74e-6) &2.26e-5(7.88e-6)  &2.16e-5(8.53e-6)  &4.58e-5(1.59e-5)   &{\bf 1.69e-5(4.34e-6)} \\\hline
Run Time (s)     & 23.24            &18.23             &14.66             &1.68              &10.36 \\\hline
(depth, width)   &(17,50)           &(13,50)           &(10,50)           & (1,20)           &(4,10) \\\hline
\end{tabular}
\end{center}
\label{Tab:L1radial-result}
\end{table*}

\begin{figure}[!t]
\begin{minipage}[b]{0.49\linewidth}
\centering
\includegraphics*[scale=.48]{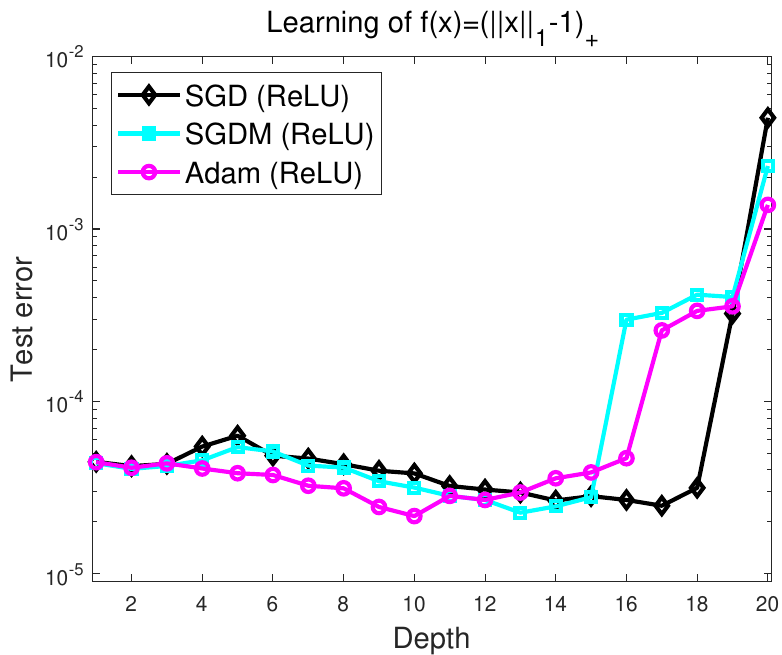}
\centerline{{\small (a) Deep ReLU nets}}
\end{minipage}
\hfill
\begin{minipage}[b]{0.49\linewidth}
\centering
\includegraphics*[scale=.48]{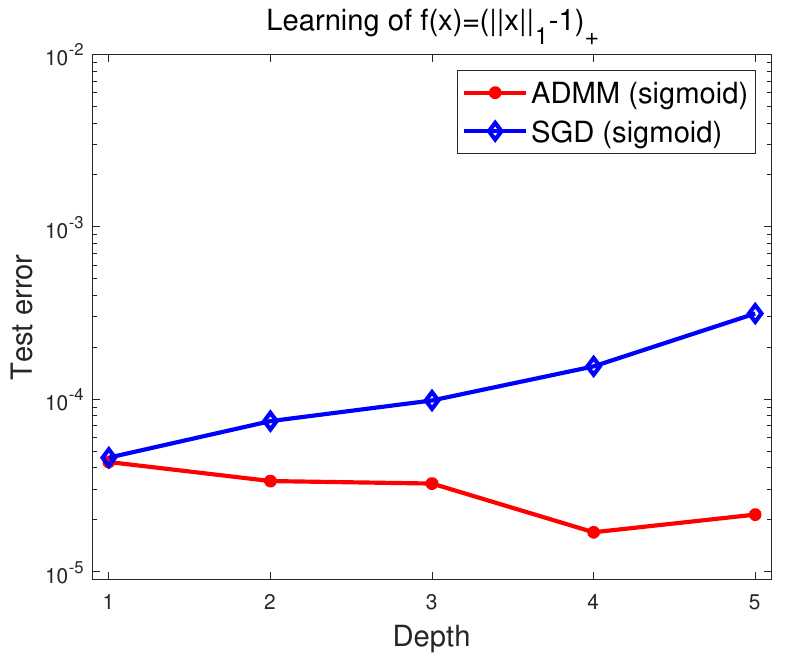}
\centerline{{\small (b) Deep sigmoid nets}}
\end{minipage}
\hfill
\caption{Effect of the depth of neural networks in learning the $L_1$ radial function.
}
\label{Fig:L1radial-depth}
\end{figure}

{\bf B. Effect of parameters and initialization.} Under the  optimal neural network structures specified in Table \ref{Tab:L1radial-result}, we consider the effect of  parameters, i.e., $(\lambda, \beta)$ for ADMM, as well as the effect of the initialization schemes for both ADMM and SGD type methods. The numerical results are shown in Figure \ref{Fig:L1radial-para-init}. From Figure \ref{Fig:L1radial-para-init}(a), we can observe that the specific parametric setting, i.e., $\lambda = 10^{-6}$ and $\beta=1$, is also empirically   effective in  learning  $L_1$ radial function. By Figure \ref{Fig:L1radial-para-init}(b), the proposed ADMM performs well for all the concerned random initialization schemes.

\begin{figure}[!t]
\begin{minipage}[b]{0.49\linewidth}
\centering
\includegraphics*[scale=.5]{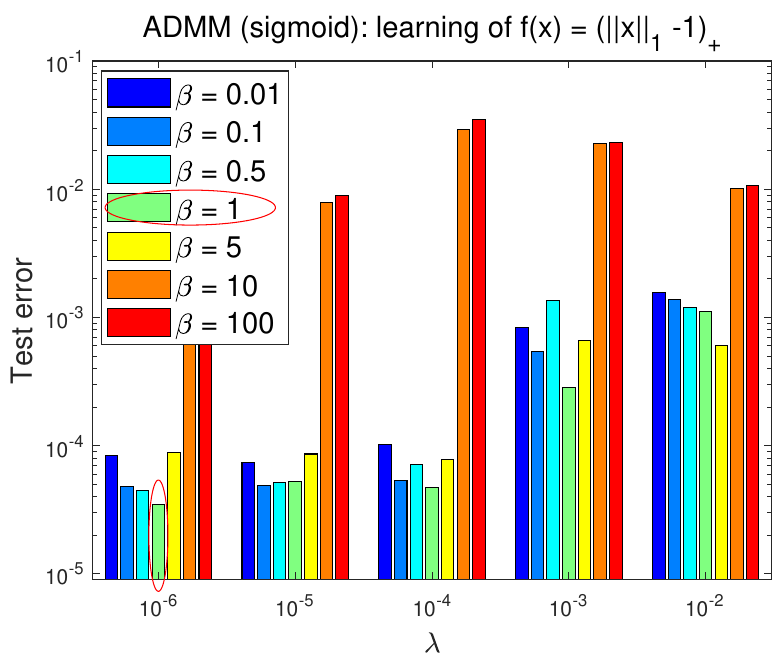}
\centerline{{\small (a) Effect of parameters of ADMM}}
\end{minipage}
\hfill
\begin{minipage}[b]{0.49\linewidth}
\centering
\includegraphics*[scale=.5]{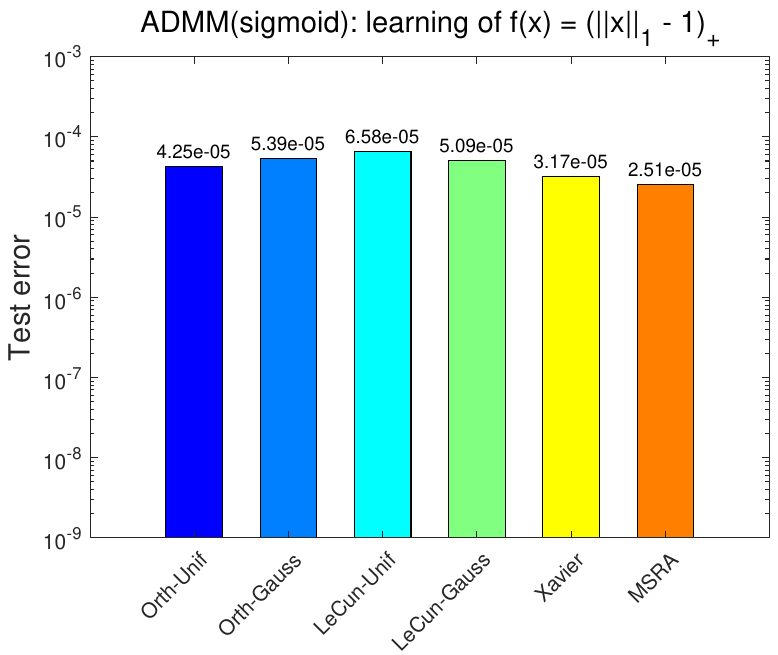}
\centerline{{\small (b) Stability of initialization}}
\end{minipage}
\hfill
\caption{Effect of parameters and initial schemes for ADMM in learning the $L_1$ radial function.
}
\label{Fig:L1radial-para-init}
\end{figure}

{\bf C. Robustness to the noise.} Moreover, we consider the performance of the proposed ADMM for  training data with different levels of noise. Specifically, under the optimal parameters specified in Table \ref{Tab:L1radial-result}, we consider several levels of noise, where the variance of Gaussian noise varies from   $\{0.1, 0.3, \ldots, 1.5\}$. Trends of  training  and test errors are presented in Figure \ref{Fig:L1radial-noise}(a) and (b) respectively. From Figure \ref{Fig:L1radial-noise}(a),  the proposed ADMM is generally trained well in the sense that the training error almost fits the true noise level. In this case, we can observe from Figure \ref{Fig:L1radial-noise}(b) that the proposed ADMM is robustness to the noise in the sense that the test error increases much slower than the increasing of the variance of Gaussian noise.

\begin{figure}[!t]
\begin{minipage}[b]{0.49\linewidth}
\centering
\includegraphics*[scale=.48]{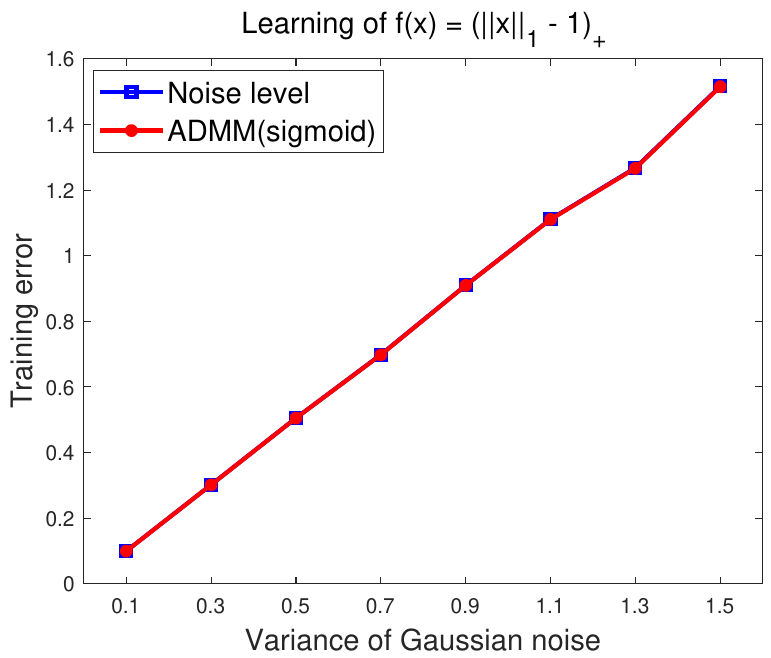}
\centerline{{\small (a) Training error}}
\end{minipage}
\hfill
\begin{minipage}[b]{0.49\linewidth}
\centering
\includegraphics*[scale=.48]{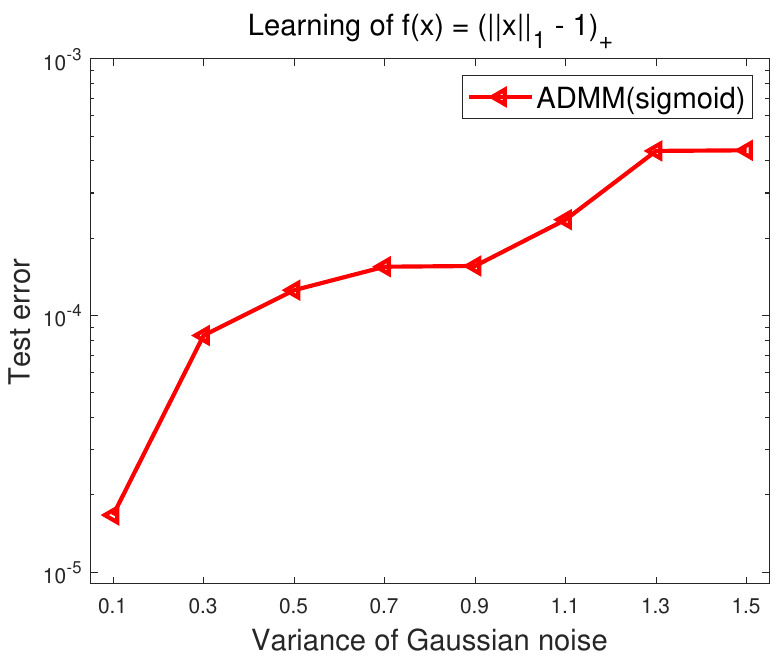}
\centerline{{\small (b) Test error}}
\end{minipage}
\hfill
\caption{Robustness of the proposed ADMM to the noise in learning $L_1$ radial function.
}
\label{Fig:L1radial-noise}
\end{figure}

\subsection{Learning  $L_2$ radial function}
\label{sc:simul-L2radial}

In this subsection, we consider to learn certain smooth $L_2$ radial function that frequently reflects the rotation-invariance feature in deep learning \citep{Chui-Lin-Zhou19}. Specifically, we adopt a two-dimensional smooth $L_2$ radial function, i.e., $f(x) = g(|x|^2)$, where $x\in [-1,1]\times [-1,1]$, $|x|^2 :=\sum_{i=1}^2 x_i^2$, and $g(t)= (1-t)_+^5 (8t^2+5t+1)$ on $\mathbb{R}$ is some Wendland function  \citep{Lin-CFN2019}. Except the target function $f$, the experimental settings in these experiments are similar to those in Section \ref{sc:simul-L1radial}.

{\bf A. Learning performance of ADMM.} The test error of the considered algorithms in learning such a smooth $L_2$ radial function is presented in Table \ref{Tab:L2radial-result}, while  trends of test errors with respect to the depth   are shown in Figure \ref{Fig:L2radial-depth}. By Table \ref{Tab:L2radial-result}, the considered smooth $L_2$ radial function can be learned by the proposed ADMM well with a  small test error. Specifically, in terms of test error, the performance of the   ADMM-sigmoid pair is slightly better than that of SGD type methods for  deep ReLU nets, and the optimal depth of deep sigmoid nets required by ADMM is much smaller than those of deep ReLU nets required by the concerned SGD type methods. Due to less depth, the running time of ADMM is less than that of the concerned SGD type methods for  deep ReLU nets under the optimal settings of neural networks.
Moreover, from Figure \ref{Fig:L2radial-depth}(a), a deeper ReLU network with about 10 layers is generally required to learn the $L_2$ radial function with a good test error, while from Figure \ref{Fig:L2radial-depth}(b), the depth of deep sigmoid nets trained by ADMM can be much smaller (i.e., about 5) to yield a good test error.

\begin{table*}
\caption{Test errors of different algorithms for learning $L_2$ radial function with $0.1$ Gaussian noise.}
\footnotesize
\begin{center}
\begin{tabular}{|c|c|c|c|c|c|}\hline
Algorithm        & SGD (ReLU)       & SGDM (ReLU)      & Adam (ReLU)      & SGD (sigmoid)     & ADMM (sigmoid) \\\hline
Test Error    &1.68e-5(6.43e-6)  &1.21e-5(5.25e-6)  &1.02e-5(4.88e-6)  &9.33e-5(1.42e-5)   &{\bf 9.28e-6(1.01e-6)} \\\hline
Run Time (s)     & 104.52            & 116.44            & 108.49            & 18.13             & 47.36\\\hline
(depth, width)   & (16,300)          & (12,400)          & (11,400)          & (4,200)          &(5,300) \\\hline
\end{tabular}
\end{center}
\label{Tab:L2radial-result}
\end{table*}

\begin{figure}[!t]
\begin{minipage}[b]{0.48\linewidth}
\centering
\includegraphics*[scale=.49]{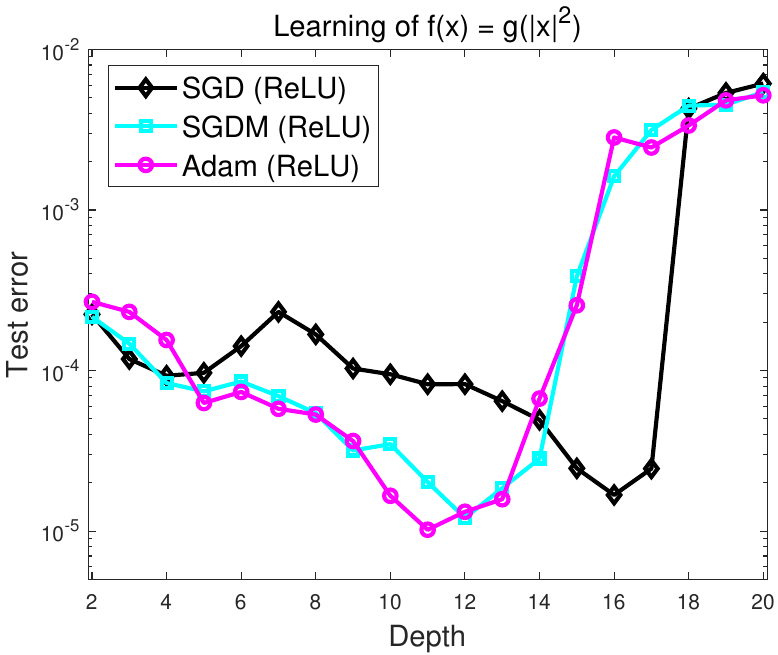}
\centerline{{\small (a) Deep ReLU nets}}
\end{minipage}
\hfill
\begin{minipage}[b]{0.48\linewidth}
\centering
\includegraphics*[scale=.49]{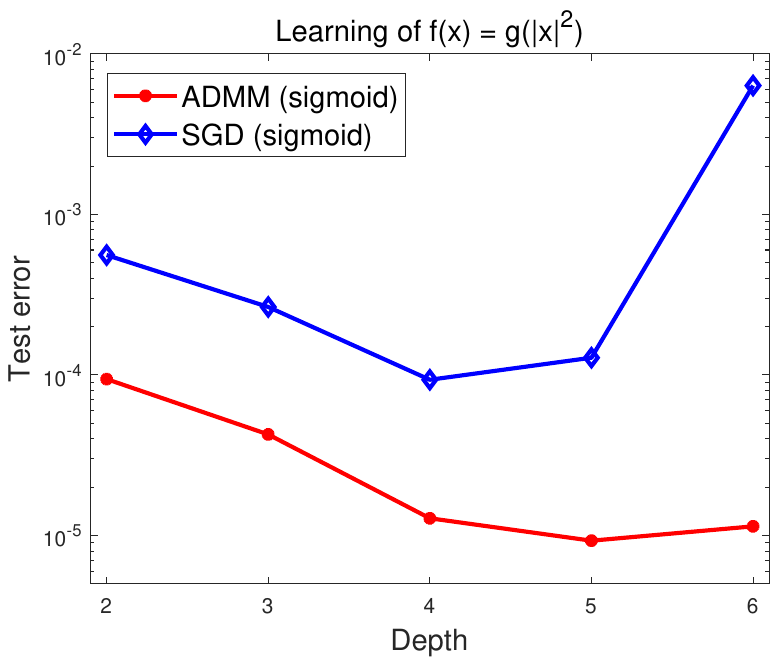}
\centerline{{\small (b) Deep sigmoid nets}}
\end{minipage}
\hfill
\caption{Effect of the depth of neural networks in learning the $L_2$ radial function.
}
\label{Fig:L2radial-depth}
\end{figure}

{\bf B. Effect of parameters and initialization.}
In this part, we consider the effect of parameters (i.e., $\lambda$ and $\beta$) of ADMM as well as the effect of initialization for learning the $L_2$ radial function in the optimal settings specified in Table \ref{Tab:L2radial-result}. The numerical results are presented in Figure \ref{Fig:L2radial-para-init}. From Figure \ref{Fig:L2radial-para-init}(a), the effect of parameters are similar to   previous three simulations and it can be observed that the specific settings, i.e., $\lambda = 10^{-6}$ and $\beta = 1$, are empirically  effective. From Figure \ref{Fig:L2radial-para-init}, we also observe that ADMM is effective to all the random initialization schemes.

\begin{figure}[!t]
\begin{minipage}[b]{0.49\linewidth}
\centering
\includegraphics*[scale=.5]{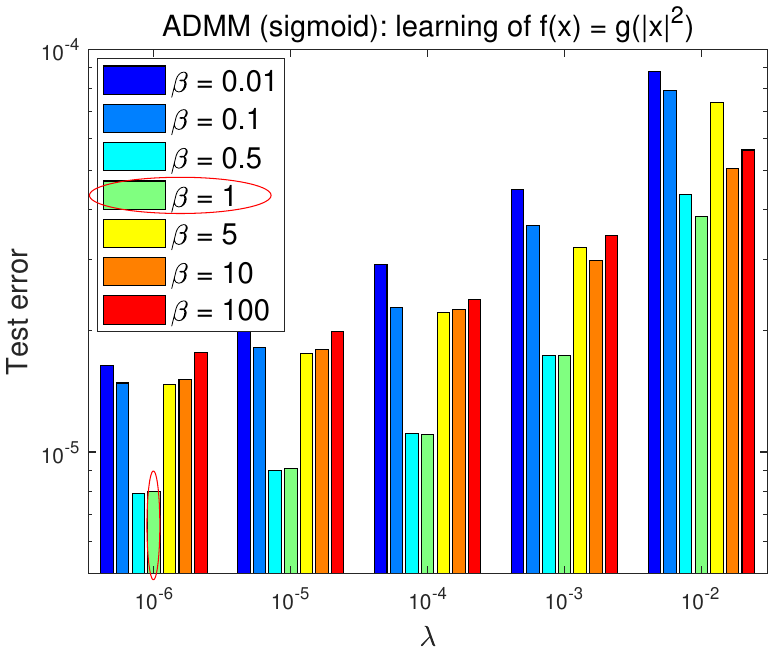}
\centerline{{\small (a) Effect of parameters of ADMM}}
\end{minipage}
\hfill
\begin{minipage}[b]{0.49\linewidth}
\centering
\includegraphics*[scale=.5]{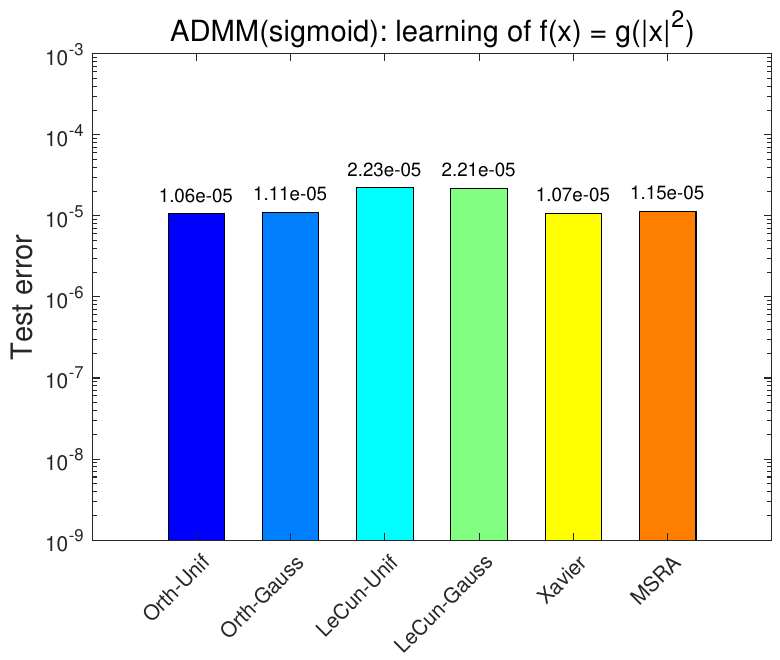}
\centerline{{\small (b) Stability of initialization}}
\end{minipage}
\hfill
\caption{Effect of parameters and initial schemes for ADMM in learning smooth $L_2$ radial function.
}
\label{Fig:L2radial-para-init}
\end{figure}

{\bf C. Robustness to noise.} Similar to the learning of $L_1$ radial function, we consider the performance of the proposed ADMM  for noisy training data with different levels of noise. Specifically, the variance of the Gaussian noise added into the training samples varies from   $\{0.1, 0.3, 0.5, 0.7, 0.9, 1.1\}$. Curves of training error and test error are shown respectively in Figure \ref{Fig:L2radial-noise}(a) and (b). From Figure \ref{Fig:L2radial-noise}, the behavior in learning $L_2$ radial function is  similar to that in learning $L_1$ radial function as shown in Figure \ref{Fig:L1radial-noise}. This demonstrates that the proposed ADMM is also robust to noise in learning such a smooth $L_2$ radial function.

\begin{figure}[!t]
\begin{minipage}[b]{0.49\linewidth}
\centering
\includegraphics*[scale=.48]{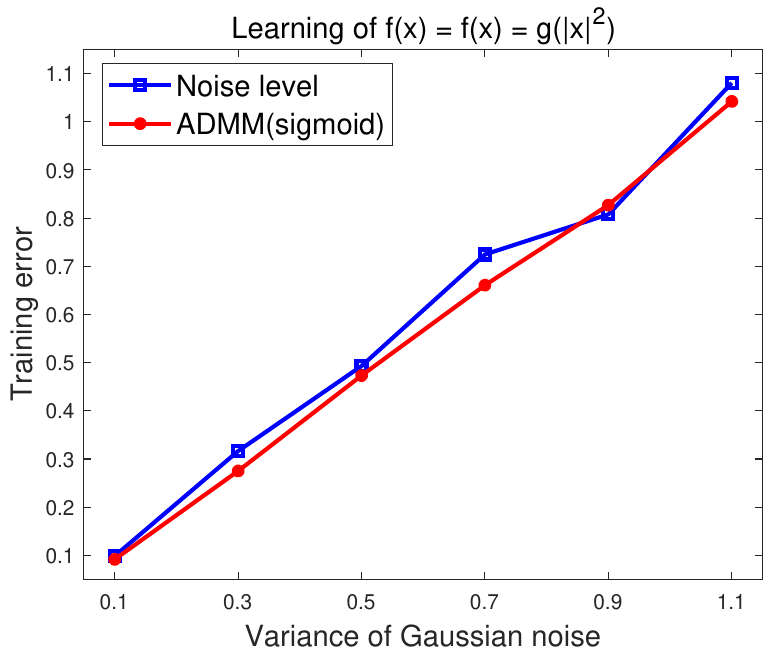}
\centerline{{\small (a) Training error}}
\end{minipage}
\hfill
\begin{minipage}[b]{0.49\linewidth}
\centering
\includegraphics*[scale=.48]{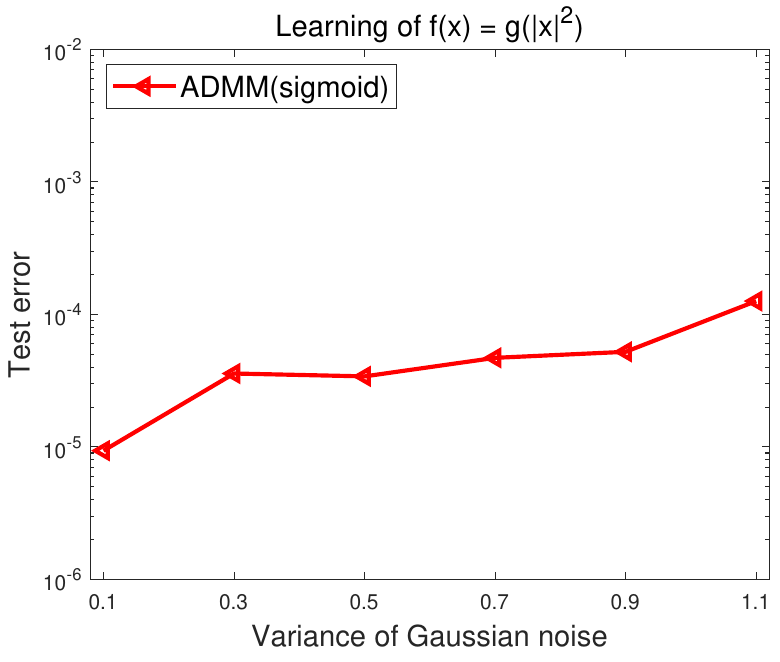}
\centerline{{\small (b) Test error}}
\end{minipage}
\hfill
\caption{Robustness of the proposed ADMM to the noise in learning $L_2$ radial function.
}
\label{Fig:L2radial-noise}
\end{figure}

\section{Real Data Experiments}
\label{sc:real-experiments}
In this section, we provide three real-data experiments over the earthquake intensity database, the extended Yale B (EYB) face recognition database and the PTB Diagnostic ECG database, to demonstrate the effectiveness of the proposed ADMM. We choose these three datasets since they can in some sense reflect certain features that can be well approximated by deep sigmoid nets, and thus, the benefits of the proposed ADMM can be embodied over these datasets. Specific experimental settings are presented in Table \ref{Tab:real-setting}, where the penalty parameter $\beta$ is empirically set as $1$ and the regularization parameter $\lambda$ is chosen via cross validation from the set $\{10^{-6}, 10^{-5}, 10^{-4}, 10^{-3}, 10^{-2}\}$ according to the previous studies of toy simulations.

\begin{table*}
\caption{Experimental settings for real-data experiments. The number of epochs for each case is set empirically to be 200.}
\tiny
\begin{center}
\begin{tabular}{|c|c|c|c|c|c|c|c|c|}\hline
\multirow{2}*{dataset} &(training size, & \multicolumn{2}{|c|}{Network structure} &\multicolumn{2}{|c|}{SGDs (sigmoid/ReLU),SGDM} & SGDM & Adam & ADMM \\
\cline{3-9}
~ &test size) & width & depth & batch size & learning rate & (momentum) & lr:0.001 & $(\lambda,\beta)$ \\\hline
Earthquake & (4173,4000) & $20\times [1:10]$ & [1:6]  & 100 & $0.1\times 0.95^k$, &~  & $\beta_1$: 0.9 & $\lambda \in 10^{[-6:-2]}$ \\
\cline{1-5}
EYB & (2432,2432) &$20\times [1:10]$ & 1 & 50 & per 10 epochs & 0.5 & $\beta_2$: 0.999 & $\beta=1$ \\
\cline{1-5}
PTB & (7000,7552) &$64\times [1:4]$ & [1:10] & 100 & ~ &~ &$\epsilon$: 1e-8 & ~ \\\hline
\end{tabular}
\end{center}
\label{Tab:real-setting}
\end{table*}

\subsection{Earthquake intensity dataset}
\label{sc:earthquake}

\textit{Earthquake Intensity Database} is from: \textit{https://www.ngdc.noaa.gov/hazard/intintro.shtml}. This database contains more than 157,000 reports on over 20,000 earthquakes that affected the United States from the year 1638 to 1985. For each record, the features include the geographic latitudes and longitudes of the epicentre and  ``reporting city'' (or, locality) where the Modified Mercalli Intensity (MMI) was observed, magnitudes (as a measure of seismic energy), and the hypocentral depth (positive downward) in kilometers from the surface. The output label of each record is measured by MMI, varying from 1 to 12 in integer. An illustration of the generation procedure of each earthquake record is shown in Figure \ref{Fig:earthquake-result}(a). In this paper, we transfer such multi-classification task into the binary classification since this database is very unbalanced (say, there is only one sample for the class with MMI being 1). Specifically, we set the labels lying in 1 to 4 as the positive class, while the other labels lying in 5 to 12 as the negative class, mainly according to the damage extent of the earthquake suggested by the referred website. Moreover, we removed those incomplete records with missing labels. After such preprocessing, there are total 8173 effective records, where the numbers of samples in positive and negative classes are respectively 5011 and 3162. We divide the total data set into the training and test sets randomly, where the training and test sample sizes are 4173 and 4000, respectively. Before training, we use  the \textit{z-scoring normalization} for each feature, that is, $\frac{x_i-\mu}{\sigma}$ with $\mu$ and $\sigma$ being respectively the mean and standard deviation of the $i$th feature $\{x_i\}$. The classification accuracies of all algorithms are shown in Table \ref{Tab:earthquake-result}. The effect of the depth of neural network, algorithmic parameters, and random initial schemes are shown in Figure \ref{Fig:earthquake-result} (b)-(d) respectively.

According to Table \ref{Tab:earthquake-result}, the performance of the proposed ADMM is comparable to the state-of-the-art methods in terms of test accuracy. Specifically, the proposed ADMM is slightly worse than Adam, and outperforms the other competitors in terms of test accuracy, while in terms of running time, the proposed ADMM is slightly faster than   Adam and SGDM under the associated optimal network settings, mainly because the optimal depth of the deep sigmoid nets trained by ADMM is less than those of deep ReLU nets trained by Adam and SGDM. Compared to the SGD counterpart for  deep sigmoid nets, the performance of the proposed ADMM is much better in terms of test accuracy. It can be observed from Figure \ref{Fig:earthquake-result}(b) that the vanilla SGD may suffer from the gradient vanishing/explosion issue when training a slightly deeper sigmoid nets (say, the depth is larger than 5) due to the saturation of the sigmoid activation, while the proposed ADMM can avoid such saturation and thus alleviate the gradient vanishing/explosion issue. From Figure \ref{Fig:earthquake-result}(c), the proposed ADMM with the default settings, i.e., $\lambda =$ 1e-6 and $\beta=1$ in general yields the best performance. Moreover, it can be observed from Figure \ref{Fig:earthquake-result}(d) that the proposed ADMM is stable to the commonly used initialization schemes under the optimal neural network structure specified in Table \ref{Tab:earthquake-result}.

\begin{table*}
\caption{Test accuracies (\%) of different algorithms for earthquake intensity database. The baseline of the test accuracy is $80.48\%$ \citep{Zeng20}.}
\footnotesize
\begin{center}
\begin{tabular}{|c|c|c|c|c|c|}\hline
Algorithm        & SGD (ReLU)       & SGDM (ReLU)      & Adam (ReLU)      & SGD (sigmoid)     & ADMM (sigmoid) \\\hline
Test Acc(\%)     &81.24(0.45)       &81.16(0.32)       &{\bf 81.31(0.36)}       &79.94(0.23)   &81.26(0.31) \\\hline
Run Time (s)     & 4.74            & 14.20            & 13.24           & 2.60             & 12.64\\\hline
(depth, width)   & (2,120)          & (5,140)          & (4,80)          & (1,100)          &(3,80) \\\hline
\end{tabular}
\end{center}
\label{Tab:earthquake-result}
\end{table*}

\begin{figure}[!t]
\begin{minipage}[b]{0.49\linewidth}
\centering
\vspace{-.5cm}
\includegraphics*[scale=.38]{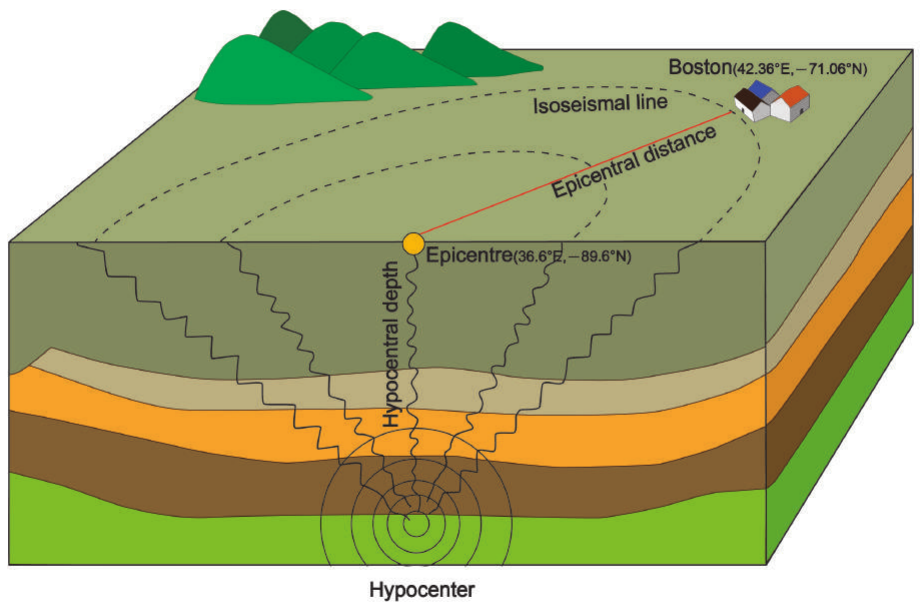}
\centerline{{\small (a) An illustration of earthquake data}}
\end{minipage}
\hfill
\begin{minipage}[b]{0.49\linewidth}
\centering
\includegraphics*[scale=.48]{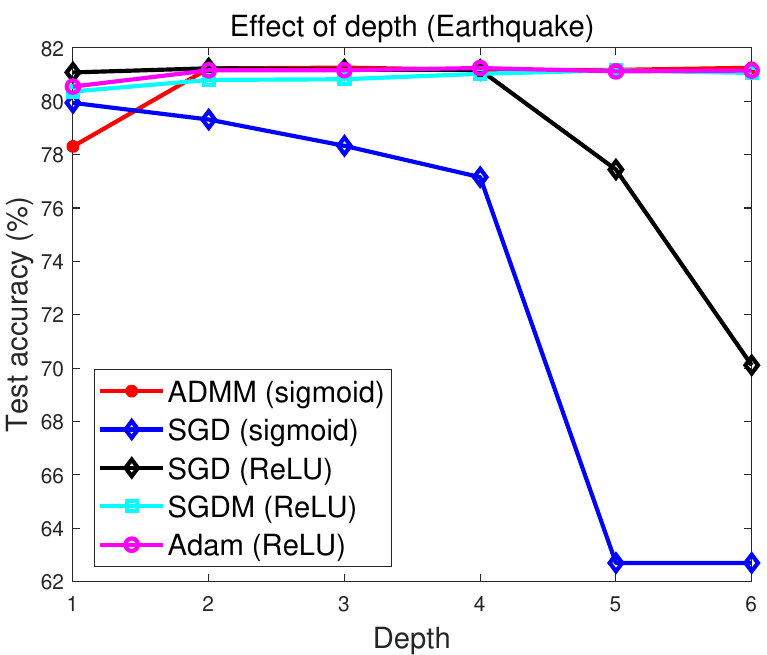}
\centerline{{\small (b) Effect of depth of NNs}}
\end{minipage}
\hfill
\begin{minipage}[b]{0.49\linewidth}
\centering
\vspace{.5cm}
\includegraphics*[scale=.48]{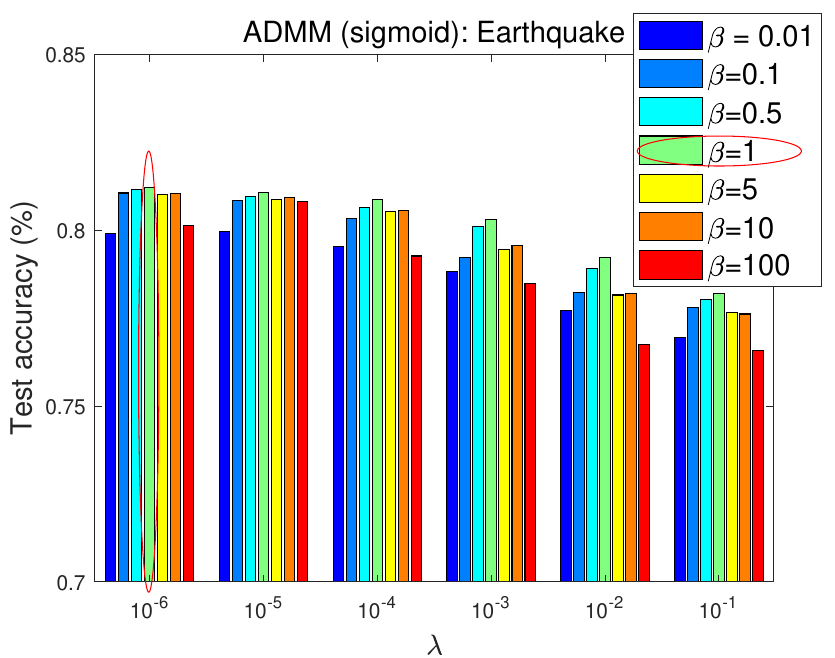}
\centerline{{\small (c) Effect of parameters}}
\end{minipage}
\hfill
\begin{minipage}[b]{0.49\linewidth}
\centering
\vspace{.5cm}
\includegraphics*[scale=.48]{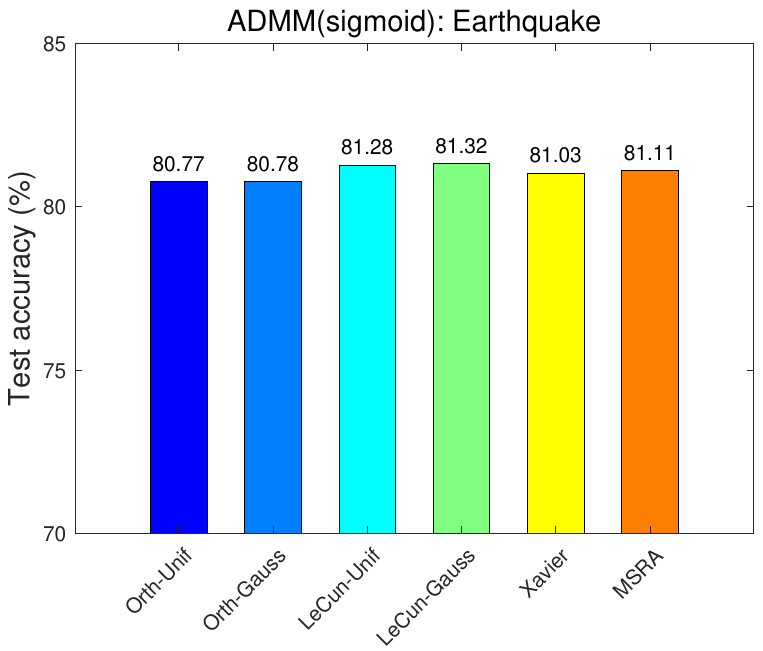}
\centerline{{\small (d) Stability to initial schemes}}
\end{minipage}
\hfill
\caption{Performance of ADMM in earthquake intensity data: (a) an illustration of the earthquake intensity data \citep{Zeng20}; (b) the effect of depth of the neural network for different algorithms; (c) the effect of algorithmic parameters for the proposed ADMM; (d) the stability of the proposed ADMM to different initial schemes.
}
\label{Fig:earthquake-result}
\end{figure}

\subsection{Extended Yale B face recognition database}
\label{sc:EYB}

In the extended Yale B (EYB) database,  well-known face recognition database \citep{Lee-EYB05}, there are in total 2432 images for 38 objects under 9 poses and 64 illumination conditions, where for each objective, there are 64 images. The pixel size of each image is $32\times32$. In our experiments, we randomly divide these 64 images for each objective into two equal parts, that is, one half of images are used for training while the rest half of images are used for testing. For each image, we normalize it via the z-scoring normalization. The specific experimental settings for this database can be found in Table \ref{Tab:real-setting}. Particularly, we empirically use a shallow neural network with depth one and various of widths, since such shallow neural network is good enough to extract the low-dimensional manifold feature of this face recognition data, as shown in Table \ref{Tab:EYB-result}. The effect of network structures and  stability of the proposed ADMM to initialization schemes are shown in Figure \ref{Fig:EYB-result}(a) and (b) respectively.

According to Table \ref{Tab:EYB-result}, the proposed ADMM achieves the state-of-the-art test accuracy (see, \cite{Lu-EYB20}) with a smaller width of the sigmoid nets when compared to the concerned competitors. From Figure \ref{Fig:EYB-result}, the proposed ADMM can achieve a very high test accuracy for most of the concerned widths of the networks and is stable to the commonly used random initial schemes.

\begin{table*}
\caption{Performance of different algorithms for extended Yale B database. The baseline of the test accuracy is about $96\%$ in \citep{Lu-EYB20}.}
\footnotesize
\begin{center}
\begin{tabular}{|c|c|c|c|c|c|}\hline
Algorithm        & SGD (ReLU)       & SGDM (ReLU)      & Adam (ReLU)      & SGD (sigmoid)     & ADMM (sigmoid) \\\hline
Test Acc(\%)     &98.84(0.28)       &97.18(0.28)       & 98.91(0.34)     &98.67(0.41)   &{\bf 98.93(0.43)} \\\hline
Run Time (s)     & 19.78            & 23.99            & 48.92           & 16.95             & 21.36\\\hline
(depth, width)   & (1,200)          & (1,200)          & (1,200)          & (1,140)          &(1,60) \\\hline
\end{tabular}
\end{center}
\label{Tab:EYB-result}
\end{table*}

\begin{figure}[!t]
\begin{minipage}[b]{0.49\linewidth}
\centering
\vspace{-.5cm}
\includegraphics*[scale=.48]{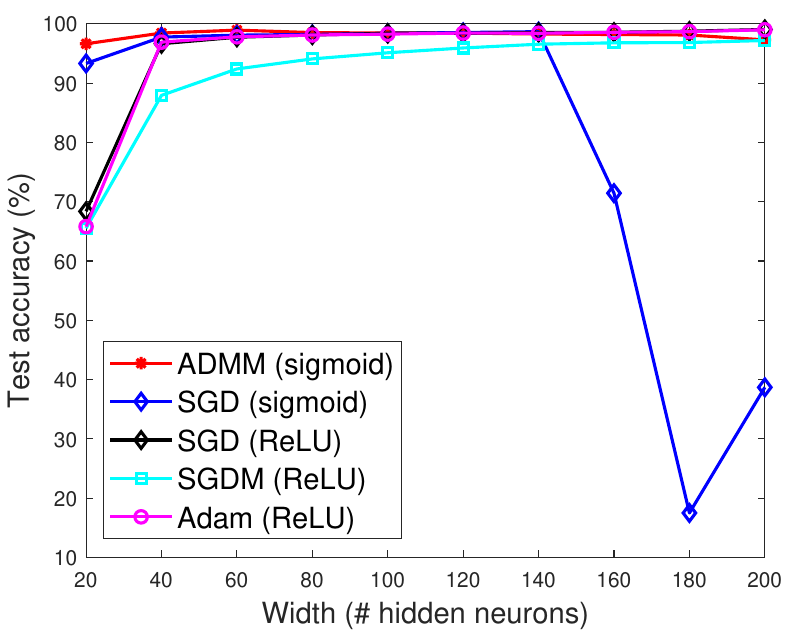}
\centerline{{\small (a) Effect of network structure.}}
\end{minipage}
\hfill
\begin{minipage}[b]{0.49\linewidth}
\centering
\includegraphics*[scale=.48]{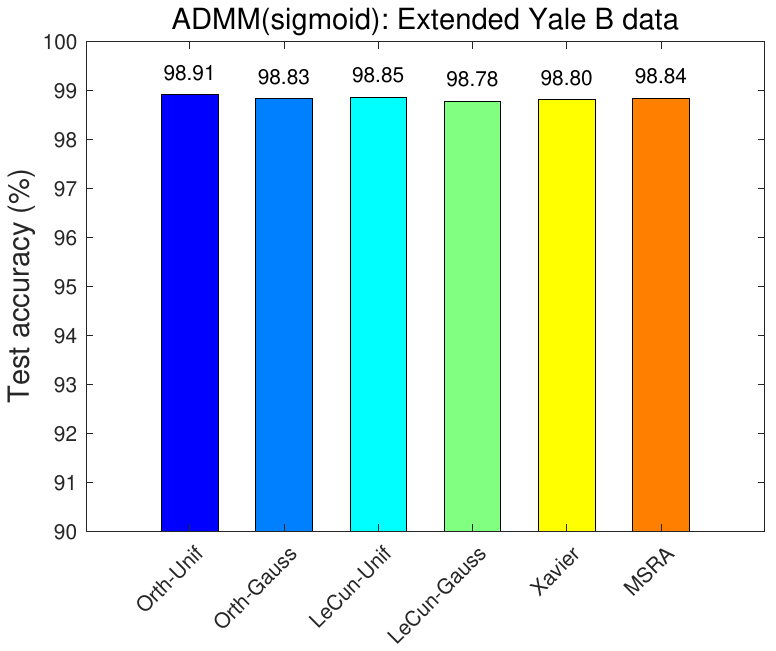}
\centerline{{\small (b) Stability to initial schemes of ADMM}}
\end{minipage}
\hfill
\caption{Performance of ADMM in extended Yale B database: (a) the effect of width for different algorithms; (b) the stability of the proposed ADMM to different random initial schemes.
}
\label{Fig:EYB-result}
\end{figure}

\subsection{PTB Diagnostic ECG database}
\label{sc:PTB}

An ECG is a 1D signal which is the result of recording the electrical activity of the heart using an electrode. It is one of popular tools that cardiologists use to diagnose heart anomalies and diseases. The PTB diagnostic ECG database is available at \url{https://github.com/CVxTz/ECG_Heartbeat_Classification} and was preprocessed by \citep{Kachuee-PTB}. There are 14,552 samples in total with 2 categories. The specific experimental settings for this database can be found in Table \ref{Tab:real-setting}. The experiment results of the proposed ADMM and concerned competitors are presented in Table \ref{Tab:PTB-result}. The effect of network structures and   stability of the proposed ADMM to initialization schemes are shown in Figure \ref{Fig:PTB-result}(a) and (b) respectively.

According to Table \ref{Tab:PTB-result}, the proposed ADMM achieves the state-of-the-art test accuracy (see, \cite{Kachuee-PTB}) with a less width of   sigmoid nets when compared to the concerned competitors. Specifically, the optimal depth of deep sigmoid nets trained by ADMM is 4, while those of deep ReLU nets trained respectively by SGD, SGDM and Adam are 8, 7, 7. This also verifies our previous claim on the advantage of deep sigmoid nets in feature representation. Due to less hidden layers, the proposed ADMM is slightly faster than the SGD competitors for  deep ReLU nets. From Figure \ref{Fig:PTB-result}(a), when the depth of deep sigmoid nets is larger than 8, the performance of all considered algorithms degrades much possibly due to the overfitting. From Figure \ref{Fig:PTB-result}(b), the proposed ADMM is stable to the commonly used random initial schemes under the optimal neural network setting as presented in Table \ref{Tab:PTB-result}.

\begin{table*}
\caption{Performance of different algorithms for PTB diagnostic ECG database. The baseline of the test accuracy is $99.20\%$ in \citep{Kachuee-PTB}.}
\footnotesize
\begin{center}
\begin{tabular}{|c|c|c|c|c|c|}\hline
Algorithm        & SGD (ReLU)       & SGDM (ReLU)      & Adam (ReLU)      & SGD (sigmoid)     & ADMM (sigmoid) \\\hline
Test Acc(\%)     &99.18(0.32)       &99.16(0.28)       & {\bf 99.25(0.25)}     &96.88(0.46)   &99.22(0.11) \\\hline
Run Time (s)     & 29.82            & 40.77            & 30.83            & 12.28             & 29.17\\\hline
(depth, width)   & (8,192)          & (7,192)          & (7,256)          & (3,256)           &(4,128) \\\hline
\end{tabular}
\end{center}
\label{Tab:PTB-result}
\end{table*}

\begin{figure}[!t]
\begin{minipage}[b]{0.49\linewidth}
\centering
\vspace{-.5cm}
\includegraphics*[scale=.48]{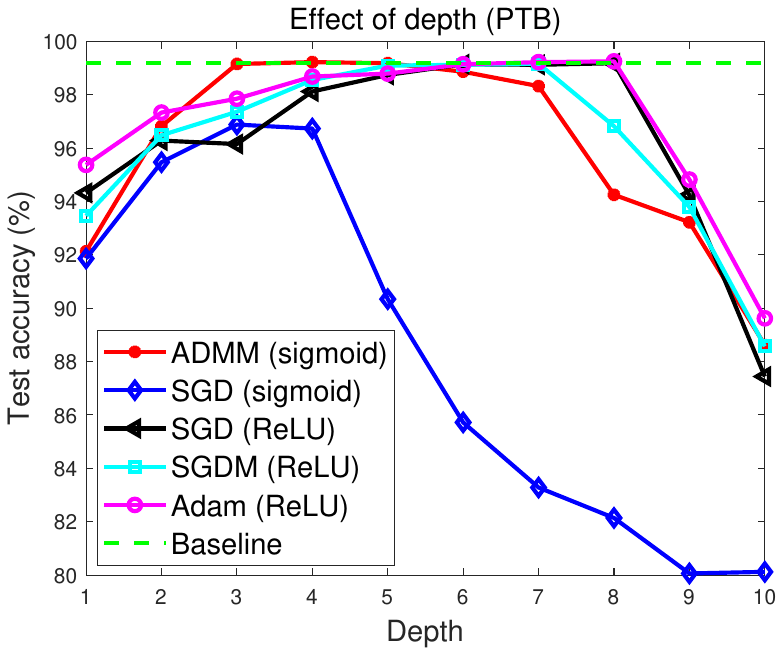}
\centerline{{\small (a) Effect of network structure.}}
\end{minipage}
\hfill
\begin{minipage}[b]{0.49\linewidth}
\centering
\includegraphics*[scale=.48]{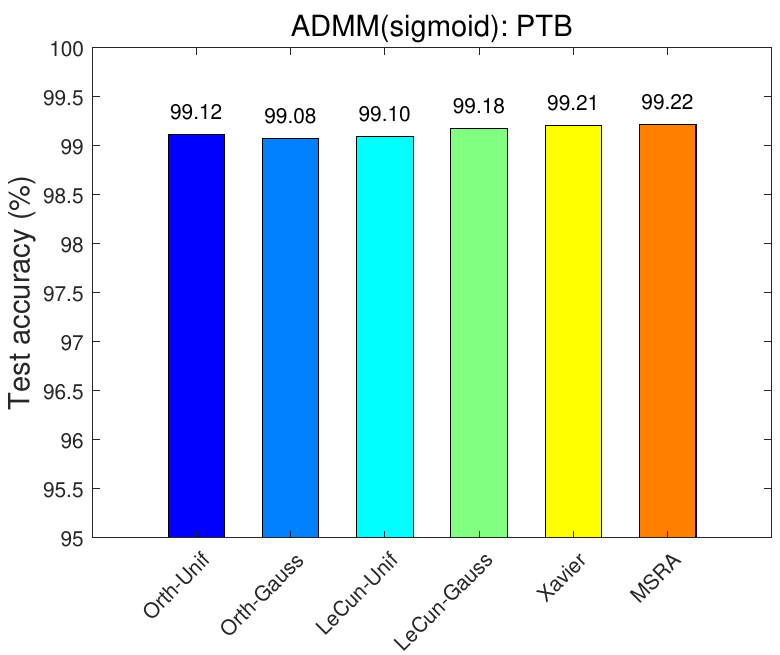}
\centerline{{\small (b) Stability to initial schemes of ADMM}}
\end{minipage}
\hfill
\caption{Performance of ADMM in PTB diagnostic ECG database: (a) the effect of depth for different algorithms; (b) the stability of the proposed ADMM to different initial schemes.
}
\label{Fig:PTB-result}
\end{figure}

\acks{We would like to thank Prof. Wotao Yin at UCLA and Dr. Yugen Yi at Jiangxi Normal University for their helpful discussions on this work. The work  is supported
 by   National Key R\&D Program of China (No.2020YFA0713900). The work of Jinshan Zeng is supported in part by the National Natural Science Foundation of China [Project No. 61977038] and by the Thousand Talents Plan of Jiangxi Province [Project NO. jxsq2019201124]. The work of Shao-Bo Lin is supported in part by the National Natural Science Foundation of China [Project No. 61876133]. This work of Yuan Yao is supported in part by Hong Kong Research Grant Council Project NO. RGC16308321 and 16303817, NSFC/RGC Joint Research Scheme N\_HKUST635/20, and ITF UIM/390. The work of Ding-Xuan Zhou is supported partially by the Research Grants Council of Hong Kong [Project No. CityU 11307319], Laboratory for AI-Powered Financial Technologies  and by the Hong Kong Institute for Data Science.   This research made use of the computing resources of the X-GPU cluster supported by the Hong Kong Research Grant Council Collaborative Research Fund: C6021-19EF.
}

\appendix

\section{Proof of Theorem \ref{theorem:app}}\label{Sec.Proof.th1}

To prove Theorem \ref{theorem:app}, we need the following ``product-gate'' for shallow sigmoid nets, which can be found  in \citep[Proposition 1]{Chui-Lin-Zhou19}.

\begin{lemma}\label{Lemma:sigmoid-product-gate}
Let $M>0$.  For any  $\nu\in(0,1)$  there exists a shallow sigmoid net
$h^{prod}_{9,\nu}:\mathbb R^2\rightarrow\mathbb R$
with 9 free parameters bounded by  $ \mathcal O(\nu^{-6})$  such that for any $t,t'\in[-M,M]$,
$$
       |tt'- h^{prod}_{9,\nu}(t,t')|\leq
       \nu.
$$
\end{lemma}

Then, we can give the proof of Theorem \ref{theorem:app} as follows.
\begin{proof}[Proof of Theorem \ref{theorem:app}]
Let $\sigma_0(t)$ be the heaviside function, i.e., $\sigma_0(t)=\left\{\begin{array}{cc}
1,&t\geq0\\
0,&t<0.\end{array}\right.$ Then, $\sigma_{relu}(t)=t\sigma_0(t)$.
A direct computation yields $\sigma(0)=1/2$ and
$\sigma'(0)=1/4$.
For $0<\mu<1/2$, according to the Taylor formula
$$
           \sigma(\mu t)=\frac12+
            \frac{\mu t}4+ \int_{0}^{\mu t}(\sigma'(u)-\sigma'(0))du,
$$
we have
$$
    t= \frac{4}\mu\sigma(\mu t)-\frac2\mu-\frac4{\mu}\int_{0}^{\mu t}(\sigma'(u)-\sigma'(0))du.
$$
Therefore,
\begin{align*}
\left|t-\frac{4}\mu\sigma(\mu t)- \frac4\mu\sigma(0 \cdot t)\right|\leq \frac4{\mu}\int_{0}^{\mu t}|\sigma'(u)-\sigma'(0)|du
    \leq
    \frac4{\mu}\max_{v\geq0}|\sigma''(v)|\int_{0}^{\mu t}udu
    \leq 2\mu t^2.
\end{align*}
Denote
$$
     h^{linear}_{2,\mu}=\frac{4}\mu\sigma(\mu t)- \frac4\mu\sigma(0 \cdot t).
$$
Then for
$
    |t|\leq M
$, there holds
\begin{equation}\label{app-1}
    |t-h^{linear}_{2,\mu}(t)|\leq 2M_0^2\mu,
\end{equation}
where $M_0>0$ satisfying $M_0 +M_0^2 = M$.
This shows that $h^{linear}_{2,\mu}$ is a good approximation of $t$. On the other hand,
for $\epsilon, \tau>0$ and $A=\frac1\tau\log\frac1\epsilon$, we have
$$
     \sigma(At)=\frac{1}{1+e^{-At}}\leq\frac{1}{1+e^{A\tau}}\leq\epsilon,  \qquad t\leq-\tau
$$
and
$$
     |\sigma(At)-1|\leq \frac{e^{-A\tau}}{1+e^{-A\tau}}\leq e^{-A\tau}\leq\epsilon,\qquad t\geq \tau,
$$
showing
\begin{equation}\label{app-2}
   |\sigma(At)-\sigma_0(t)|\leq\epsilon,\qquad t\in[-M_0,-\tau]\cup[\tau,M_0].
\end{equation}
Since $|\sigma(At)|\leq 1$ and $|h^{linear}_{2,\mu}(t)|\leq |t|+2M_0^2\mu\leq M_0+M_0^2$ for $|t|\leq M_0$, we then utilize the ``product-gate'' exhibited in Lemma \ref{Lemma:sigmoid-product-gate} with $M=M_0+M_0^2$  to construct a deep sigmoid net with two hidden layers and at most 27 free parameters to approximate $\sigma_{relu}(t)$.
Define
$$
    h^{relu}_{9,2,\mu,\nu,A}(t)=h^{prod}_{9,\nu}\left(\sigma(At),h^{linear}_{2,\mu}(t)\right)
$$
for $t\in[-M_0,M_0]$.  We then have from Lemma \ref{Lemma:sigmoid-product-gate}, (\ref{app-1}) and (\ref{app-2}) that for any  $t\in[-M_0,-\tau]\cup[\tau,M_0]$
\begin{eqnarray*}
     &&|t\sigma_0(t)-h^{relu}_{9,2,\mu,\nu,A}(t)|\\
     &\leq&
     |t\sigma_0(t)-t\sigma(At)|+|t\sigma(At)-h^{linear}_{2,\mu}(t)\sigma(At)|
     +
     |h^{linear}_{2,\mu}(t)\sigma(At)-h^{relu}_{9,2,\mu,\nu,A}(t)|\\
     &\leq&
     M_0\epsilon+2M_0^2\mu+\nu
\end{eqnarray*}
and
$$
   | h^{relu}_{9,2,\mu,\nu,A}(t) |\leq  C\nu^{-6},\qquad\forall t\in[-M,M].
$$
Let $\epsilon=\mu=\nu=\varepsilon$. We have for any $0<\varepsilon<1/2$,
\begin{equation}\label{step.1}
    |\sigma_{relu}(t)-h^{relu}_{9,2,\mu,\nu,A}(t)|\leq (M_0+1+2M_0^2)\varepsilon,\qquad t\in[-M_0,-\tau]\cup[\tau,M_0]
\end{equation}
and the free parameters of $h^{relu}_{9,2,\mu,\nu,A}(t)$ are bounded by $\max\{\mathcal O(\frac1{\varepsilon^6}),\frac1\tau\log\frac1\varepsilon\}$.
Then, setting $\tau=\varepsilon^7$, we have
\begin{eqnarray*}
    &&\int_{-M}^M|\sigma_{relu}(t)-h^{relu}_{9,2,\mu,\nu,A}(t)|^pdt
    =\left(\int_{_M}^{-\tau}+\int_{-\tau}^\tau +\int_{\tau}^M\right)|\sigma_{relu}(t)-h^{relu}_{9,2,\mu,\nu,A}(t)|^p\\
    &\leq&
    2M\varepsilon+ 2C\tau\varepsilon^{-6}
    \leq 2(M+C)\varepsilon.
 \end{eqnarray*}
This completes the proof of Theorem \ref{theorem:app} by a simple scaling.
\end{proof}

\section{Generic convergence of ADMM without normalization}
\label{app:ADMM-GenericDNN}

In this appendix, we consider more general settings than that in Section \ref{sc:convergence-result}, where $X$ and $Y$ are not necessarily normalized with unit norms, and the numbers of neurons of hidden layers can be different, and the activation function $\sigma$ can be any twice differentiable activation satisfying the following assumptions.
\begin{assumption}
\label{Assump:activ-fun}
Let $\sigma: \mathbb{R} \rightarrow \mathbb{R}$ be a twice-differentiable bounded function with bounded first- and second-order derivatives, namely, there exist positive constants $L_0 (\geq \frac{1}{8}), L_1, L_2$ such that: $|\sigma(u)|\leq L_0$, $|\sigma'(u)|\leq L_1$ and $|\sigma''(u)|\leq L_2$ for any $u\in \mathbb{R}$.
Moreover,  $\sigma$ is either a real analytic \citep[Definition 1.1.5]{Krantz2002-real-analytic} or semialgebraic function \citep{Bochnak-semialgebraic1998}.
\end{assumption}
Besides the sigmoid activation, some typical activations satisfying Assumption \ref{Assump:activ-fun} include the sigmoid-type activations \citep{Lin-CFN2019} such as the hyperbolic tangent activation. For the abuse use of notation, in this appendix, we still use $\sigma$ as any activation satisfying Assumption \ref{Assump:activ-fun}.
Before presenting our main theorem under these generic settings, we define the following constants:
\begin{align}
&L_3:= 2(L_1^2 + L_2L_0 + L_2), \label{Eq:constant-L3}\\
&\gamma:= \max_{1\leq i \leq N} \|W_i^0\|_F, \label{Eq:gamma}\\
&d_{\min}:= \min_{1\leq i \leq N-1} d_i, \label{Eq:dmin}\\
&f_{\min}:= \sqrt{6}\left(\sqrt{3L_1}+2(L_0L_3)^{1/2}(nd_{\min})^{1/4}\right), \label{Eq:fmin}\\
&\alpha_3:= \left(\frac{f_{\min}}{L_1} \right)^2, \label{Eq:alpha3}\\
&C_3:= \max \left\{ \max_{0\leq j\leq N-2} \frac{2L_0 \sqrt{nd_{j+1}}}{\gamma^j}, \frac{\|Y\|_F}{(\beta_N-3)\gamma^{N-1}}\right\}, \label{Eq:C3}
\end{align}
\begin{align*}
&\tilde{\lambda}_i:= 3L_1C_3\beta_i\gamma^{i-3}(4C_3\gamma^{i-1}+L_0\sqrt{nd_i}) \left(1+\sqrt{\frac{6L_3C_3^2\gamma^{2i-2}}{L_1(4C_3\gamma^{i-1}+L_0\sqrt{nd_i})}} \right), \ 2\leq i\leq N-1,\\
&\bar{\lambda}:= \max_{2\leq i\leq N-1}\left\{ \tilde{\lambda}_i, \frac{1}{6} (1+3L_1^{-1}L_2L_3\gamma^{i-1})^2C_3^2\gamma^{2(i-2)}\beta_i \right\}, \nonumber\\
&\hat{\lambda}:= L_1\beta_1 \|X\|_F (4C_3 + L_0\sqrt{nd_1})\gamma^{-1}\left(1+\sqrt{\frac{2L_3C_3\|X\|_F\gamma}{L_1(4C_3 + L_0\sqrt{nd_1})}}\right).\nonumber
\end{align*}

With these defined constants, we impose some conditions on the the penalty parameters $\{\beta_i\}_{i=1}^N$ in the augmented Lagrangian,  the regularization parameter $\lambda$, the minimal number of hidden neurons $d_{\min}$, and the initializations of $\{V_i^0\}_{i=1}^N$ and $\{\Lambda_i^0\}_{i=1}^N$ as follows
\begin{align}
&\beta_N \geq 3.5, \label{Eq:cond-betaN}\\
&\frac{\beta_{N-1}}{\beta_N} \geq 16 \gamma^2, \label{Eq:cond-betaN-1-N}\\
&\frac{\beta_i}{\beta_{i+1}} \geq \max \left\{ 6\sqrt{N}(2L_1^2+(4L_3+L_2)C_3\gamma^i)\gamma^2, 6(\sqrt{3L_1}+\sqrt{2L_3C_3\gamma^i})^2\gamma^2\right\},
\ i=1,\ldots,N-2, \label{Eq:cond-betai-i+1}\\
&\lambda \geq \max \left\{12\beta_NC_3^2\gamma^{2N-4}, \bar{\lambda}, \hat{\lambda} \right\}, \label{Eq:cond-lambda}\\
&d_{\min} \geq \frac{\left(\max\left\{ \sqrt{24N+1}L_1-\sqrt{18L_1},0 \right\}\right)^4}{n(24L_0L_3)^2}, \label{Eq:cond-dmin}\\
&\|V_i^0\|_F \leq 3C_3 \gamma^{i-1}, \quad \|\Lambda_i^0\|_F \leq C_3\beta_i\gamma^{i-1}, \quad i=1,\ldots, N. \label{Eq:cond-initial}
\end{align}

Under these assumptions, we state the main convergence theorem of ADMM as follows.
\begin{theorem}
\label{Thm:global-generic}
Let Assumption \ref{Assump:activ-fun} hold.
Let $\{\cQ^k:=(\{W_i^k\}_{i=1}^N, \{V_i^k\}_{i=1}^N, \{\Lambda_i^k\}_{i=1}^N)\}$ be a sequence generated by Algorithm \ref{alg:ADMM}
with $h_i^k = {\mathbb L}(\|V_{i}^{k-1} - \beta_i^{-1}\Lambda_i^{k-1}\|_{\max})$ for $i= 1,\ldots, N-1$.
and $\mu_j^k = {\mathbb L}(\|V_{j+1}^{k-1} - \beta_{j+1}^{-1}\Lambda_{j+1}^{k-1}\|_{\max})$ for $j=1,\ldots, N-2$, where $\mathbb{L}(\cdot)$ is defined in \eqref{Eq:Lipschitz-constant}.
Assume that \eqref{Eq:cond-betaN}-\eqref{Eq:cond-initial} hold, then the following hold:
\begin{enumerate}
\item[(a)] $\{\cL(\cQ^k)\}$ is convergent.

\item[(b)] $\{\cQ^k\}$ converges to a stationary point $\cQ^*:= (\{W_i^*\}_{i=1}^N, \{V_i^*\}_{i=1}^N, \{\Lambda_i^*\}_{i=1}^N)$ of $\cL$, which is also a KKT point \eqref{Eq:kkt-cond} of problem \eqref{Eq:dnn-L2-admm-reg}, implying $\{W_i^*\}_{i=1}^N$ is a stationary point of problem \eqref{Eq:dnn-L2-org} with $\lambda' = 2\lambda/n$.

\item[(c)] $\frac{1}{K}\sum_{k=1}^K \|\nabla \cL(\cQ^k)\|_F^2 \rightarrow 0$ at a ${\cal O}(\frac{1}{K})$ rate.
\end{enumerate}
\end{theorem}

Theorem \ref{Thm:Conv-ADMM-sigmoid} presented in the context is a special case of Theorem \ref{Thm:global-generic} with $\gamma = 1$, $\|X\|_F = \|Y\|_F =1$, $\|W_i^0\|_F = 1, \ i=1,\ldots, N$, and the initialization strategy \eqref{Eq:initialization-admm}.
Actually, the initialization strategy \eqref{Eq:initialization-admm} satisfies \eqref{Eq:cond-initial} shown as follows:
\begin{align}
&\|V_j^0\|_F \leq L_0\sqrt{nd_j} \leq \frac{1}{2}C_3\gamma^{j-1}, \ j=1, \ldots, N-1, \label{Eq:Vj0}\\
&\|V_N^0\|_F \leq \gamma\cdot \frac{1}{2}C_3\gamma^{N-2} = \frac{1}{2}C_3 \gamma^{N-1}, \label{Eq:VN0}\\
&\|\Lambda_i^0\|_F =0, \ i=1, \ldots, N, \nonumber
\end{align}
where the first inequality in \eqref{Eq:Vj0} holds by the boundedness of activation, and the second inequality in \eqref{Eq:Vj0} holds by the definition \eqref{Eq:C3} of $C_3$, and the inequality in \eqref{Eq:VN0} holds for $\|W_N^0\|_F \leq \gamma$ and \eqref{Eq:Vj0} with $j=N-1$.
By the definitions \eqref{Eq:def-hatQ} and \eqref{Eq:def-hatL} of $\hcQ^k$ and $\hcL$ , if we can show that Theorem \ref{Thm:globconv-hatQk} holds under the assumptions of Theorem \ref{Thm:global-generic}, then we directly yield Theorem \ref{Thm:global-generic}.
Thus, we only need to prove Theorem \ref{Thm:globconv-hatQk} under the assumptions of Theorem \ref{Thm:global-generic}.

\section{Preliminaries}
\label{app:preliminaries}
Before presenting the  proof of Theorem \ref{Thm:globconv-hatQk} under the assumptions of Theorem \ref{Thm:global-generic},
we provide some preliminary definitions and lemmas which serve as the basis of our proof.

\subsection{Dual expressed by primal}
\label{app:dual-expressed-primal}

According to the specific updates of Algorithm \ref{alg:ADMM}, we show that the updates of dual variables $\{\Lambda_i^k\}_{i=1}^N$ can be expressed explicitly by the updates of primal variables $\{W_i^k\}_{i=1}^N$ and $\{V_i^k\}_{i=1}^N$ as in the following lemma.

\begin{lemma}[Dual expressed by primal]
\label{Lemm:dual-expressed-primal}
Suppose that Assumption \ref{Assump:activ-fun} holds.
Let $\{{\cal Q}^k := \left(\{W_i^k\}_{i=1}^N, \{V_i^k\}_{i=1}^N, \{\Lambda_i^k\}_{i=1}^N\right)\}$ be a sequence generated by Algorithm \ref{alg:ADMM}.
Then we have
\begin{align}
\label{Eq:update-Lambdank*}
&\Lambda_N^k = V_N^k -Y, \quad \forall k \in \mathbb{N},\\
&\Lambda_{N-1}^k
= ({W_N^k})^T \Lambda_N^k + \beta_N {(W_N^k)}^T (V_N^k - V_N^{k-1}),  \label{Eq:update-Lambdan-1k*}\\
&\Lambda_j^k
= ({W_{j+1}^k})^T \left(\Lambda_{j+1}^k \odot \sigma'(W_{j+1}^k V_j^{k-1}) \right) +\beta_{j+1}({W_{j+1}^k})^T
\left[ \left( (\sigma(W_{j+1}^kV_j^{k-1}) - \sigma(W_{j+1}^k V_j^k)) \right.\right.\nonumber\\
&\left. \left. + (V_{j+1}^k - V_{j+1}^{k-1})\right)\odot \sigma'(W_{j+1}^k V_j^{k-1}) + \mu_j^k W_{j+1}^k(V_j^k - V_j^{k-1})/2 \right],
\ j=N-2,\ldots, 1. \label{Eq:update-Lambdajk*}
\end{align}
\end{lemma}

\begin{proof}
We firstly derive the explicit updates of $\{W_i^k\}_{i=1}^N$ and $\{V_j^k\}_{j=1}^N$,
then based on these updates, we prove Lemma \ref{Lemm:dual-expressed-primal}.

\textbf{1) $W_i$-subproblems:}
According to the update \eqref{Eq:Wnk-prox}, $W_N^k$ is updated via
\begin{align}
\label{Eq:update-WNk*}
W_N^k = (\beta_NV_N^{k-1} - \Lambda_N^{k-1})(V_{N-1}^{k-1})^T \left(\lambda {\bf I} + \beta_N V_{N-1}^{k-1}{V_{N-1}^{k-1}}^T\right)^{-1}.
\end{align}
By \eqref{Eq:Wik-prox}, for $i=1,\ldots, N-1$, we get
\begin{align}
W_i^k
&= W_i^{k-1}\frac{\beta_ih_i^k}{2}V_{i-1}^{k-1}{(V_{i-1}^{k-1})}^T \left(\lambda {\bf I} + \frac{\beta_ih_i^k}{2}V_{i-1}^{k-1}{(V_{i-1}^{k-1})}^T\right)^{-1}
\nonumber\\
&- \left[\left(\Lambda_i^{k-1} + \beta_i(\sigma(W_i^{k-1}V_{i-1}^{k-1})-V_{i}^{k-1}) \right)\odot \sigma'(W_i^{k-1}V_{i-1}^{k-1}) \right]
{(V_{i-1}^{k-1})}^T \left(\lambda {\bf I} + \frac{\beta_ih_i^k}{2}V_{i-1}^{k-1}{(V_{i-1}^{k-1})}^T\right)^{-1} \nonumber\\
& = W_i^{k-1} - W_i^{k-1}\left({\bf I} + \frac{\beta_ih_i^k}{2\lambda }V_{i-1}^{k-1}{(V_{i-1}^{k-1})}^T\right)^{-1} \label{Eq:update-Wik*}\\
&- \left[\left(\Lambda_i^{k-1} + \beta_i(\sigma(W_i^{k-1}V_{i-1}^{k-1})-V_{i}^{k-1}) \right)\odot \sigma'(W_i^{k-1}V_{i-1}^{k-1}) \right]
{(V_{i-1}^{k-1})}^T \left(\lambda {\bf I} + \frac{\beta_ih_i^k}{2}V_{i-1}^{k-1}{V_{i-1}^{k-1}}^T\right)^{-1}. \nonumber
\end{align}
Particularly, when $i=1$, $W_1^k$ is updated by
\begin{align}
W_1^k
&= W_1^{k-1}\frac{\beta_1h_1^k}{2}V_0{V_0}^T \left(\lambda {\bf I} + \frac{\beta_1h_1^k}{2}V_0{V_0}^T\right)^{-1} \nonumber\\
&- \left[\left(\Lambda_1^{k-1} + \beta_1(\sigma(W_1^{k-1}V_0)-V_1^{k-1}) \right)\odot \sigma'(W_1^{k-1}V_0) \right]
{V_0}^T \left(\lambda {\bf I} + \frac{\beta_1h_1^k}{2}V_0{V_0}^T\right)^{-1} \nonumber \\
&=W_1^{k-1} - W_1^{k-1}\left({\bf I} + \frac{\beta_1h_1^k}{2\lambda}V_0{V_0}^T\right)^{-1} \label{Eq:update-W1k*} \\
&- \left[\left(\Lambda_1^{k-1} + \beta_1(\sigma(W_1^{k-1}V_0)-V_1^{k-1}) \right)\odot \sigma'(W_1^{k-1}V_0) \right]
{V_0}^T \left(\lambda {\bf I} + \frac{\beta_1h_1^k}{2}V_0{V_0}^T\right)^{-1}. \nonumber
\end{align}

\textbf{2) $V_j$-subproblems:}
According to \eqref{Eq:Vnk-prox}, it holds
\begin{align}
\label{Eq:opt0-vN}
V_N^k - Y - \left[\Lambda_N^{k-1} + \beta_N \left(W_N^k V_{N-1}^k - V_N^k\right)\right]=0.
\end{align}
By the relation $\Lambda_N^k = \Lambda_N^{k-1} + \beta_N \left(W_N^k V_{N-1}^k - V_N^k\right)$, \eqref{Eq:opt0-vN} implies
\begin{align}
\label{Eq:update-LambdaN-VN}
\Lambda_N^k = V_N^k -Y, \quad \forall k \in \mathbb{N},
\end{align}
which shows \eqref{Eq:update-Lambdank*} in Lemma \ref{Lemm:dual-expressed-primal}.
Substituting the equality \eqref{Eq:update-LambdaN-VN} with the index value $k-1$ into \eqref{Eq:opt0-vN} yields
\begin{align}
\label{Eq:update-Vnk}
V_N^k = \frac{1}{1+\beta_N}V_N^{k-1} + \frac{\beta_N}{1+\beta_N} W_N^k V_{N-1}^k.
\end{align}

According to \eqref{Eq:Vn-1k-prox},
it holds
\begin{align*}
-\left[\Lambda_{N-1}^{k-1}+\beta_{N-1}(\sigma(W_{N-1}^kV_{N-2}^k)-V_{N-1}^k)\right] + {(W_N^k)}^T\left[\Lambda_N^{k-1}+\beta_N\left(W_N^kV_{N-1}^{k} -V_N^{k-1}\right)\right] =0,
\end{align*}
which implies
\begin{align}
&V_{N-1}^k = \label{Eq:update-Vn-1k} \\
& (\beta_{N-1} {\bf I} + \beta_N {(W_N^k)}^TW_N^k)^{-1} \left[ \Lambda_{N-1}^{k-1}+ \beta_{N-1}\sigma(W_{N-1}^k V_{N-2}^k) - {(W_N^k)}^T \left(\Lambda_N^{k-1} - \beta_N V_N^{k-1}\right)\right], \nonumber
\end{align}
and together with the updates of $\Lambda_{N-1}^k$ and $\Lambda_{N}^k$ in Algorithm \ref{alg:ADMM} yields
\begin{align*}
\Lambda_{N-1}^k
=(W_N^k)^T\left[ \Lambda_N^{k-1} + \beta_N \left(W_N^kV_{N-1}^k - V_N^{k-1} \right)\right]
= {(W_N^k)}^T \left(\Lambda_N^k + \beta_N(V_N^k - V_N^{k-1})\right).
\end{align*}
This implies \eqref{Eq:update-Lambdan-1k*} in Lemma \ref{Lemm:dual-expressed-primal}.

By \eqref{Eq:Vjk-prox}, for $j=1,\ldots, N-2$, $V_j^k$ satisfies the following optimality condition
\begin{align*}
&-\left[\Lambda_j^{k-1} + \beta_j(\sigma(W_j^kV_{j-1}^k)-V_j^k) \right] +\frac{\beta_{j+1}\mu_j^k}{2}{W_{j+1}^k}^TW_{j+1}^k (V_j^k - V_j^{k-1})\\
&+ {W_{j+1}^k}^T \left[ \left(\Lambda_{j+1}^{k-1} + \beta_{j+1} ( \sigma(W_{j+1}^kV_j^{k-1})-V_{j+1}^{k-1}) \right) \odot \sigma'(W_{j+1}^k V_j^{k-1})\right]
 =0,
\end{align*}
which implies
\begin{align}
V_j^k
&= \left( \beta_j {\bf I} + \frac{\beta_{j+1}\mu_j^k}{2} {W_{j+1}^k}^T W_{j+1}^k \right)^{-1}
\left[ \frac{1}{2}\beta_{j+1}\mu_j^k{W_{j+1}^k}^T W_{j+1}^kV_j^{k-1} + \left(\Lambda_j^{k-1} + \beta_j \sigma(W_j^kV_{j-1}^k)\right) \right.
\nonumber\\
& \left. +  {W_{j+1}^k}^T \left( [\Lambda_{j+1}^{k-1} + \beta_{j+1} ( \sigma(W_{j+1}^kV_j^{k-1})-V_{j+1}^{k-1})]
\odot \sigma'(W_{j+1}^k V_j^{k-1})\right) \right] \nonumber
\\
&= V_j^{k-1} - \left( {\bf I} + \frac{\beta_{j+1}\mu_j^k}{2\beta_j } {W_{j+1}^k}^T W_{j+1}^k \right)^{-1}V_j^{k-1} \nonumber\\
&+\left( \beta_j {\bf I} + \frac{\beta_{j+1}\mu_j^k}{2} {W_{j+1}^k}^T W_{j+1}^k \right)^{-1}
\left[\left( \Lambda_j^{k-1} + \beta_j \sigma(W_j^kV_{j-1}^k)\right) \right. \nonumber\\
& \left. +  {W_{j+1}^k}^T \left( \left[\Lambda_{j+1}^{k-1} + \beta_{j+1} ( \sigma(W_{j+1}^kV_j^{k-1})-V_{j+1}^{k-1})\right]
\odot \sigma'(W_{j+1}^k V_j^{k-1})\right) \right]. \label{Eq:update-Vjk*}
\end{align}
and together with the updates of $\Lambda_{j}^k$ and $\Lambda_{j+1}^k$ in Algorithm \ref{alg:ADMM} yields
\begin{align*}
\Lambda_j^k
&= {(W_{j+1}^k)}^T \left[ \left(\Lambda_{j+1}^{k-1} + \beta_{j+1} \left( \sigma(W_{j+1}^kV_j^{k-1})-V_{j+1}^{k-1}\right) \right)
\odot \sigma'(W_{j+1}^k V_j^{k-1})\right]
\\
&+ \frac{\beta_{j+1}\mu_j^k}{2}{(W_{j+1}^k)}^TW_{j+1}^k (V_j^k - V_j^{k-1}), \nonumber\\
&= {(W_{j+1}^k)}^T \left(\Lambda_{j+1}^k \odot \sigma'(W_{j+1}^k V_j^{k-1}) \right) +\beta_{j+1}{(W_{j+1}^k)}^T
\left[ \left( (\sigma(W_{j+1}^kV_j^{k-1}) - \sigma(W_{j+1}^k V_j^k)) \right.\right.\nonumber\\
&\left. \left.+ (V_{j+1}^k - V_{j+1}^{k-1})\right)\odot \sigma'(W_{j+1}^k V_j^{k-1}) + \mu_j^k W_{j+1}^k(V_j^k - V_j^{k-1})/2 \right]. 
\end{align*}
The final equality implies \eqref{Eq:update-Lambdajk*} in Lemma \ref{Lemm:dual-expressed-primal}.
This completes the proof of this lemma.
\end{proof}

\subsection{Kurdyka-{\L}ojasiewicz property}
\label{app:Kurdyka-Lojasiewicz-inequality}

The Kurdyka-{\L}ojasiewicz (KL) property \citep{Lojasiewicz-KL1993,Kurdyka-KL1998} plays a crucial role in the convergence analysis of nonconvex algorithm (see, \cite{Attouch2013}).
The following definition is adopted from \citep{Bolte-KL2007a}.

\begin{definition}[KL property]
\label{Def-KLProp}
An extended real valued function $h: {\cal X} \rightarrow \mathbb{R}\cup \{+\infty\}$
is said to have the Kurdyka-{\L}ojasiewicz property at $x^*\in\mathrm{dom}(\partial h)$ if there exist a neighborhood $U$ of $x^*$, a constant $\eta>0$, and a continuous concave function $\phi(s) = cs^{1-\theta}$ for some $c>0$ and $\theta \in [0,1)$
such that the Kurdyka-{\L}ojasiewicz inequality holds
\begin{equation}
\phi'(h(x)-h(x^*)) \mathrm{dist}(0,\partial h(x))\geq 1, \ \forall x \in U \cap \mathrm{dom}(\partial h) \ \text{and} \ h(x^*)<h(x) <h(x^*)+\eta,
\label{Eq:KLIneq}
\end{equation}
where $\partial h(x)$ denotes the \textit{limiting-subdifferential} of $h$ at $x\in \mathrm{dom}(h)$ (introduced in \cite{Mordukhovich-2006}), $\mathrm{dom}(h):=\{x\in {\cal X}: h(x)<+\infty\}$, $\mathrm{dom}(\partial h) :=\{x\in {\cal X}: \partial h(x)\neq \emptyset\}$, and $\mathrm{dist}(0,\partial h(x)) := \min \{\|z\|: z\in \partial h(x)\}$, where $\|\cdot\|$ represents the Euclidean norm.

Proper lower semi-continuous functions which satisfy the Kurdyka-{\L}ojasiewicz inequality at each point of $\mathrm{dom}(\partial h)$ are called KL functions.
\end{definition}

Note that we have adopted in the definition of KL inequality \eqref{Eq:KLIneq} the following notational conventions: $0^0=1, \infty/\infty=0/0=0.$
This property was firstly introduced by \citep{Lojasiewicz-KL1993} on real analytic functions \citep{Krantz2002-real-analytic} for $\theta \in [\frac{1}{2},1)$, then was extended to functions defined on the $o$-minimal structure in \citep{Kurdyka-KL1998}, and later was extended to nonsmooth subanalytic functions in \citep{Bolte-KL2007a}.
In the following, we give the definitions of real-analytic and semialgebraic functions.

\begin{definition}[Real analytic, Definition 1.1.5 in \citep{Krantz2002-real-analytic}]
\label{Def:real-analytic}
A function $h$ with domain being an open set $U\subset \mathbb{R}$ and range either the real or the complex numbers, is said to be \textbf{real analytic} at $u$ if the function $f$ may be represented by a convergent power series on some interval of positive radius centered at $u$:
\[
h(x) = \sum_{j=0}^{\infty} \alpha_j(x-u)^j,
\]
for some $\{\alpha_j\} \subset \mathbb{R}$.
The function is said to be \textbf{real analytic} on $V\subset U$ if it is real analytic at each $u \in V.$
The real analytic function $f$ over $\mathbb{R}^p$ for some positive integer $p>1$ can be defined similarly.
\end{definition}

According to \citep{Krantz2002-real-analytic}, some typical real analytic functions include polynomials, exponential functions, and the logarithm, trigonometric and power functions  on any open set of their domains.
One can verify whether a multivariable real function $h({\bf x})$ on $\mathbb{R}^p$ is analytic by checking the analyticity of $g(t):= h({\bf x}+ t {\bf y})$ for any ${\bf x}, {\bf y} \in \mathbb{R}^p$.
The following lemma shows some important properties of real analytic functions.

\begin{lemma}[\citealp{Krantz2002-real-analytic}]
\label{Lemm:real-analytic}
The sums, products, and compositions of real analytic functions are real analytic functions.
\end{lemma}

Let $h: \mathbb{R}^p \to \mathbb{R} \cup \{+\infty\}$ be an extended-real-valued function (respectively, $h:\mathbb{R}^p \rightrightarrows \mathbb{R}^q$ be a point-to-set mapping), its \textit{graph} is defined by
\begin{align*}
&\mathrm{Graph}(h) := \{({x},y)\in \mathbb{R}^p \times \mathbb{R}: y = h({x})\}, \\
(\text{resp.}\; &\mathrm{Graph}(h) := \{({x},{y})\in \mathbb{R}^p \times \mathbb{R}^q: {y} \in h({x})\}),
\end{align*}
and its domain by $\mathrm{dom}(h):=\{{x}\in \mathbb{R}^p: h({x})<+\infty\}$ (resp. $\mathrm{dom}(h) :=\{{x}\in\mathbb{R}^p: h({x})\neq \emptyset\}$).

\begin{definition}[Semialgebraic]\hfill
\label{Def:semialgebraic}
\begin{enumerate}
\item[(a)]
A set ${\cal D} \subset \mathbb{R}^p$ is called semialgebraic \citep{Bochnak-semialgebraic1998} if it can be represented as
\[
{\cal D} = \bigcup_{i=1}^s \bigcap_{j=1}^t \{x\in \mathbb{R}^p: P_{ij}(x) = 0, Q_{ij}(x)>0\},
\]
where $P_{ij}, Q_{ij}$ are real polynomial functions for $1\leq i \leq s, 1\leq j \leq t.$

\item[(b)]
A function $h:\mathbb{R}^p\rightarrow \mathbb{R}\cup \{+\infty\}$ (resp. a point-to-set mapping $h:\mathbb{R}^p \rightrightarrows \mathbb{R}^q$) is called \textit{semialgebraic} if its graph $\mathrm{Graph}(h)$ is semialgebraic.
\end{enumerate}
\end{definition}

According to \citep{Lojasiewicz1965-semianalytic,Bochnak-semialgebraic1998} and \cite[I.2.9, p.52]{Shiota1997-subanalytic}, the class of semialgebraic sets is stable under the operation of finite union, finite intersection, Cartesian product or complementation. Some typical examples include \text{polynomial} functions, the indicator function of a semialgebraic set, and the \text{Euclidean norm} \cite[p.26]{Bochnak-semialgebraic1998}.

\begin{lemma}[KL properties of $\cL$ and $\hcL$]
\label{Lemm:KL-property}
Suppose that Assumption \ref{Assump:activ-fun} holds, then both $\cL$ and $\hcL$ are KL functions.
\end{lemma}

\begin{proof}
Let $\cQ:=\left(\{W_i\}_{i=1}^N, \{V_i\}_{i=1}^N, \{\Lambda_i\}_{i=1}^N \right)$, $\hcQ:=\left(\cQ,\{\hat{V}_i\}_{i=1}^N \right)$  and
\begin{align*}
&\cL_1(\cQ):= \frac{1}{2}\|V_N-Y\|_F^2 + \frac{\lambda}{2} \sum_{i=1}^N \|W_i\|_F^2 + \sum_{i=1}^N \frac{\beta_i}{2} \|\sigma_i(W_iV_{i-1}) - V_i\|_F^2,\\
&\cL_2(\cQ):= \sum_{i=1}^N \langle \Lambda_i, \sigma_i(W_iV_{i-1}) - V_i\rangle.
\end{align*}
Then
$
\cL(\cQ) = \cL_1(\cQ) + \cL_2(\cQ),
\hcL(\hcQ) = \cL(\cQ) + \sum_{i=1}^N \xi_i \|V_i - \hat{V}_i\|_F^2,
$
where $\xi_i>0$, $i=1,\ldots,N$.
According to the same arguments as in the proof of \cite[Proposition 2]{Zeng-BCD19}, $\cL_1$ is real analytic (resp. semialgebraic) if $\sigma_i$ is real analytic (resp. semialgebraic).
By the closedness of real analytic (resp. semialgebraic) functions under the sum, product and composition (see, \cite{Krantz2002-real-analytic,Bochnak-semialgebraic1998}), we can show that $\cL_2$ is also real analytic (resp. semialgebraic) if $\sigma_i$ is real analytic (resp. semialgebraic).
Thus, ${\cL}$ is a finite sum of real analytic or semialgebraic functions. According to \cite{Shiota1997-subanalytic}, $\cL$ is a subanalytic function.
By Assumption \ref{Assump:activ-fun}, $\cL$ is continuous. Thus, $\cL$ is a KL function by \cite[Theorem 3.1]{Bolte-KL2007a}.
Since $\sum_{i=1}^N \xi_i \|V_i - \hat{V}_i\|_F^2$ is polynomial, $\hcL$ is also a KL function by a similar argument.
This completes the proof.
\end{proof}

\section{Proofs for Theorem \ref{Thm:globconv-hatQk}}
\label{app:proof-main-theorem}

As stated in \Cref{sc:keystone-proofidea}, the main idea of proof of Theorem \ref{Thm:globconv-hatQk} is shown as follows:
we firstly establish the desired sufficient descent lemma (see, Lemma \ref{Lemm:suff-descent}) via estimating the progress made by one step update, and bounding dual by primal as well as showing the boundedness of the sequence,
then develop the desired relative error lemma (see, Lemma \ref{Lemm:bound-grad}) via the optimality conditions of all subproblems, the Lipschitz continuity of the activation as well as the boundedness of the sequence,
and finally prove this theorem via \citep[Theorem 2.9]{Attouch2013}, together with Lemma \ref{Lemm:KL-property} and the continuous assumption of the activation.
In the following, we establish these lemmas followed by the detailed proof of Theorem \ref{Thm:globconv-hatQk}.

\subsection{Proof for Lemma \ref{Lemm:suff-descent}: Sufficient descent lemma}
\label{app:sufficient-descent-lemma}

In order to prove Theorem \ref{Thm:globconv-hatQk}, the following \textit{sufficient descent} lemma plays a key role.
\begin{lemma}[Sufficient descent]
\label{Lemm:suff-descent}
Under assumptions of Theorem \ref{Thm:global-generic}, for $k\geq 2$, there holds
\begin{align}
\label{Eq:suff-descent}
{\hcL}(\hcQ^k) \leq {\hcL}(\hcQ^{k-1}) - a\left(\sum_{i=1}^N \|W_i^k - W_i^{k-1}\|_F^2 + \sum_{i=1}^N \|V_i^k - V_i^{k-1}\|_F^2\right),
\end{align}
where $a$ is some positive constant specified later in \eqref{Eq:a} in Appendix \ref{app:proof-sufficient-descent}.
\end{lemma}

From Lemma \ref{Lemm:suff-descent},
we establish the sufficient descent property of an auxiliary sequence $\{\hcQ^k\}$ instead of the sequence $\{\cQ^k\}$ itself, along a new Laypunov function $\hcL$ but not the original augmented Lagrangian $\cL$.
This is different from the convergence analysis of ADMM in \citep{Wang-ADMM2018} for linear constrained optimization problems, where the sufficient descent lemma is shown for the original sequence along the augmented Lagrangian (see, \cite[Lemma 5]{Wang-ADMM2018}).
In order to establish Lemma \ref{Lemm:suff-descent}, the following three lemmas are required,
where the first lemma shows the progress made by one step update (called, \textit{one-step progress lemma}),
the second lemma bounds the discrepancies of two successive dual updates via those of the primal updates (called, \textit{dual-bounded-by-primal lemma}),
and the third lemma shows the boundedness of the sequence (called, \textit{boundedness lemma}).

\subsubsection{Lemma \ref{Lemm:descent-two-iterates}: One-step progress lemma}
\label{app:proof-change-quantity}

We present the first lemma that estimates the progress made by a single update of ADMM.

\begin{lemma}[One-step progress]
\label{Lemm:descent-two-iterates}
Let Assumption \ref{Assump:activ-fun} hold.
Let $\big\{{\cal Q}^k := \big(\{W_i^k\}_{i=1}^N, \{V_i^k\}_{i=1}^N, $ $\{\Lambda_i^k\}_{i=1}^N\big)\big\}$ be a sequence generated by Algorithm \ref{alg:ADMM} with $\{h_i^k\}_{i=1}^{N-1}$ and $\{\mu_j^k\}_{j=1}^{N-2}$ specified in \eqref{Eq:hik} and \eqref{Eq:mujk}, respectively.
Then for any integer $k\geq 1$, the following holds
\begin{align}
\label{Eq:descent-two-iterates}
{\cal L}({\cal Q}^k )
&\leq {\cal L}({\cal Q}^{k-1}) - \sum_{i=1}^N \left( \frac{\lambda}{2}\|W_i^k - W_i^{k-1}\|_F^2+\frac{\beta_i h_i^k}{4}\|(W_i^k - W_i^{k-1})V_{i-1}^{k-1}\|_F^2 \right) \\
&-\sum_{j=1}^{N-1} \left( \frac{\beta_j}{2}\|V_j^k - V_j^{k-1}\|_F^2 + \frac{\beta_{j+1}\mu_j^k}{4} \|W_{j+1}^k(V_j^k - V_j^{k-1})\|_F^2\right) - \frac{1+\beta_N}{2} \|V_N^k - V_N^{k-1}\|_F^2 \nonumber\\
&+\sum_{i=1}^N \beta_i^{-1}\|\Lambda_i^k - \Lambda_i^{k-1}\|_F^2, \nonumber
\end{align}
where $V_0^k \equiv X$, $h_N^k = 1$ and $\mu_{N-1}^k =1$.
\end{lemma}

From Lemma \ref{Lemm:descent-two-iterates},
there are two key parts that contribute to the progress along the augmented Lagrangian sequence, namely, the descent part arisen by the primal updates and the ascent part brought by the dual updates.
Due to the existence of the dual ascent part, the convergence of nonconvex ADMM is usually very challengeable.
By \eqref{Eq:descent-two-iterates}, in order to further estimate the progress in terms of the primal updates,
we shall bound these dual ascent parts via the primal updates as shown in Lemma \ref{Lemm:dual-controlled-primal} below.

To prove Lemma \ref{Lemm:descent-two-iterates}, we firstly establish two preliminary lemmas.

\begin{lemma}
\label{Lemm:descent-lemma}
Given a constant $c\in \mathbb{R}$, let $f_c$ be the function on $\mathbb{R}$ given by $f_c(u) = (\sigma(u)-c)^2$. Then the following holds
\[
f_c(v) \leq f_c(u) + f_c'(u)(v-u) + \frac{\mathbb{L}(|c|)}{2} (v-u)^2, \forall u, v\in \mathbb{R}
\]
where $\mathbb{L}(|c|)$ is defined in \eqref{Eq:Lipschitz-constant}.
\end{lemma}

\begin{proof}
According to Assumption \ref{Assump:activ-fun}, by some simple derivations, we can show $|f_c''(u)| \leq \mathbb{L}(|c|), \forall u \in \mathbb{R}$.
This yields the inequality
$
f_c(v) \leq f_c(u) + f_c'(u)(v-u) + \frac{\mathbb{L}(|c|)}{2} (v-u)^2, \forall u, v\in \mathbb{R}.
$
\end{proof}

Note that the $W_i^k$ $(i=1,\ldots, N-1)$ and $V_j^k$ $(j=1,\ldots, N-2)$ updates involve the following update schemes, i.e.,
\begin{align}
& W^k = \arg\min_W \left\{\frac{\lambda}{2} \|W\|_F^2 + \beta H_{\sigma}^k (W; A, B)\right\}, \label{Eq:W-update}\\
& V^k = \arg\min_V \left\{ \frac{\lambda}{2} \|V-C\|_F^2 + \beta M_{\sigma}^k (V; A, B)\right\} \label{Eq:V-update}
\end{align}
for some matrices $A,B$ and $C$, positive constants $\lambda$ and $\beta$.
Based on Lemma \ref{Lemm:descent-lemma}, we provide a lemma to estimate the descent quantities of the above two updates.

\begin{lemma}
\label{Lemm:descent-prox-linear}
Suppose that Assumption \ref{Assump:activ-fun} holds.
Let $W^k$ and $V^k$ be updated according to \eqref{Eq:W-update} and \eqref{Eq:V-update}, respectively, then
\begin{align}
\frac{\lambda}{2}\|W^k\|_F^2 + \beta H_{\sigma}(W^k;A,B)
& \leq \frac{\lambda}{2}\|W^{k-1}\|_F^2 + \beta H_{\sigma}(W^{k-1};A,B) \label{Eq:W-descent}\\
&- \frac{\lambda}{2}\|W^k - W^{k-1}\|_F^2 -\frac{\beta h^k}{4} \|(W^k - W^{k-1})A\|_F^2, \nonumber\\
\frac{\lambda}{2}\|V^k - C\|_F^2 + \beta M_{\sigma}(V^k;A,B)
&\leq \frac{\lambda}{2}\|V^{k-1}-C\|_F^2 + \beta M_{\sigma}(V^{k-1};A,B) \label{Eq:V-descent}\\
&- \frac{\lambda}{2}\|V^k - V^{k-1}\|_F^2 -\frac{\beta \mu^k}{4} \|A(V^k - V^{k-1})\|_F^2,\nonumber
\end{align}
where $h^k:= \mathbb{L}(\|B\|_{\max})$ and $\mu^k := \mathbb{L}(\|B\|_{\max})$.
\end{lemma}

\begin{proof}
We first establish the descent inequality \eqref{Eq:W-descent} then similarly show \eqref{Eq:V-descent}.

Let $h(W):= \frac{\lambda}{2} \|W\|_F^2 + \beta H_{\sigma}^k (W;A,B)$. By Taylor's formula, the optimality of $W^k$, and noting that $h(W)$ is a quadratic function, there holds
\[
h(W^{k-1}) = h(W^k) + \frac{\lambda}{2}\|W^k - W^{k-1}\|_F^2 + \frac{\beta h^k}{4} \|(W^k - W^{k-1})A\|_F^2,
\]
which implies
\begin{align*}
&\frac{\lambda}{2}\|W^{k-1}\|_F^2 + \beta H_{\sigma}(W^{k-1};A,B) \\
&= \frac{\lambda}{2}\|W^k\|_F^2 + \beta \left( H_{\sigma}(W^{k-1};A,B) + \langle \nabla H_{\sigma}(W^{k-1};A,B), W^k - W^{k-1}\rangle + \frac{h^k}{4}\|(W^k - W^{k-1})A\|_F^2\right) \\
&+ \frac{\lambda}{2}\|W^k - W^{k-1}\|_F^2 + \frac{\beta h^k}{4} \|(W^k - W^{k-1})A\|_F^2 \\
&\geq \frac{\lambda}{2}\|W^k\|_F^2 + \beta H_{\sigma}(W^{k};A,B)
+ \frac{\lambda}{2}\|W^k - W^{k-1}\|_F^2 + \frac{\beta h^k}{4} \|(W^k - W^{k-1})A\|_F^2,
\end{align*}
where the final inequality holds by the definition \eqref{Eq:H-sigmoid} of $H_{\sigma}(W;A,B) = \frac{1}{2}\|\sigma(WA)-B\|_F^2$ and by specializing Lemma \ref{Lemm:descent-lemma} with $v=[W^kA]_{ij}$, $u=[W^{k-1}A]_{ij}$ and $c = B_{ij}$ for any $i,j$, where $[W^kA]_{ij}$ and $[W^{k-1}A]_{ij}$ are the $(i,j)$-th entries of $W^kA$ and $W^{k-1}A$, respectively.
This yields \eqref{Eq:W-descent}.

Similarly, we can establish the inequality \eqref{Eq:V-descent}. This completes the proof.
\end{proof}

Based on Lemma \ref{Lemm:descent-prox-linear}, we prove Lemma \ref{Lemm:descent-two-iterates} as follows.

\begin{proof}[Proof of Lemma \ref{Lemm:descent-two-iterates}]
We establish \eqref{Eq:descent-two-iterates} via estimating the progress for each block update.
At first, we consider the $W_N^k$ update. By \eqref{Eq:Wnk-prox}, it is easy to show
\begin{align}
&{\cal L}(W_{<N}^{k-1}, W_N^k, \{V_j^{k-1}\}_{j=1}^N, \{\Lambda_j^{k-1}\}_{j=1}^N)
\leq {\cal L}(W_{<N}^{k-1}, W_N^{k-1}, \{V_j^{k-1}\}_{j=1}^N, \{\Lambda_j^{k-1}\}_{j=1}^N) \nonumber\\
&- \frac{\lambda}{2} \|W_N^k - W_N^{k-1}\|_F^2 - \frac{\beta_N}{2} \|(W_N^k - W_N^{k-1})V_{N-1}^{k-1}\|_F^2.  \label{Eq:descent-Wnk}
\end{align}
By \eqref{Eq:Wik-prox}, $W_i^k$ $(i=1,\ldots, N-1)$ is updated according to \eqref{Eq:W-update} with $\lambda = \lambda$, $\beta =\beta_i$, $A=V_{i-1}^{k-1}$  and $B = V_{i}^{k-1} - \beta_i^{-1}\Lambda_i^{k-1}$. Then by Lemma \ref{Lemm:descent-prox-linear}, it holds
\begin{align}
&{\cal L}(W_{<i}^{k-1}, W_i^k, W_{>i}^k, \{V_j^{k-1}\}_{j=1}^N, \{\Lambda_j^{k-1}\}_{j=1}^N)
\leq {\cal L}(W_{<i}^{k-1}, W_i^{k-1}, W_{>i}^k, \{V_j^{k-1}\}_{j=1}^N, \{\Lambda_j^{k-1}\}_{j=1}^N) \nonumber\\
& - \frac{\lambda}{2}\|W_i^k - W_i^{k-1}\|_F^2 - \frac{\beta_ih_i^k}{4} \|(W_i^k - W_i^{k-1})V_{i-1}^{k-1}\|_F^2. \label{Eq:descent-Wik}
\end{align}
Similarly, for the $V_j^k$-update ($j=1,\ldots, N-2$), by \eqref{Eq:Vjk-prox} and Lemma \ref{Lemm:descent-prox-linear}, the following holds
\begin{align}
&{\cal L}(\{W_i^{k}\}_{i=1}^N, V_{<j}^{k}, V_j^k, V_{>j}^{k-1}, \{\Lambda_i^{k-1}\}_{i=1}^N)
\leq {\cal L}(\{W_i^{k}\}_{i=1}^N, V_{<j}^{k}, V_j^{k-1}, V_{>j}^{k-1}, \{\Lambda_i^{k-1}\}_{i=1}^N) \nonumber\\
& - \frac{\beta_j}{2}\|V_j^k - V_j^{k-1}\|_F^2 - \frac{\beta_{j+1}\mu_j^k}{4} \|W_{j+1}^k(V_j^k - V_j^{k-1})\|_F^2. \label{Eq:descent-Vjk}
\end{align}
For the $V_{N-1}^k$ and $V_N^k$ updates, by \eqref{Eq:Vn-1k-prox} and \eqref{Eq:Vnk-prox}, we can easily obtain the following
\begin{align}
& {\cal L}(\{W_i^{k}\}_{i=1}^N, V_{<N-1}^{k}, V_{N-1}^k, V_{N}^{k-1}, \{\Lambda_i^{k-1}\}_{i=1}^N)
\leq {\cal L}(\{W_i^{k}\}_{i=1}^N, V_{<N-1}^{k}, V_{N-1}^{k-1}, V_{N}^{k-1}, \{\Lambda_i^{k-1}\}_{i=1}^N) \nonumber\\
& - \frac{\beta_{N-1}}{2} \|V_{N-1}^k - V_{N-1}^{k-1}\|_F^2 - \frac{\beta_N}{2} \|W_N^k(V_{N-1}^k - V_{N-1}^{k-1})\|_F^2 \label{Eq:descent-Vn-1k}
\end{align}
and
\begin{align}
&{\cal L}(\{W_i^{k}\}_{i=1}^N, V_{<N}^{k}, V_{N}^k, \{\Lambda_i^{k-1}\}_{i=1}^N)
\leq {\cal L}(\{W_i^{k}\}_{i=1}^N, V_{<N}^{k}, V_{N}^{k-1}, \{\Lambda_i^{k-1}\}_{i=1}^N) \nonumber\\
&- \frac{1+\beta_N}{2} \|V_{N}^k - V_{N}^{k-1}\|_F^2.\label{Eq:descent-Vnk}
\end{align}
Particularly, by the updates of $\Lambda_j^k$ $(j=1,\ldots, N)$, we have
\begin{align}
&{\cal L}(\{W_i^k\}_{i=1}^N, \{V_j^k\}_{j=1}^N, \{\Lambda_i^k\}_{i=1}^N) \nonumber\\
&= {\cal L}(\{W_i^k\}_{i=1}^N, \{V_j^k\}_{j=1}^N, \{\Lambda_i^{k-1}\}_{i=1}^N) + \sum_{i=1}^N \langle \Lambda_i^k - \Lambda_i^{k-1},\sigma_i(W_i^kV_{i-1}^k)-V_i^k \rangle \nonumber\\
&= {\cal L}(\{W_i^k\}_{i=1}^N, \{V_j^k\}_{j=1}^N, \{\Lambda_i^{k-1}\}_{i=1}^N) + \sum_{i=1}^N \beta_i^{-1} \|\Lambda_i^k - \Lambda_i^{k-1}\|_F^2. \label{Eq:ascent-Lambda}
\end{align}
Summing up \eqref{Eq:descent-Wnk}-\eqref{Eq:ascent-Lambda} yields \eqref{Eq:descent-two-iterates}.
\end{proof}

\subsubsection{Lemma \ref{Lemm:dual-controlled-primal}: Dual-bounded-by-primal lemma}
\label{app:proof-dual-by-primal}

By Lemma \ref{Lemm:descent-two-iterates}, how to control the amount of ascent part brought by the dual updates via the amount of descent part characterized by the primal updates is very important.
The following lemma shows that the dual ascent quantity $\{\|\Lambda_j^k - \Lambda_j^{k-1}\|^2_F\}_{j=1}^N$ can be bounded by the discrepancies between two successive primal updates $\{\|W_i^k - W_i^{k-1}\|^2_F\}_{i=1}^N$,
$\{\|V_i^k - V_i^{k-1}\|^2_F\}_{i=1}^N$, and $\{\|V_i^{k-1} - V_i^{k-2}\|^2_F\}_{i=1}^N$ via a recursive way.

\begin{lemma}[Dual-bounded-by-primal]
\label{Lemm:dual-controlled-primal}
Let Assumption \ref{Assump:activ-fun} hold.
For any positive integer $k\geq 2$, the following hold
\begin{align}
&\|\Lambda_N^k - \Lambda_N^{k-1}\|_F = \|V_N^k - V_N^{k-1}\|_F, \label{Eq:dual-N}\\
&\|\Lambda_{N-1}^k - \Lambda_{N-1}^{k-1}\|_F
\leq \|W_N^k\|_F\cdot \|\Lambda_N^k - \Lambda_N^{k-1}\|_F + \|\Lambda_N^{k-1}\|_F\cdot \|W_N^k - W_N^{k-1}\|_F \nonumber\\
&\quad \quad \quad \quad \quad \quad \quad \quad +\beta_N \|W_N^k\|_F \cdot \|V_N^k - V_N^{k-1}\|_F + \beta_N \|W_N^{k-1}\|_F\cdot \|V_N^{k-1} - V_N^{k-2}\|_F, \label{Eq:dual-N-1}
\end{align}
and for $j=1,\ldots,N-2,$
\begin{align}
\|\Lambda_j^k - \Lambda_j^{k-1}\|_F
& \leq L_1 \|W_{j+1}^k\|_F \cdot \|\Lambda_{j+1}^k - \Lambda_{j+1}^{k-1}\|_F \label{Eq:dual-j}\\
&+\left(L_1\|\Lambda_{j+1}^{k-1}\|_F + L_2 \|W_{j+1}^{k-1}\|_F \cdot \|\Lambda_{j+1}^{k-1}\|_F\cdot \|V_j^{k-1}\|_F\right)\cdot \|W_{j+1}^k - W_{j+1}^{k-1}\|_F \nonumber\\
&+L_1\beta_{j+1}\left(\|W_{j+1}^k\|_F \cdot \|V_{j+1}^k - V_{j+1}^{k-1}\|_F + \|W_{j+1}^{k-1}\|_F\cdot\|V_{j+1}^{k-1}-V_{j+1}^{k-2}\|_F\right) \nonumber\\
&+\left(L_1^2 + \frac{\mu_j^k}{2} \right)\beta_{j+1}\|W_{j+1}^k\|_F^2 \cdot \|V_j^k - V_j^{k-1}\|_F \nonumber\\
&+\left( (L_1^2 + \mu_j^{k-1}/2)\cdot\beta_{j+1} + L_2 \|\Lambda_{j+1}^{k-1}\|_F\right) \|W_{j+1}^{k-1}\|_F^2 \cdot \|V_j^{k-1} - V_j^{k-2}\|_F, \nonumber
\end{align}
where $L_1$ and $L_2$ are two constants specified in Assumption \ref{Assump:activ-fun}.
\end{lemma}

From Lemma \ref{Lemm:dual-controlled-primal},
the amount of the dual ascent part at $j$-th layer is related to all the later layers (i.e., $i=j+1,\ldots, N$) via a recursive way.
Besides these terms $\{\|W_i^k - W_i^{k-1}\|^2_F\}_{i=1}^N$ and $\{\|V_i^k - V_i^{k-1}\|^2_F\}_{i=1}^N$ exist in the upper bounds,
the discrepancies between the previous two updates $\{\|V_i^{k-1} - V_i^{k-2}\|^2_F\}_{i=1}^N$ are also involved in the upper bounds.
This may bring some challenge to construct the Lyapunov function such that the sequence or its variant is a descent sequence,
because in this case, the augmented Lagrangian shall not be an appropriate Lyapunov function by Lemma \ref{Lemm:descent-two-iterates},
where the amount of descent part is only characterized by $\{\|W_i^k - W_i^{k-1}\|^2_F\}_{i=1}^N$ and $\{\|V_i^k - V_i^{k-1}\|^2_F\}_{i=1}^N$ without $\{\|V_i^{k-1} - V_i^{k-2}\|^2_F\}_{i=1}^N$.

\begin{proof}
The equality \eqref{Eq:dual-N} follows directly from \eqref{Eq:update-Lambdank*}.
By the update \eqref{Eq:update-Lambdan-1k*} of $\Lambda_{N-1}^k$, the following holds
\begin{align*}
&\Lambda_{N-1}^k - \Lambda_{N-1}^{k-1}\\
& = {(W_N^k)}^T\Lambda_N^k - {(W_N^{k-1})}^T \Lambda_N^{k-1} + \beta_N {(W_N^k)}^T(V_N^k - V_N^{k-1}) - \beta_N {(W_N^{k-1})}^T(V_N^{k-1}-V_N^{k-2})\\
& = {(W_N^k)}^T (\Lambda_N^k - \Lambda_N^{k-1}) + (W_N^k - W_N^{k-1})^T \Lambda_N^{k-1} + \beta_N {(W_N^k)}^T(V_N^k - V_N^{k-1}) \\
& - \beta_N{(W_N^{k-1})}^T(V_N^{k-1}-V_N^{k-2}),
\end{align*}
which implies \eqref{Eq:dual-N-1} directly by the triangle inequality.
For $j=1, \ldots, N-2$, by the update of \eqref{Eq:update-Lambdajk*},
\begin{align*}
&\Lambda_j^k - \Lambda_j^{k-1}
= {(W_{j+1}^k)}^T\left( \Lambda_{j+1}^k \odot \sigma'(W_{j+1}^k V_j^{k-1})\right)
- {(W_{j+1}^{k-1})}^T\left( \Lambda_{j+1}^{k-1} \odot \sigma'(W_{j+1}^{k-1} V_j^{k-2})\right)\\
&+\beta_{j+1}{(W_{j+1}^k)}^T \bigg[ \left((\sigma(W_{j+1}^kV_j^{k-1})-\sigma(W_{j+1}^k V_j^k))+(V_{j+1}^k - V_{j+1}^{k-1})\right)\odot \sigma'(W_{j+1}^k V_j^{k-1})\\
& + \frac{\mu_j^k}{2}W_{j+1}^k (V_j^k - V_j^{k-1}) \bigg]\\
&-\beta_{j+1}{(W_{j+1}^{k-1})}^T \bigg[ \left((\sigma(W_{j+1}^{k-1}V_j^{k-2})-\sigma(W_{j+1}^{k-1} V_j^{k-1}))+(V_{j+1}^{k-1} - V_{j+1}^{k-2})\right)\odot \sigma'(W_{j+1}^{k-1} V_j^{k-2}) \\
& + \frac{\mu_j^{k-1}}{2}W_{j+1}^{k-1} (V_j^{k-1} - V_j^{k-2}) \bigg].
\end{align*}
By Assumption \ref{Assump:activ-fun} and the triangle inequality, the above equality implies that
\begin{align}
&\|\Lambda_j^k - \Lambda_j^{k-1}\|_F
\leq  \|{(W_{j+1}^k)}^T\left( \Lambda_{j+1}^k \odot \sigma'(W_{j+1}^k V_j^{k-1})\right)
- {(W_{j+1}^{k-1})}^T\left( \Lambda_{j+1}^{k-1} \odot \sigma'(W_{j+1}^{k-1} V_j^{k-2})\right)\|_F \nonumber\\
&+\beta_{j+1}\|{W_{j+1}^k}\|_F\left( (L_1^2 + \mu_j^k/2)\|W_{j+1}^k\|_F \|V_j^k - V_j^{k-1}\|_F + L_1 \|V_{j+1}^k - V_{j+1}^{k-1}\|_F\right) \label{Eq:lambda-j-k-k-1}\\
&+\beta_{j+1}\|{W_{j+1}^{k-1}}\|_F\left( (L_1^2 + \mu_j^{k-1}/2)\|W_{j+1}^{k-1}\|_F \|V_j^{k-1} - V_j^{k-2}\|_F + L_1 \|V_{j+1}^{k-1} - V_{j+1}^{k-2}\|_F\right).\nonumber
\end{align}
Note that
\begin{align}
\label{Eq:lambda-j-k-k-1-part1}
& \|{(W_{j+1}^k)}^T\left( \Lambda_{j+1}^k \odot \sigma'(W_{j+1}^k V_j^{k-1})\right) - {(W_{j+1}^{k-1})}^T\left( \Lambda_{j+1}^{k-1} \odot \sigma'(W_{j+1}^{k-1} V_j^{k-2})\right)\|_F \nonumber\\
& \leq \|W_{j+1}^k - W_{j+1}^{k-1}\|_F \|\Lambda_{j+1}^{k} \odot \sigma'(W_{j+1}^k V_j^{k-1})\|_F \nonumber\\
&+ \|W_{j+1}^{k-1}\|_F \|\Lambda_{j+1}^k \odot \sigma'(W_{j+1}^k V_j^{k-1}) - \Lambda_{j+1}^{k-1} \odot \sigma'(W_{j+1}^{k-1} V_j^{k-2})\|_F \nonumber\\
&\leq L_1 \|\Lambda_{j+1}^{k}\|_F\|W_{j+1}^k - W_{j+1}^{k-1}\|_F
+L_1\|W_{j+1}^{k-1}\|_F \|\Lambda_{j+1}^k - \Lambda_{j+1}^{k-1}\|_F \nonumber\\
&+L_2 \|W_{j+1}^{k-1}\|_F \|\Lambda_{j+1}^{k-1}\|_F \|V_j^{k-1}\|_F \|W_{j+1}^k - W_{j+1}^{k-1}\|_F \nonumber\\
&+L_2 \|W_{j+1}^{k-1}\|_F^2 \|\Lambda_{j+1}^{k-1}\|_F \|V_j^{k-1} - V_j^{k-2}\|_F,
\end{align}
where the final inequality holds for
\begin{align*}
&\|\Lambda_{j+1}^k \odot \sigma'(W_{j+1}^k V_j^{k-1}) - \Lambda_{j+1}^{k-1} \odot \sigma'(W_{j+1}^{k-1} V_j^{k-2})\|_F \\
&\leq \|(\Lambda_{j+1}^k - \Lambda_{j+1}^{k-1})\odot \sigma'(W_{j+1}^k V_j^{k-1})\|_F
+ \|\Lambda_{j+1}^{k-1} \odot (\sigma'(W_{j+1}^k V_j^{k-1}) - \sigma'(W_{j+1}^{k-1}V_j^{k-2}))\|_F \\
&\leq L_1 \|\Lambda_{j+1}^k - \Lambda_{j+1}^{k-1}\|_F + L_2 \|\Lambda_{j+1}^{k-1}\|_F \|W_{j+1}^kV_j^{k-1} - W_{j+1}^{k-1}V_j^{k-2}\|_F\\
&\leq L_1 \|\Lambda_{j+1}^k - \Lambda_{j+1}^{k-1}\|_F +
L_2 \|\Lambda_{j+1}^{k-1}\|_F \left(\|W_{j+1}^k - W_{j+1}^{k-1}\|_F\|V_j^{k-1}\|_F + \|W_{j+1}^{k-1}\|_F \|V_j^{k-1}-V_j^{k-2}\|_F \right)
\end{align*}
by Assumption \ref{Assump:activ-fun} and the triangle inequality.
Substituting \eqref{Eq:lambda-j-k-k-1-part1} into \eqref{Eq:lambda-j-k-k-1} yields \eqref{Eq:dual-j}.
This completes the proof of this lemma.
\end{proof}

\subsubsection{Lemma \ref{Lemm:boundedness-seq}: Boundedness lemma}
\label{app:proof-boundedness-sequence}

Note that in the upper bounds of Lemma \ref{Lemm:dual-controlled-primal}, the terms $\{\|W_i^k - W_i^{k-1}\|_F\}_{i=1}^N$, $\{\|V_i^k - V_i^{k-1}\|_F\}_{i=1}^N$ and $\{\|V_i^{k-1} - V_i^{k-2}\|^2_F\}_{i=1}^N$ are multiplied by many other terms including $\{\|W_i^k\|_F\}_{i=1}^N$, $\{\|V_i^k\|_F\}_{i=1}^N$, $\{\|\Lambda_i^k\|_F\}_{i=1}^N$, and the locally Lipschitz constants $\{h_i^k := {\mathbb L}(\|V_{i}^{k-1} - \beta_i^{-1}\Lambda_i^{k-1}\|_{\max})\}_{i=1}^{N-1}$ and
$\{\mu_j^k := {\mathbb L}(\|V_{j+1}^{k-1} - \beta_{j+1}^{-1}\Lambda_{j+1}^{k-1}\|_{\max})\}_{j=1}^{N-2}$, highly depending on the current or previous updates.
In order to make these bounds in Lemma \ref{Lemm:dual-controlled-primal} only depend on those desired terms,
the following boundedness property of the sequence is required.

Instead of the conditions of Theorem \ref{Thm:global-generic},
we impose the following weaker conditions:
\begin{align}
&\beta_N \geq 3.5, \label{Eq:betaN}\\
&\frac{\beta_{N-1}}{\beta_N} \geq 7\gamma^2, \label{Eq:betaN*} \\
&\frac{\beta_i}{\beta_{i+1}} \geq 6\left(\sqrt{3L_1} + \sqrt{2L_3C_3\gamma^i}\right)^2\gamma^2,
\quad i=1,\ldots, N-2, \label{Eq:betaj*}\\
&\lambda \geq \max\left\{\hat{\lambda}, 12\beta_NC_3^2\gamma^{2N-4}, \max_{2\leq j\leq N-1} \tilde{\lambda}_j  \right\}, \label{Eq:lambda}\\
&\|W_i^0\|_F\leq \gamma, \quad \|V_i^0\|_F \leq 3C_3 \gamma^{i-1}, \quad \|\Lambda_i^0\|_F \leq C_3\beta_i\gamma^{i-1}, \quad i=1,\ldots, N. \label{Eq:initial-cond}
\end{align}
It can be seen that the conditions \eqref{Eq:betaN}-\eqref{Eq:betaj*} on $\beta_i$'s are slightly weaker than the conditions \eqref{Eq:cond-betaN}-\eqref{Eq:cond-betai-i+1}.

\begin{lemma}[Boundedness]
\label{Lemm:boundedness-seq}
Under Assumption \ref{Assump:activ-fun} and the above conditions \eqref{Eq:betaN}-\eqref{Eq:initial-cond}, for any $k\in \mathbb{N}$, there hold
\begin{align}
\label{Eq:boundedness}
&\|W_i^k\|_F \leq \gamma, \quad \|V_i^k\|_F \leq 3C_3\gamma^{i-1}, \quad \|\Lambda_i^k\|_F \leq C_3\beta_i\gamma^{i-1}, \quad i=1,\ldots, N, \\
& h_i^k \leq 4L_3 C_3 \gamma^{i-1}, \quad i=1, \ldots,  N-1, \label{Eq:hi-bound-sigmoid}\\
& \mu_i^k \leq 4L_3C_3 \gamma^i, \quad i=1, \ldots, N-2, \label{Eq:mui-bound-sigmoid}
\end{align}
where $\gamma:= \max_{1\leq i \leq N} \|W_i^0\|_F$ (particularly, $\gamma = 1$ in the normalized case), $C_3$ and $L_3$ are
specified later in \eqref{Eq:C3} and \eqref{Eq:constant-L3}, respectively.
\end{lemma}

The boundedness of the sequence is mainly derived by the specific updates of the algorithm and the introduced $\ell_2$ regularization.

\begin{proof}[Proof of Lemma \ref{Lemm:boundedness-seq}]
We first show that the boundedness condition holds for $k=1$.
By the definitions of \eqref{Eq:Lipschitz-constant} and \eqref{Eq:constant-L3}, it holds
\[
\mathbb{L}(|c|) \leq L_3 |c|, \quad \forall \ |c| \geq 1.
\]
By the settings \eqref{Eq:hik}, \eqref{Eq:mujk} of $h_i^k$ and $\mu_i^k$, and \eqref{Eq:initial-cond},
\begin{align}
\label{Eq:hi-bound}
&h_i^1 \leq \mathbb{L}(\|V_i^0\|_F + \beta_i^{-1}\|\Lambda_i^0\|_F) \leq \mathbb{L}(4 C_3 \gamma^{i-1}) \leq 4L_3 C_3 \gamma^{i-1}, \quad i=1,\ldots, N-1, \\
\label{Eq:mui-bound}
&\mu_i^1 \leq \mathbb{L}(\|V_{i+1}^0\|_F + \beta_{i+1}^{-1}\|\Lambda_{i+1}^0\|_F) \leq \mathbb{L}(4C_3 \gamma^i) \leq 4L_3C_3 \gamma^i, \quad  i=1,\ldots, N-2,
\end{align}
where the final inequalities in both \eqref{Eq:hi-bound} and \eqref{Eq:mui-bound} hold for $4 C_3 \gamma^{i-1} \geq 1$ by the definition \eqref{Eq:C3} and $L_0 \geq \frac{1}{8}$ in Assumption \ref{Assump:activ-fun}.
In the following, we show that \eqref{Eq:initial-cond} holds.

\textbf{(1) On boundedness of $W_N^1$.}
By \eqref{Eq:update-WNk*},
\begin{align*}
\|W_N^1\|_F
& \leq \lambda^{-1} \cdot 12 \beta_N C_3^2 \gamma^{2N-3} \leq \gamma,
\end{align*}
where the last inequality follows from the assumption \eqref{Eq:lambda} of $\lambda$.

\textbf{(2) On boundedness of $W_i^1$, $i=N-1,\ldots, 2$.}
By \eqref{Eq:update-Wik*},
\begin{align*}
\|W_i^1\|_F
&\leq \left(1-\frac{\lambda}{\lambda+18\beta_iL_3C_3^3\gamma^{3i-5}}\right)\gamma
+ \frac{3L_1\beta_iC_3\gamma^{i-2}(4C_3\gamma^{i-1}+L_0\sqrt{nd_i})}{\lambda}.
\end{align*}
To make $\|W_i^1\|_F \leq \gamma$, it requires
$
\lambda \geq \frac{a_i + \sqrt{a_i^2 + 4a_ib_i}}{2},
$
where
$
 a_i := 3L_1 \beta_i C_3 \gamma^{i-3}(4C_3\gamma^{i-1}+L_0\sqrt{nd_i}),
$
and
$
 b_i = 18 \beta_i L_3C_3^3 \gamma^{3i-5}.
$
By the assumption \eqref{Eq:lambda} of $\lambda$, we have
\begin{align*}
\lambda \geq a_i \left(1+\sqrt{\frac{b_i}{a_i}} \right) \geq \frac{a_i + \sqrt{a_i^2 + 4a_ib_i}}{2}.
\end{align*}
Thus,
$\|W_i^1\|_F \leq \gamma$ for $i=2,\ldots,N-1.$

\textbf{(3) On boundedness of $W_1^1$.}
By \eqref{Eq:update-W1k*},
\begin{align*}
\|W_1^1\|_F \leq \left(1-\frac{\lambda}{\lambda+2\beta_1L_3C_3\|X\|_F^2} \right)\gamma
+ \frac{L_1\beta_1 \|X\|_F(4C_3+L_0\sqrt{nd_1})}{\lambda}.
\end{align*}
Similarly, by the assumption of $\lambda$ \eqref{Eq:lambda}, we can show that if
\begin{align*}
\lambda \geq a_1\left(1+\sqrt{\frac{b_1}{a_1}} \right) \geq \frac{a_1 + \sqrt{a_1^2 + 4a_1b_1}}{2},
\end{align*}
where
$
a_1:= L_1\beta_1\|X\|_F(4C_3+L_0\sqrt{nd_1})\gamma^{-1}, \ b_1:= 2\beta_1L_3C_3\|X\|_F^2,
$
then
$
\|W_1^1\|_F \leq \gamma.
$

\textbf{(4) On boundedness of $V_j^1$, $j=1,\ldots,N-2$.}
By \eqref{Eq:update-Vjk*},
\begin{align*}
\|V_j^1\|_F \leq \left(1-\frac{\rho_j}{\rho_j+2L_3C_3\gamma^{j+2}} \right)\cdot 3C_3\gamma^{j-1} + (C_3\gamma^{j-1}+L_0\sqrt{nd_j})
+ \frac{L_1\gamma(4C_3\gamma^j + L_0\sqrt{nd_{j+1}})}{\rho_j},
\end{align*}
where $\rho_j:=\frac{\beta_j}{\beta_{j+1}}$.
To guarantee $\|V_j^1\|_F \leq 3C_3\gamma^{j-1}$, it requires
\begin{align*}
\rho_j \geq \frac{\bar{b}_j + \sqrt{\bar{b}_j^2 + 4\bar{a}_j\bar{c}_j}}{2\bar{a}_j},
\end{align*}
where $\bar{a}_j = 2-\frac{L_0\sqrt{nd_j}}{C_3\gamma^{j-1}},$ $\bar{b}_j = 2L_3C_3\gamma^{j+2}+2L_3\gamma^3L_0\sqrt{nd_j}+4L_1\gamma^2 + \frac{L_1L_0\sqrt{nd_{j+1}}}{C_3\gamma^{j-2}},$ and $\bar{c}_j = 2L_1L_3\gamma^4(4C_3\gamma^j + L_0\sqrt{nd_{j+1}}).$
By the definition of $C_3$ \eqref{Eq:C3},
\[
\bar{a}_j \geq \frac{3}{2}, \quad \bar{b}_j \leq \left(4.5L_1+3L_3C_3\gamma^j\right)\gamma^2, \quad \bar{c}_j \leq 9L_1L_3C_3\gamma^{j+4},
\]
where the bound on $\bar{b}_j$ follows from the following facts
\begin{align*}
2L_0\sqrt{nd_j} \leq C_3\gamma^{j-1},\quad
\frac{L_1L_0\sqrt{nd_{j+1}}}{C_3\gamma^{j-2}} \leq \frac{1}{2}L_1\gamma^2,
\end{align*}
and the bound on $ \bar{c}_j$ is due to $L_0\sqrt{nd_{j+1}} \leq \frac{1}{2}C_3\gamma^{j}$.

Thus, it yields
\begin{align*}
\frac{\bar{b}_j + \sqrt{\bar{b}_j^2 + 4\bar{a}_j\bar{c}_j}}{2\bar{a}_j}
&\leq \frac{1}{3}\bar{b}_j \left(1+\sqrt{1+\frac{6\bar{c}_j}{\bar{b}_j^2}} \right)
\leq \frac{2}{3}\bar{b}_j + \frac{\sqrt{6\bar{c}_j}}{3} \\
&\leq \left(3L_1 + \sqrt{6L_1L_3C_3\gamma^j} + 2L_3C_3\gamma^j\right)\gamma^2,
\end{align*}
where the first inequality holds for $\bar{a}_j \geq \frac{3}{2}$,
and the final inequality holds for the upper bounds of $\bar{b}_j$ and $\bar{c}_j$.
Thus, we show the boundedness of $V_j^1$ under our assumptions for any $j=1,\ldots, N-2$.

\textbf{(5) On boundedness of $V_{N-1}^1$.}
By \eqref{Eq:update-Vn-1k},
\begin{align*}
\|V_{N-1}^1\|_F
&\leq C_3 \gamma^{N-2} + L_0 \sqrt{nd_{N-1}} + 4C_3 \gamma^N \rho_{N-1}^{-1} \\
&\leq \frac{3}{2}C_3 \gamma^{N-2} + 4C_3 \gamma^N \rho_{N-1}^{-1}
\leq 3C_3 \gamma^{N-2}
\end{align*}
where
the first inequality holds by Assumption \ref{Assump:activ-fun} and \eqref{Eq:initial-cond},
the second inequality by the definition \eqref{Eq:C3} of $C_3$, and the final inequality is due to \eqref{Eq:betaN*}.

\textbf{(6) On boundedness of $V_N^1$.}
By \eqref{Eq:update-Vnk}, it shows that
\begin{align*}
\|V_N^1\|_F \leq \frac{3C_3\gamma^{N-1}}{1+\beta_N} + \frac{\beta_N}{1+\beta_N} \gamma \cdot 3C_3\gamma^{N-2} \leq 3C_3\gamma^{N-1}.
\end{align*}

\textbf{(7) On boundedness of $\Lambda_N^1$.}
By \eqref{Eq:update-Lambdank*},
\begin{align*}
\|\Lambda_N^1\|_F \leq \|V_N^1\|_F + \|Y\|_F \leq 3C_3\gamma^{N-1} + \|Y\|_F \leq C_3 \beta_N\gamma^{N-1},
\end{align*}
where the final inequality holds by the definition \eqref{Eq:C3} of $C_3$.

\textbf{(8) On boundedness of $\Lambda_{N-1}^1$.}
By \eqref{Eq:update-Lambdan-1k*},
\begin{align*}
\|\Lambda_{N-1}^1\|_F \leq 7C_3\beta_N\gamma^N \leq C_3\beta_{N-1}\gamma^{N-2},
\end{align*}
where the final inequality is due to \eqref{Eq:betaN*}.

\textbf{(9) On boundedness of $\Lambda_j^1$, $j=N-2,\ldots, 1$ by induction.}
By \eqref{Eq:update-Lambdajk*},
\begin{align*}
\|\Lambda_j^1\|_F
&\leq \beta_{j+1}\gamma \left(2L_1L_0\sqrt{nd_{j+1}}+7L_1C_3\gamma^j + 12L_3C_3^2\gamma^{2j} \right)\\
&\leq C_3\beta_{j+1}\gamma^{j+1}(8L_1+12L_3C_3\gamma^j)
\leq C_3\beta_j\gamma^{j-1},
\end{align*}
where the second inequality holds for $2L_0\sqrt{nd_{j+1}} \leq C_3\gamma^j$, and the final inequality holds for \eqref{Eq:betaj*}.

Therefore, we have shown that \eqref{Eq:boundedness}-\eqref{Eq:mui-bound-sigmoid} hold for $k=1$.
Similarly, we can show that once \eqref{Eq:boundedness}-\eqref{Eq:mui-bound-sigmoid} hold for some $k$, then they will hold for $k+1$.
Hence, we can show \eqref{Eq:boundedness}-\eqref{Eq:mui-bound-sigmoid} hold for any $k\in \mathbb{N}$ recursively.
This completes the proof of this lemma.
\end{proof}

\subsubsection{Proof of Lemma \ref{Lemm:suff-descent}: Sufficient descent lemma}
\label{app:proof-sufficient-descent}

To prove Lemma \ref{Lemm:suff-descent}, we first present a key lemma
based on Lemma \ref{Lemm:dual-controlled-primal} and Lemma \ref{Lemm:boundedness-seq}. For any $k\geq 2$ and $j=1,\ldots, N-2$, we denote
\begin{align*}
{\cal E}_{1,j}^k
&:= (L_1\gamma)^{N-j}L_1^{-1}C_3\beta_N\gamma^{N-2}\|W_N^k - W_N^{k-1}\|_F\\
& + \sum_{i=j+1}^{N-1} (L_1\gamma)^{i-j} (C_3\beta_i\gamma^{i-2}+3L_1^{-1}L_2C_3^2\beta_i\gamma^{2i-3}) \|W_{i}^k - W_{i}^{k-1}\|_F,\\
{\cal E}_{2,j}^k
&:= (L_1\gamma)^{N-j}(1+\beta_N)L_1^{-1}\|V_N^k - V_N^{k-1}\|_F
+ (L_1\gamma)^{N-1-j}\beta_{N-1}\|V_{N-1}^k - V_{N-1}^{k-1}\|_F\\
& + \sum_{i=j+1}^{N-2} (L_1\gamma)^{i-j} \left( \beta_i + (L_1^2+2L_3C_3\gamma^i)\gamma^2 \beta_{i+1}\right)\|V_{i}^k - V_{i}^{k-1}\|_F\\
&+(L_1^2+2L_3C_3\gamma^j)\gamma^2 \beta_{j+1} \|V_j^k - V_j^{k-1}\|_F,
\end{align*}
and
\begin{align*}
{\cal E}_{3,j}^k
& := (L_1\gamma)^{N-j}\beta_N L_1^{-1} \|V_N^{k-1}-V_N^{k-2}\|_F
+ (L_1\gamma)^{N-1-j}\beta_{N-1} \|V_{N-1}^{k-1} - V_{N-1}^{k-2}\|_F\\
&+\sum_{i=j+1}^{N-2} (L_1\gamma)^{i-j} \left[ \beta_i + (L_1^2 + 2L_3C_3\gamma^i + L_2C_3\gamma^i)\gamma^2 \beta_{i+1}\right]\|V_i^{k-1} - V_i^{k-2}\|_F\\
&+ \left( L_1^2 + 2L_3C_3\gamma^j + L_2C_3\gamma^j \right)\beta_{j+1}\gamma^2 \|V_j^{k-1}-V_j^{k-2}\|_F.
\end{align*}

\begin{lemma}
\label{Lemm:dual-controlled-primal-bounded}
Under assumptions of Lemma \ref{Lemm:boundedness-seq},
for any $k\geq 2$, we have
\begin{align*}
&\|\Lambda_N^k - \Lambda_N^{k-1}\|_F  = \|V_N^k - V_N^{k-1}\|_F, \nonumber\\
&\|\Lambda_{N-1}^k - \Lambda_{N-1}^{k-1}\|_F \leq
C_3 \beta_N \gamma^{N-1} \|W_N^k - W_N^{k-1}\|_F + \gamma (1+\beta_N) \|V_N^k - V_N^{k-1}\|_F \nonumber\\
& \quad \quad \quad \quad \quad \quad \quad \quad + \beta_N \gamma \|V_N^{k-1} - V_N^{k-2}\|_F, \nonumber
\end{align*}
and for $j=1,\ldots, N-2$,
\begin{align*}
\|\Lambda_j^k - \Lambda_j^{k-1}\|_F \leq {\cal E}_{1,j}^k + {\cal E}_{2,j}^k + {\cal E}_{3,j}^k.
\end{align*}
Moreover, the above inequalities imply
\begin{align}
\label{Eq:dual-bound-primal}
\sum_{i=1}^N \|\Lambda_i^k - \Lambda_i^{k-1}\|_F^2 \leq \alpha \sum_{i=1}^N \left(\|W_i^k - W_i^{k-1}\|_F^2 + \|V_i^k - V_i^{k-1}\|_F^2 + \|V_i^{k-1} - V_i^{k-2}\|_F^2\right)
\end{align}
for some constant $\alpha>0$ specified in the proof.
\end{lemma}

\begin{proof}
The bounds of $\|\Lambda_N^k - \Lambda_N^{k-1}\|_F$ and $\|\Lambda_{N-1}^k - \Lambda_{N-1}^{k-1}\|_F$ are obvious by Lemma \ref{Lemm:dual-controlled-primal} and Lemma \ref{Lemm:boundedness-seq}.
For $j=1,\ldots, N-2$, by Lemma \ref{Lemm:dual-controlled-primal} and Lemma \ref{Lemm:boundedness-seq}, it holds
\begin{align*}
&\|\Lambda_j^k - \Lambda_j^{k-1}\|_F
\leq
L_1 \gamma \|\Lambda_{j+1}^k - \Lambda_{j+1}^{k-1}\|_F + T_{j+1}^k + I_j^k,
\end{align*}
where $T_{j+1}^k := (L_1C_3\beta_{j+1}\gamma^j + 3C_3^2L_2\beta_{j+1}\gamma^{2j}) \|W_{j+1}^k - W_{j+1}^{k-1}\|_F
+ L_1\gamma\beta_{j+1}(\|V_{j+1}^{k}-V_{j+1}^{k-1}\|_F+\|V_{j+1}^{k-1}-V_{j+1}^{k-2}\|_F),$
and
\begin{align*}
I_j^k :=(L_1^2 + 2L_3C_3\gamma^j)\beta_{j+1}\gamma^2\|V_j^k - V_j^{k-1}\|_F
+ \left( L_1^2 + 2L_3C_3\gamma^j + L_2C_3\gamma^j \right)\beta_{j+1}\gamma^2\|V_j^{k-1} - V_j^{k-2}\|_F.
\end{align*}
By the above inequality, we have
\begin{align*}
\|\Lambda_j^k - \Lambda_j^{k-1}\|_F
& \leq (L_1\gamma)^{N-1-j} \|\Lambda_{N-1}^k - \Lambda_{N-1}^{k-1}\|_F + (L_1 \gamma)^{N-2-j} T_{N-1}^k \\
& + \sum_{i=1}^{N-2-j} (L_1\gamma)^{i-1} \left( T_{j+i}^k + L_1\gamma I_{j+i}^k\right) + I_j^k.
\end{align*}
Substituting the definitions of $T_{j}^k$ and $I_j^k$ into this inequality and after some simplifications yields the desired bound for $\|\Lambda_j^k - \Lambda_j^{k-1}\|_F$.
Summing up all the above inequalities and using several times of the basic inequality $(\sum_{i=1}^p u_i)^2 \leq p \sum_{i=1}^p u_i^2$ for any $u\in \mathbb{R}^p$ yields \eqref{Eq:dual-bound-primal} with some positive constant $\alpha$.
This completes the proof.
\end{proof}

Based on Lemma \ref{Lemm:descent-two-iterates}, Lemma \ref{Lemm:boundedness-seq} and Lemma \ref{Lemm:dual-controlled-primal-bounded},  we prove Lemma \ref{Lemm:suff-descent} as follows.

\begin{proof}[Proof of Lemma \ref{Lemm:suff-descent}]
By \eqref{Eq:cond-betai-i+1} and the definition \eqref{Eq:C3} of $C_3$ , we have for $j=1,\ldots,N-2$, $\frac{\beta_j}{\beta_{j+1}} \geq f_{\min}^2 \gamma^2,$
and
\begin{align}
\label{Eq:betaj-i}
\beta_j \geq f_{\min}^{2(i-j)}\gamma^{2(i-j)}\beta_i, \quad j<i \leq N-1.
\end{align}
By \eqref{Eq:fmin}-\eqref{Eq:alpha3} and \eqref{Eq:cond-dmin}, it holds
\begin{align}
\label{Eq:cond-alpha3}
\alpha_3\geq 24N+1.
\end{align}

To prove this lemma, we first estimate $\|\Lambda_i^k - \Lambda_i^{k-1}\|_F^2$ for any $i=1,\ldots, N$.
By Lemma \ref{Lemm:dual-controlled-primal-bounded}, we get
\begin{align}
\|\Lambda_N^k - \Lambda_N^{k-1}\|_F^2 = \|V_N^k - V_N^{k-1}\|_F^2, \label{Eq:LambdaN-square}
\end{align}
and using the basic inequality $\left(\sum_{i=1}^3 a_i\right)^2 \leq 3 \sum_{i=1}^3 a_i^2$,
\begin{align}
\|\Lambda_{N-1}^k - \Lambda_{N-1}^{k-1}\|_F^2
&\leq 3C_3^2\beta_N^2 \gamma^{2(N-1)}\|W_N^k - W_N^{k-1}\|_F^2 + 3\gamma^2(1+\beta_N)^2 \|V_N^k - V_N^{k-1}\|_F^2 \label{Eq:LambdaN-1-square}\\
&+ 3\beta_N^2\gamma^2 \|V_N^{k-1}-V_N^{k-2}\|_F^2, \nonumber
\end{align}
and for $j=1,\ldots, N-2$, using the inequality $\left(\sum_{i=1}^n a_i\right)^2 \leq n \sum_{i=1}^n a_i^2$,
\begin{align}
\label{Eq:Lambdaj-square}
\|\Lambda_j^k - \Lambda_j^{k-1}\|_F^2 \leq 2(N-j) {\cal T}_{1,j}^k + 4(N-j+1) ({\cal T}_{2,j}^k+{\cal T}_{3,j}^k),
\end{align}
where
\begin{align*}
{\cal T}_{1,j}^k
&= (L_1\gamma)^{2(N-j)}L_1^{-2}C_3^2\beta_N^2\gamma^{2(N-2)}\|W_N^k - W_N^{k-1}\|_F^2\\
&+\sum_{i=j+1}^{N-1} (L_1\gamma)^{2(i-j)} (1+3L_1^{-1}L_2C_3\gamma^{i-1})^2C_3^2\beta_i^2\gamma^{2(i-2)} \|W_{i}^k - W_{i}^{k-1}\|_F^2,\\
{\cal T}_{2,j}^k
&= (L_1\gamma)^{2(N-j)}(1+\beta_N)^2L_1^{-2}\|V_N^k - V_N^{k-1}\|_F^2+(L_1\gamma)^{2(N-1-j)}\beta_{N-1}^2\|V_{N-1}^k - V_{N-1}^{k-1}\|_F^2\\
&+\sum_{i=j+1}^{N-2} (L_1\gamma)^{2(i-j)} \left[ \beta_i + (L_1^2+2L_3C_3\gamma^i)\gamma^2 \beta_{i+1}\right]^2\|V_{i}^k - V_{i}^{k-1}\|_F^2 \nonumber\\
&+(L_1^2+2L_3C_3\gamma^j)^2\gamma^4 \beta_{j+1}^2 \|V_j^k - V_j^{k-1}\|_F^2,
\end{align*}
and
\begin{align*}
{\cal T}_{3,j}^k
&=(L_1\gamma)^{2(N-j)}\beta_N^2 L_1^{-2} \|V_N^{k-1}-V_N^{k-2}\|_F^2+(L_1\gamma)^{2(N-1-j)}\beta_{N-1}^2 \|V_{N-1}^{k-1} - V_{N-1}^{k-2}\|_F^2\\
&+\sum_{i=j+1}^{N-2} (L_1\gamma)^{2(i-j)} \left[ \beta_i + (L_1^2 + 2L_3C_3\gamma^i + L_2C_3\gamma^i)\gamma^2 \beta_{i+1}\right]^2\|V_i^{k-1} - V_i^{k-2}\|_F^2\\
&+\left( L_1^2 + 2L_3C_3\gamma^j + L_2C_3\gamma^j \right)^2\beta_{j+1}^2\gamma^4 \|V_j^{k-1}-V_j^{k-2}\|_F^2.
\end{align*}

Substituting \eqref{Eq:LambdaN-square}, \eqref{Eq:LambdaN-1-square} and \eqref{Eq:Lambdaj-square} into Lemma \ref{Lemm:descent-two-iterates} and after some simplifications yields
\begin{align}
\label{Eq:descent-two-iterates-bound}
{\cL}(\cQ^k) + \sum_{i=1}^N \xi_i\|V_i^k - V_i^{k-1}\|_F^2
&\leq {\cL}(\cQ^{k-1}) + \sum_{i=1}^N \xi_i\|V_i^{k-1} - V_i^{k-2}\|_F^2\\
&- \sum_{i=1}^N \zeta_i \|W_i^k - W_i^{k-1}\|_F^2 - \sum_{i=1}^N (\eta_i - \xi_i) \|V_i^k - V_i^{k-1}\|_F^2, \nonumber
\end{align}
where
\begin{align}
& \zeta_N = \frac{\lambda}{2} - 3C_3^2\beta_{N-1}^{-1}\beta_N^2\gamma^{2(N-1)} - 2L_1^{-2} C_3^2\beta_N^2\gamma^{2(N-2)}\sum_{j=1}^{N-2} \beta_j^{-1} (N-j)(L_1\gamma)^{2(N-j)} \nonumber\\
& \zeta_i = \frac{\lambda}{2} - 2(1+3L_1^{-1}L_2L_3\gamma^{i-1})^2C_3^2\beta_i^2\gamma^{2(i-2)} \sum_{j=1}^{i-1} \beta_j^{-1}(N-j)(L_1\gamma)^{2(i-j)}, \quad i=2,\ldots, N-1,\nonumber\\
& \zeta_1 = \frac{\lambda}{2}, \nonumber
\end{align}
\begin{align*}
&\eta_N = \frac{1+\beta_N}{2} - \beta_N^{-1} -3\gamma^2\left( 1+\beta_N\right)^2 \beta_{N-1}^{-1} - \frac{4(1+\beta_N)^2}{L_1^2} \sum_{j=1}^{N-2} \beta_j^{-1}(N-j+1)(L_1\gamma)^{2(N-j)}, \nonumber
\end{align*}
\begin{align}
&\xi_N = 3\gamma^2 \beta_N^2\beta_{N-1}^{-1} + \frac{4\beta_N^2}{L_1^2} \sum_{j=1}^{N-2} \beta_j^{-1}(N-j+1)(L_1\gamma)^{2(N-j)}, \label{Eq:xi-N}
\end{align}
\begin{align*}
&\eta_{N-1} = \frac{\beta_{N-1}}{2} - 4\beta_{N-1}^2 \sum_{j=1}^{N-2} \beta_j^{-1}(N-j+1)(L_1\gamma)^{2(N-1-j)}, \nonumber
\end{align*}
\begin{align}
&\xi_{N-1} = 4\beta_{N-1}^2 \sum_{j=1}^{N-2} \beta_j^{-1}(N-j+1)(L_1\gamma)^{2(N-1-j)}, \label{Eq:xi-N-1}
\end{align}
and for $i=2,\ldots, N-2$,
\begin{align}
&\eta_i
= \frac{\beta_i}{2} - 4\left[\beta_i + (L_1^2 + 2L_3C_3\gamma^i)\gamma^2 \beta_{i+1} \right]^2\sum_{j=1}^{i-1} \beta_j^{-1}(N-j+1)(L_1\gamma)^{2(i-j)} \nonumber\\
&\ \ \ - 4(L_1^2+2L_3C_3\gamma^i)^2\gamma^4\beta_{i+1}^2\beta_i^{-1}(N-i+1), \nonumber\\
&\xi_i
= 4\left[ \beta_i + (L_1^2+2L_3C_3\gamma^i+L_2C_3\gamma^i)\gamma^2\beta_{i+1}\right]^2 \sum_{j=1}^{i-1} \beta_j^{-1}(N-j+1)(L_1\gamma)^{2(i-j)} \label{Eq:xi-i}\\
&\ \ \ + 4\left(L_1^2+2L_3C_3\gamma^i+L_2C_3\gamma^i\right)^2 \gamma^4 \beta_{i+1}^2\beta_i^{-1}(N-i+1), \nonumber
\end{align}
and $\eta_1 = \frac{\beta_1}{2} - 4(L_1^2+2L_3C_3\gamma)^2\gamma^4\beta_2^2\beta_1^{-1}N$, and
\begin{align}
&\xi_1 = 4 \left( L_1^2+2L_3C_3\gamma + L_2C_3\gamma\right)^2 \gamma^4 \beta_2^2\beta_1^{-1}N. \label{Eq:xi-1}
\end{align}

Based on \eqref{Eq:descent-two-iterates-bound}, to get \eqref{Eq:suff-descent}, we need to show that
\begin{align}
\label{Eq:positive-constant}
\zeta_i >0, \quad \eta_i - \xi_i >0, \quad i=1,\ldots, N.
\end{align}
Then, let
\begin{align}
\label{Eq:a}
a:=\min\{\zeta_i, \eta_i - \xi_i, i=1,\ldots, N\},
\end{align}
we get \eqref{Eq:suff-descent}.  In the following, we show \eqref{Eq:positive-constant}.

It is obvious that $\zeta_1 = \frac{\lambda}{2}>0$. For $i=2,\ldots, N-1$, by \eqref{Eq:betaj-i},
\begin{align*}
\zeta_i
&\geq \frac{\lambda}{2} - 2C_3^2\beta_i(1+3L_1^{-1}L_2L_3\gamma^{i-1})^2\gamma^{2(i-2)} \sum_{j=1}^{i-1}(N-j) \alpha_3^{-(i-j)} \\
&> \frac{\lambda}{2} - 2C_3^2\beta_i(1+3L_1^{-1}L_2L_3\gamma^{i-1})^2\gamma^{2(i-2)} \cdot \frac{N}{\alpha_3-1}
\geq 0,
\end{align*}
where
the final inequality is due to $\alpha_3 >24N+1$ and the assumption of $\lambda$, i.e., \eqref{Eq:cond-lambda}.
Similarly, we can show that $\zeta_N>0$ as follows
\begin{align*}
\zeta_N
&\geq \frac{\lambda}{2} - \beta_NC_3^2\gamma^{2(N-2)}\cdot\left(\frac{3}{16}+\frac{1}{8}\sum_{j=1}^{N-2} (N-j)\alpha_3^{-(N-j-1)}\right)\\
&> \frac{\lambda}{2} - \beta_NC_3^2\gamma^{2(N-2)}\cdot\left(\frac{3}{16} +\frac{N}{8(\alpha_3-1)} \right)\\
&>\frac{\lambda}{2} - \frac{1}{5}\beta_NC_3^2\gamma^{2(N-2)}> 0.
\end{align*}

At the end, we show $\eta_i - \xi_i >0$ for $i=1,\ldots, N$.
Note that
\begin{align*}
\eta_1 - \xi_1
= \frac{\beta_1}{2} - 4\left[(L_1^2+2L_3C_3\gamma)^2 + (L_1^2+2L_3C_3\gamma+L_2C_3\gamma)^2 \right]\gamma^4\beta_2^2\beta_1^{-1}N
>0,
\end{align*}
where we have used \eqref{Eq:cond-betai-i+1} $\frac{\beta_1}{\beta_2} \geq 4\sqrt{N}\left[L_1^2+(2L_3+L_2)C_3\gamma\right] \gamma^2$.

For $i=2,\ldots, N-2$,
let
$
\alpha_1:=(L_1^2+2L_3C_3\gamma^i)\gamma^2, \ \alpha_2:=(L_1^2+2L_3C_3\gamma^i+L_2C_3\gamma^i)\gamma^2.
$
Note that
\begin{align}
\label{Eq:etai-xii}
\eta_i - \xi_i
& = \frac{\beta_i}{2} - 4\left[(\beta_i+\alpha_1\beta_{i+1})^2 +(\beta_i+\alpha_2\beta_{i+1})^2 \right]\sum_{j=1}^{i-1}\beta_j^{-1}(N-j+1)(L_1\gamma)^{2(i-j)} \nonumber\\
&-4(\alpha_1^2+\alpha_2^2)\beta_{i+1}^2\beta_i^{-1}(N-i+1) \nonumber\\
&> \frac{\beta_i}{2} - 4\left[(\beta_i+\alpha_1\beta_{i+1})^2 +(\beta_i+\alpha_2\beta_{i+1})^2 \right]\sum_{j=1}^{i-1}\beta_i^{-1}(N-j+1)\alpha_3^{-(i-j)} \nonumber\\
&-4N(\alpha_1^2+\alpha_2^2)\beta_{i+1}^2\beta_i^{-1} \nonumber\\
&> \beta_i \left[\frac{1}{2} -\frac{4N}{\alpha_3-1}\left((1+\alpha_1 \cdot \frac{\beta_{i+1}}{\beta_i})^2 + (1+\alpha_2\cdot \frac{\beta_{i+1}}{\beta_i})^2 \right)-4N(\alpha_1^2 + \alpha_2^2)\left(\frac{\beta_{i+1}}{\beta_i} \right)^2 \right]\nonumber\\
&=\frac{4N}{\alpha_3-1}\beta_i \left[\frac{\alpha_3-1}{8N}-2 - 2(\alpha_1+\alpha_2)\cdot \frac{\beta_{i+1}}{\beta_i}-\alpha_3(\alpha_1^2+\alpha_2^2)\cdot \left(\frac{\beta_{i+1}}{\beta_i}\right)^2 \right] \nonumber\\
&>\frac{4N}{\alpha_3-1}\beta_i \left[\frac{\alpha_3-1}{8N}-2 - 2(\alpha_1+\alpha_2)\cdot \frac{\beta_{i+1}}{\beta_i}-\alpha_3(\alpha_1+\alpha_2)^2\cdot \left(\frac{\beta_{i+1}}{\beta_i}\right)^2 \right] \nonumber\\
&\geq 0,
\end{align}
where
the final inequality follows from \eqref{Eq:cond-betai-i+1},
$
\alpha_3>24N+1,
$
and
\[
\frac{1+\sqrt{1+\alpha_3\left(\frac{\alpha_3-1}{8N}-2\right)}}{\frac{\alpha_3-1}{8N}-2} \leq 1+\sqrt{24N+2} \leq 6\sqrt{N}.
\]

Similarly, notice that
\begin{align*}
\eta_{N-1} - \xi_{N-1}
& = \frac{\beta_{N-1}}{2} - 8\beta_{N-1}^2 \sum_{j=1}^{N-2} \beta_j^{-1}(N-j+1)(L_1\gamma)^{2(N-1-j)} \\
&\geq \frac{\beta_{N-1}}{2} - 8\beta_{N-1} \sum_{j=1}^{N-2} (N-j+1)\alpha_3^{-(N-1-j)}\\
&>\beta_{N-1}\left(\frac{1}{2}-\frac{8N}{\alpha_3-1} \right)
\geq \frac{1}{6}\beta_{N-1}>0.
\end{align*}
Finally, note that
\begin{align*}
\label{Eq:etaN-xiN}
\eta_N - \xi_N
&= \frac{1+\beta_N}{2} - \beta_N^{-1} -3\gamma^2 \beta_{N-1}^{-1} [(1+\beta_N)^2 + \beta_N^2]\\
&-\frac{4}{L_1^2}[(1+\beta_N)^2+\beta_N^2]\sum_{j=1}^{N-2}\beta_j^{-1}(N-j+1)(L_1\gamma)^{2(N-j)}\\
&\geq \frac{1+\beta_N}{2} - \beta_N^{-1} -3\gamma^2 \beta_{N-1}^{-1} [(1+\beta_N)^2 + \beta_N^2]\\
&-4\beta_{N-1}^{-1}\gamma^2[(1+\beta_N)^2+\beta_N^2]\sum_{j=1}^{N-2}(N-j+1)\alpha_3^{-(N-1-j)}\\
&>\frac{\beta_N^2+\beta_N-2}{2\beta_N} -2(3+\frac{4N}{\alpha_3-1})(\beta_N^2+\beta_N+1)\beta_{N-1}^{-1}\gamma^2\\
&\geq \frac{\beta_N^2+\beta_N-2}{2\beta_N} -\frac{19}{3}(\beta_N^2+\beta_N+1)\beta_{N-1}^{-1}\gamma^2\\
&> 0,
\end{align*}
where
the final inequality follows from $\beta_N \geq 3.5$, which implies
\[
16 > \frac{38}{3} \cdot \frac{\beta_N^2+\beta_N+1}{\beta_N^2+\beta_N-2},
\]
and \eqref{Eq:cond-betaN-1-N}. This completes the proof.
\end{proof}

\subsubsection{Proof of Lemma \ref{Lemm:bound-grad}: Relative error lemma}
\label{app:proof-relative-error}

In the following, we provide a lemma to show that the gradients of the augmented Lagrangian and the new Lyapunov function can be bounded by the discrepancy between two successive updates.
Such a lemma is important to show the global convergence of a descent sequence by \citep[Theorem 2.9]{Attouch2013}.

\begin{lemma}
\label{Lemm:bound-grad}
Under conditions of Theorem \ref{Thm:global-generic},
for any positive $k\geq 2$, there exists some positive constant $\bar{b}$ such that
\begin{align}
\label{Eq:bound-grad-L}
\|\nabla \cL(\cQ^k)\|_F \leq \bar{b}\sum_{i=1}^N (\|W_i^k - W_i^{k-1}\|_F + \|V_i^k - V_i^{k-1}\|_F + \|V_i^{k-1} - V_i^{k-2}\|_F),
\end{align}
and $\|\nabla \hcL(\hcQ^k)\|_F \leq \hat{b} \|\hcQ^k - \hcQ^{k-1}\|_F,$
where $\hat{b}=\sqrt{3N}b$ and $b = \bar{b}+4\max_{1\leq i\leq N} \xi_i.$
\end{lemma}

\begin{proof}
Note that
\[
\nabla \cL(\cQ^k) = \left( \left\{\frac{\partial \cL(\cQ^k)}{\partial W_i}\right\}_{i=1}^N,
\left\{\frac{\partial \cL(\cQ^k)}{\partial V_i}\right\}_{i=1}^N, \left\{\frac{\partial \cL(\cQ^k)}{\partial \Lambda_i}\right\}_{i=1}^N\right),
\]
then
\begin{align}
\label{Eq:bound-grad-L0}
\left\| \nabla \cL(\cQ^k)\right\|_F \leq \sum_{i=1}^N \left( \left\|\frac{\partial \cL(\cQ^k)}{\partial W_i}\right\|_F
+ \left\|\frac{\partial \cL(\cQ^k)}{\partial V_i}\right\|_F + \left\|\frac{\partial \cL(\cQ^k)}{\partial \Lambda_i}\right\|_F\right).
\end{align}
In order to bound $\|\nabla \cL(\cQ^k)\|_F$, we need to bound each component of $\nabla \cL(\cQ^k)$.

\textbf{On $\left\|\frac{\partial \cL(\cQ^k)}{\partial W_N}\right\|_F$:}
By the optimality condition of \eqref{Eq:Wnk-prox},
\begin{align*}
\lambda W_N^k + \beta_N(W_N^kV_{N-1}^{k-1}-V_N^{k-1}){V_{N-1}^{k-1}}^T + \Lambda_N^{k-1}{V_{N-1}^{k-1}}^T =0,
\end{align*}
which implies
\begin{align*}
&\frac{\partial \cL(\cQ^k)}{\partial W_N}
= \lambda W_N^k + \beta_N(W_N^kV_{N-1}^{k}-V_N^{k}){(V_{N-1}^{k})}^T + \Lambda_N^{k}{(V_{N-1}^{k})}^T \\
&= \beta_N \left[ (W_N^kV_{N-1}^k - V_N^k)(V_{N-1}^k - V_{N-1}^{k-1})^T +
\left(W_N^k(V_{N-1}^k - V_{N-1}^{k-1}) - (V_N^k - V_N^{k-1})\right){V_{N-1}^{k-1}}^T\right] \\
&+ \Lambda_N^{k-1}(V_{N-1}^k - V_{N-1}^{k-1})^T + (\Lambda_N^k - \Lambda_N^{k-1})(V_{N-1}^k)^T.
\end{align*}
By the boundedness of the sequence \eqref{Eq:boundedness}, the above equality yields
\begin{align}
\label{Eq:bound-grad-Wn}
\left\|\frac{\partial \cL(\cQ^k)}{\partial W_N}\right\|_F
\leq 10\beta_NC_3\gamma^{N-1}\|V_{N-1}^{k}-V_{N-1}^{k-1}\|_F + 3C_3\gamma^{N-2}(\beta_N+1) \|V_N^k - V_N^{k-1}\|_F.
\end{align}

\textbf{On $\left\|\frac{\partial \cL(\cQ^k)}{\partial W_i}\right\|_F$:} For $i=2,\ldots,N-1$, by the optimality condition of \eqref{Eq:Wik-prox},
\begin{align*}
&\lambda W_i^k + \left((\beta_i \sigma(W_i^{k-1}V_{i-1}^{k-1})-\beta_i V_i^{k-1} + \Lambda_i^{k-1})\odot \sigma'(W_i^{k-1}V_{i-1}^{k-1}) \right)
{V_{i-1}^{k-1}}^T \\
&+ \frac{\beta_i h_i^k}{2} (W_i^k - W_i^{k-1}){V_{i-1}^{k-1}}^T =0,
\end{align*}
which implies
\begin{align*}
\frac{\partial \cL(\cQ^k)}{\partial W_i}
&= \lambda W_i^k + \left((\beta_i \sigma(W_i^{k}V_{i-1}^{k})-\beta_i V_i^{k} + \Lambda_i^{k})\odot \sigma'(W_i^{k}V_{i-1}^{k}) \right)
{V_{i-1}^{k}}^T\\
&= \left((\beta_i \sigma(W_i^{k}V_{i-1}^{k})-\beta_i V_i^{k} + \Lambda_i^{k})\odot \sigma'(W_i^{k}V_{i-1}^{k}) \right)
{V_{i-1}^{k}}^T \\
&- \left((\beta_i \sigma(W_i^{k-1}V_{i-1}^{k-1})-\beta_i V_i^{k-1} + \Lambda_i^{k-1})\odot \sigma'(W_i^{k-1}V_{i-1}^{k-1}) \right)
{V_{i-1}^{k-1}}^T\\
&-\frac{\beta_i h_i^k}{2} (W_i^k - W_i^{k-1}){V_{i-1}^{k-1}}^T\\
&=\left[\beta_i\left(\sigma(W_i^kV_{i-1}^k)-\sigma(W_i^{k-1}V_{i-1}^k)+\sigma(W_i^{k-1}V_{i-1}^k) - \sigma(W_i^{k-1}V_{i-1}^{k-1})\right)\odot \sigma'(W_i^kV_{i-1}^k) \right.\\
&+ \left( \beta_i(V_i^{k-1} - V_i^k) + (\Lambda_i^k - \Lambda_i^{k-1})\right) \odot \sigma'(W_i^k V_{i-1}^k)\\
&+ \left(\beta_i\sigma(W_i^{k-1}V_{i-1}^{k-1}) -\beta_iV_i^{k-1}+\Lambda_i^{k-1}\right)
\odot \left(\sigma'(W_i^{k}V_{i-1}^{k})-\sigma'(W_i^{k-1}V_{i-1}^{k})\right)
\\
&\left. + \left(\beta_i\sigma(W_i^{k-1}V_{i-1}^{k-1}) -\beta_iV_i^{k-1}+\Lambda_i^{k-1}\right)
\odot \left(\sigma'(W_i^{k-1}V_{i-1}^{k}) - \sigma'(W_i^{k-1}V_{i-1}^{k-1}) \right) \right]{V_{i-1}^k}^T\\
&+\left((\beta_i \sigma(W_i^{k-1}V_{i-1}^{k-1})-\beta_i V_i^{k-1} + \Lambda_i^{k-1})\odot \sigma'(W_i^{k-1}V_{i-1}^{k-1}) \right) (V_{i-1}^{k} - V_{i-1}^{k-1})^T\\
&-\frac{\beta_i h_i^k}{2} (W_i^k - W_i^{k-1}){V_{i-1}^{k-1}}^T.
\end{align*}
By Assumption \ref{Assump:activ-fun} and Lemma \ref{Lemm:boundedness-seq}, the above equality yields
\begin{align}
\label{Eq:bound-grad-Wi}
&\left\|\frac{\partial \cL(\cQ^k)}{\partial W_i}\right\|_F \\
&\leq 3C_3\gamma^{i-2}
\left[3\beta_iC_3\gamma^{i-2}\left(L_1^2+L_0L_2\sqrt{nd_i}+4L_2C_3\gamma^{i-1}+\frac{2}{3}L_3\gamma\right)\|W_i^k - W_i^{k-1}\|_F \right.
\nonumber\\
&+\beta_i \left[L_1^2\gamma + (L_1 + L_2\gamma)(L_0\sqrt{nd_i}+4C_3\gamma^{i-1}) \right]\|V_{i-1}^k - V_{i-1}^{k-1}\|_F
\nonumber\\
&\left. +\beta_i L_1 \|V_i^k - V_i^{k-1}\|_F + L_1 \|\Lambda_i^k - \Lambda_i^{k-1}\|_F \right].\nonumber
\end{align}

\textbf{On $\left\|\frac{\partial \cL(\cQ^k)}{\partial W_1}\right\|_F$:}
Similarly, by the optimality condition of \eqref{Eq:Wik-prox} with $i=1$,
\begin{align*}
\lambda W_1^k + \left[\left( \beta_1 \sigma(W_1^{k-1}X)-\beta_1V_1^{k-1}+\Lambda_1^{k-1}\right) \odot \sigma'(W_1^{k-1}X)\right] X^T
+ \frac{\beta_1h_1^k}{2}(W_1^k - W_1^{k-1})X^T =0,
\end{align*}
which implies
\begin{align*}
&\frac{\partial \cL(\cQ^k)}{\partial W_1}
= \lambda W_1^k + \left[\left( \beta_1 \sigma(W_1^{k}X)-\beta_1V_1^{k}+\Lambda_1^{k}\right) \odot \sigma'(W_1^{k}X)\right] X^T\\
&=\left[\left(\beta_1(\sigma(W_1^kX)-\sigma(W_1^{k-1}X)) -\beta_1(V_1^k - V_1^{k-1}) + (\Lambda_1^k - \Lambda_1^{k-1})\right)
\odot \sigma'(W_1^kX) \right] X^T\\
&+\left[ \left( \beta_1\sigma(W_1^{k-1}X)-\beta_1V_1^{k-1} + \Lambda_1^{k-1}\right) \odot (\sigma'(W_1^kX)-\sigma'(W_1^{k-1}X))\right] X^T\\
&+\frac{\beta_1h_1^k}{2}(W_1^{k-1}-W_1^{k})X^T.
\end{align*}
The above inequality yields
\begin{align}
\label{Eq:bound-grad-W1}
\left\|\frac{\partial \cL(\cQ^k)}{\partial W_1}\right\|_F
&\leq \beta_1\|X\|_F\left(\|X\|_F\cdot (L_1^2+L_0L_2\sqrt{nd_1}+4L_2C_3)+2L_3C_3 \right) \|W_1^k - W_1^{k-1}\|_F \nonumber\\
&+\beta_1L_1\|X\|_F \|V_1^k - V_1^{k-1}\|_F + L_1\|X\|_F \|\Lambda_1^k - \Lambda_1^{k-1}\|_F.
\end{align}

\textbf{On $\left\|\frac{\partial \cL(\cQ^k)}{\partial V_j} \right\|_F$ $(1\leq j \leq N-2)$:}
By the optimality condition of \eqref{Eq:Vjk-prox},
\begin{align*}
&\beta_j(V_j^k - \sigma(W_j^kV_{j-1}^k))  + {W_{j+1}^k}^T \left[ \left(\Lambda_{j+1}^{k-1}+\beta_{j+1}(\sigma(W_{j+1}^kV_j^{k-1})-V_{j+1}^{k-1}) \right)
\odot \sigma'(W_{j+1}^kV_j^{k-1})\right] \\
&- \Lambda_j^{k-1} +\frac{\beta_{j+1}\mu_j^k}{2}{W_{j+1}^k}^TW_{j+1}^k(V_j^k - V_j^{k-1})=0,
\end{align*}
which implies
\begin{align*}
&\frac{\partial \cL(\cQ^k)}{\partial V_j} \\
&= {W_{j+1}^k}^T \left[\left((\Lambda_{j+1}^k - \Lambda_{j+1}^{k-1}) + \beta_{j+1}(\sigma(W_{j+1}^kV_j^k)-\sigma(W_{j+1}^kV_j^{k-1}))
+\beta_{j+1}(V_{j+1}^{k-1}-V_{j+1}^k) \right) \right.\\
&\left.\odot \sigma'(W_{j+1}^kV_j^k) \right]\\
&+{W_{j+1}^k}^T \left[\left(\Lambda_{j+1}^{k-1} +\beta_{j+1}(\sigma(W_{j+1}^kV_j^{k-1})-V_{j+1}^{k-1}) \right)\odot
\left(\sigma'(W_{j+1}^kV_j^k) - \sigma'(W_{j+1}^kV_j^{k-1})\right) \right]\\
&+(\Lambda_j^{k-1}-\Lambda_j^k) + \frac{\beta_{j+1}\mu_j^k}{2}{W_{j+1}^k}^TW_{j+1}^k(V_j^{k-1}-V_j^k).
\end{align*}
The above equality yields
\begin{align}
\label{Eq:bound-grad-Vj}
\left\|\frac{\partial \cL(\cQ^k)}{\partial V_j} \right\|_F
&\leq \beta_{j+1}\gamma^2\left(2C_3\gamma^j(2L_2+L_3)+L_0L_2\sqrt{nd_{j+1}}\right) \|V_j^k - V_j^{k-1}\|_F \\
&+\beta_{j+1}L_1\gamma\|V_{j+1}^k - V_{j+1}^{k-1}\|_F + \|\Lambda_j^k - \Lambda_j^{k-1}\|_F + L_1\gamma \|\Lambda_{j+1}^k - \Lambda_{j+1}^{k-1}\|_F. \nonumber
\end{align}

\textbf{On $\left\|\frac{\partial \cL(\cQ^k)}{\partial V_{N-1}} \right\|_F$:}
By the optimality condition of \eqref{Eq:Vn-1k-prox},
\[
\beta_{N-1}(V_N^k - \sigma(W_{N-1}^k V_{N-2}^k)) - \Lambda_{N-1}^{k-1} + {W_N^k}^T \left(\Lambda_N^{k-1}+\beta_N(W_N^kV_{N-1}^k - V_N^{k-1}) \right) =0,
\]
which implies
\begin{align*}
\frac{\partial \cL(\cQ^k)}{\partial V_{N-1}}
&= \beta_{N-1}(V_N^k - \sigma(W_{N-1}^k V_{N-2}^k)) - \Lambda_{N-1}^{k} + {W_N^k}^T \left(\Lambda_N^{k}+\beta_N(W_N^kV_{N-1}^k - V_N^{k}) \right)\\
&=\Lambda_{N-1}^{k-1} - \Lambda_{N-1}^k + {W_N^k}^T(\Lambda_N^k - \Lambda_N^{k-1}) + \beta_N{W_N^k}^T(V_N^{k-1}-V_N^k).
\end{align*}
The above equality implies
\begin{align}
\label{Eq:bound-grad-Vn-1}
\left\|\frac{\partial \cL(\cQ^k)}{\partial V_{N-1}} \right\|_F \leq \beta_N\gamma \|V_N^k - V_N^{k-1}\|_F
+ \|\Lambda_{N-1}^k - \Lambda_{N-1}^{k-1}\|_F + \gamma \|\Lambda_N^k - \Lambda_N^{k-1}\|_F.
\end{align}

\textbf{On $\left\|\frac{\partial \cL(\cQ^k)}{\partial V_N}\right\|_F$:}
Similarly, by the optimality condition of \eqref{Eq:Vnk-prox}, we get
\begin{align}
\label{Eq:bound-grad-Vn}
\left\|\frac{\partial \cL(\cQ^k)}{\partial V_N} \right\|_F = \|\Lambda_N^k - \Lambda_N^{k-1}\|_F.
\end{align}

Moreover, for $i=1,\ldots, N$, by the update of $\Lambda_i^k$, we can easily yield
\begin{align}
\label{Eq:bound-grad-Lambi}
\left\|\frac{\partial \cL(\cQ^k)}{\partial \Lambda_i} \right\|_F = \beta_i^{-1} \|\Lambda_i^k - \Lambda_i^{k-1}\|_F.
\end{align}

Substituting \eqref{Eq:bound-grad-Wn}-\eqref{Eq:bound-grad-Lambi} into \eqref{Eq:bound-grad-L0}, and after some simplifications, we get
\begin{align}
\label{Eq:bound-grad-L1}
\left\|\nabla \cL(\cQ^k) \right\|_F
&\leq \bar{\alpha}\sum_{i=1}^{N} (\|W_i^k - W_i^{k-1}\|_F + \|V_i^k - V_i^{k-1}\|_F + \|\Lambda_i^k - \Lambda_i^{k-1}\|_F)
\end{align}
for some $\bar{\alpha}>0.$
By Lemma \ref{Lemm:dual-controlled-primal-bounded},
substituting these upper bounds of $\|\Lambda_i^k - \Lambda_i^{k-1}\|_F$ $(i=1,\ldots,N)$ into \eqref{Eq:bound-grad-L1}
and after some simplifications implies \eqref{Eq:bound-grad-L}  for some constant $\bar{b}$.

By \eqref{Eq:bound-grad-L}, it is easy to derive
\begin{align*}
\|\nabla \hcL(\hcQ^k)\|_F
&\leq \|\nabla \cL(\cQ^k)\|_F + \sum_{i=1}^N 4\xi_i \|V_i^k - V_i^{k-1}\|_F \nonumber\\
&\leq b \sum_{i=1}^N (\|W_i^k - W_i^{k-1}\|_F+\|V_i^k - V_i^{k-1}\|_F+\|V_i^{k-1} - V_i^{k-2}\|_F) \nonumber\\
&\leq \hat{b} \|\hcQ^k - \hcQ^{k-1}\|_F,
\end{align*}
where
$b=\bar{b}+4\max_{1\leq i\leq N} \xi_i$ and $\hat{b}=\sqrt{3N}b.$
This completes the proof.
\end{proof}

\subsection{Proof of Theorem \ref{Thm:globconv-hatQk}}

Now we provide the detailed proof of Theorem \ref{Thm:globconv-hatQk} based on the above lemmas.

\begin{proof}[Proof of Theorem \ref{Thm:globconv-hatQk}]

\textbf{(a)} By Lemma \ref{Lemm:boundedness-seq}, the boundedness of $\{\cQ^k\}$ implies the sequence $\cL(\cQ^k)$ is lower bounded,
and so is $\hcL(\hcQ^k)$ by its definition \eqref{Eq:def-hatL}.
By Lemma \ref{Lemm:suff-descent}, $\hcL(\hcQ^k)$ is monotonically non-increasing, therefore, $\hcL(\hcQ^k)$ is convergent.

\textbf{(b)}
Again by Lemma \ref{Lemm:boundedness-seq}, $\hcQ^k$ is bounded,
and thus there exists a subsequence $\hcQ^{k_j}$ such that $\hcQ^{k_j} \rightarrow \hcQ^*$ as $j\rightarrow \infty$.
Since $\hcL$ is continuous by Assumption \ref{Assump:activ-fun},
then $\lim_{j\rightarrow \infty} \hcL(\hcQ^{k_j})=\hcL(\hcQ^*)$.
This implies the \textit{continuity condition} in the analysis framework formulated in \citep{Attouch2013} holds.
Together with the sufficient descent (Lemma \ref{Lemm:suff-descent}), relative error  (Lemma \ref{Lemm:bound-grad}) and Kurdyka-{\L}ojasiewicz (Lemma \ref{Lemm:KL-property}) properties, the whole sequence convergence to a stationary point is derived via following \citep[Theorem 2.9]{Attouch2013}.

\textbf{(c)} The ${\cal O}(1/K)$ rate can be easily derived by Lemma \ref{Lemm:suff-descent}, Lemma \ref{Lemm:dual-controlled-primal-bounded} and Lemma \ref{Lemm:bound-grad}.
Specifically, by Lemma \ref{Lemm:suff-descent}, it is easy to show
\begin{align}
\label{Eq:square-summable-primal}
\frac{1}{K} \sum_{k=2}^K (\|{\cal W}^k - {\cal W}^{k-1}\|_F^2+\|{\cal V}^k - {\cal V}^{k-1}\|_F^2) \leq \frac{\hcL(\hcQ^1)-\hcL(\hcQ^*)}{aK},
\end{align}
which implies
\begin{align}
\label{Eq:square-summable-hatv}
\frac{1}{K} \sum_{k=2}^K \sum_{i=1}^N \|\hat{V}_i^k - \hat{V}^{k-1}\|_F^2
= \frac{1}{K} \sum_{k=1}^{K-1} \|{\cal V}^k - {\cal V}^{k-1}\|_F^2
\leq \frac{a\|{\cal V}^1 - {\cal V}^{0}\|_F^2+(\hcL(\hcQ^1)-\hcL(\hcQ^*))}{aK}.
\end{align}
By \eqref{Eq:dual-bound-primal} in Lemma \ref{Lemm:dual-controlled-primal-bounded}, \eqref{Eq:square-summable-primal} and \eqref{Eq:square-summable-hatv}, there holds
\begin{align}
\label{Eq:square-summable-dual}
\frac{1}{K} \sum_{k=2}^K \sum_{i=1}^N \|\Lambda_i^k - \Lambda_i^{k-1}\|_F^2 \leq \bar{C}\cdot\frac{a\|{\cal V}^1 - {\cal V}^{0}\|_F^2+(\hcL(\hcQ^1)-\hcL(\hcQ^*))}{aK},
\end{align}
for some positive constant $\bar{C}$.
By \eqref{Eq:square-summable-primal}--\eqref{Eq:square-summable-dual}, and Lemma \ref{Lemm:bound-grad}, it implies
\[
\frac{1}{K} \sum_{k=2}^K \|\nabla {\hcL}(\hcQ^k)\|_F^2 \leq \hat{C}\cdot\frac{a\|{\cal V}^1 - {\cal V}^{0}\|_F^2+(\hcL(\hcQ^1)-\hcL(\hcQ^*))}{aK},
\]
for some positive constant $\hat{C}$.
This completes the proof.
\end{proof}

\vskip 0.2in


\end{document}